%% file: neurips_2022.tex
\documentclass{article}

\usepackage[preprint,nonatbib]{neurips_2022}

\usepackage[utf8]{inputenc} 
\usepackage[T1]{fontenc}    
\usepackage{hyperref}       
\usepackage{url}            
\usepackage{booktabs}       
\usepackage{amsfonts}       
\usepackage{nicefrac}       
\usepackage{microtype}      
\usepackage{xcolor}         
\usepackage[font=footnotesize,labelfont=bf]{caption}

\usepackage[square,numbers]{natbib}
\usepackage[acronym,nowarn,section,nonumberlist]{glossaries}
\glsdisablehyper{}
\input{"./acronym"}

\usepackage{"./GrandMacros"}
\usepackage{amsmath}
\usepackage{amssymb}
\usepackage{mathtools}
\usepackage{amsthm}
\usepackage{bbm}
\usepackage{enumerate}
\usepackage[shortlabels]{enumitem}
\usepackage{lipsum}  
\usepackage{comment}
\usepackage[capitalize,noabbrev]{cleveref}
\usepackage{algorithm}
\usepackage{algpseudocode}
\theoremstyle{plain}
\newtheorem{theorem}{Theorem}
\newtheorem{proposition}[theorem]{Proposition}
\newtheorem{lemma}[theorem]{Lemma}

\theoremstyle{definition}
\newtheorem{definition}{Definition}

\theoremstyle{remark}
\newtheorem{remark}{Remark}

\usepackage[textsize=tiny]{todonotes}
\newcommand{\wdeng}[1]{{\color{blue} (\textbf{wdeng:} #1)}}

\title{Towards a Unified Framework for Uncertainty-aware Nonlinear Variable Selection with Theoretical Guarantees
}

\author{%
  Wenying Deng \\
  Harvard University\\
  \texttt{wdeng@g.harvard.edu} \\
 \And
  Beau Coker \\
  Harvard University\\
 \texttt{beaucoker@g.harvard.edu} \\
 \AND
  Rajarshi Mukherjee\\
 Harvard University \\
 \texttt{ram521@mail.harvard.edu} \\
 \And
 Jeremiah Zhe Liu\thanks{Co-senior author. Work done at Harvard University.}\\
 Harvard University \& Google Research\\
 \texttt{jereliu@google.com} \\
 \AND
 Brent A. Coull$^*$ \\
 Harvard University\\
 \texttt{bcoull@hsph.harvard.edu} \\
}

\begin{document}

\maketitle

\begin{abstract}
 We develop a simple and unified framework for nonlinear variable selection that incorporates uncertainty in the prediction function and is compatible with a wide range of machine learning models (e.g., tree ensembles, kernel methods, neural networks, etc). In particular, for a learned nonlinear model $f(\mathbf{x})$, we consider quantifying the importance of an input variable $\mathbf{x}^j$ using the integrated partial derivative $\Psi_j = \Vert \frac{\partial}{\partial \mathbf{x}^j} f(\mathbf{x})\Vert^2_{P_\mathcal{X}}$. We then (1) provide a principled approach for quantifying variable selection uncertainty by deriving its posterior distribution, and (2) show that the approach is generalizable even to non-differentiable models such as tree ensembles. Rigorous Bayesian nonparametric theorems are derived to guarantee the posterior consistency and asymptotic uncertainty of the proposed approach. Extensive  simulations and experiments on healthcare benchmark datasets confirm that the proposed algorithm outperforms existing classic and recent variable selection methods.
\end{abstract}

\vspace{-1em}
\section{Introduction}
Variable selection is often of fundamental interest in many data science applications, providing benefits in prediction error, interpretability, and computation by excluding unnecessary variables. 
As datasets grow in complexity and size, it is crucial that variable selection methods can account for complex dependencies among variables while remaining computationally feasible. Furthermore, as the number of approaches to model such datasets have increased, it is crucial that the importance of each variable can be compared across model classes and extended to new ones as they are developed.

While there are established approaches for variable selection in linear models (e.g., LASSO regression \cite{hastie_statistical_2015}), there is little consensus 
in methodology or theory for variable selection in nonlinear models.
Generalized additive models \cite{hastie_generalized_1990} use similar variable selection methods as their linear counterparts \cite{wang_gams_2014}, but the additivity assumption for nonlinear functions of the variables is too restrictive in many applications. Random Forests (RF) \cite{breiman_random_2001} measure variable importance using an impurity measure, which is based on the average reduction of the loss function were a given variable be removed from the model. \cite{friedman_greedy_2001} extended this method to boosting, where the variable importance is generalized by considering the average over all of the decision trees. Deep neural networks (DNNs) are widely-used for many artificial intelligence applications, and a substantial effort has been invested into developing DNNs with variable selection capabilities. Typically, this class of models involves manipulating the input layer, for example by imposing an $L_1$ penalty \cite{castellano_variable_2000, feng_sparse-input_2019}, using backward selection \cite{castellano_variable_2000}, or knockoffs \cite{lu_deeppink_2018}.
Unfortunately, each model class based on DNNs requires a tailored variable selection procedure, which limits comparability across different model formulations. 

Bayesian variable selection methods provide principled uncertainty quantification in variable importance estimates as well as a complete characterization of their dependency structure. These methods allow the variable selection procedure to tailor its decision rule with respect to the correlation structure \cite{liu_variable_2021}. Yet, as in frequentist models, each method has a different definition of a variable's importance. For example, in Bayesian additive regression trees (BART), a variable's importance can be measured by the proportion of trees that use it \cite{chipman_bart_2010}, while in \gls{GP} models, a variable's importance can be measured by the frequency of the fluctuations of the resulting function (e.g., the length-scale parameter as controlled by the automatic relevance determination) in the direction of the variable \cite{neal_bayesian_1996, wipf2007new}. Furthermore, the traditional Bayesian modeling procedures tend to be computationally burdensome, making them less feasible for large-scale applications
\citep{andrieu2003introduction}.

Our work starts with the observation that many machine learning models can be written as kernel methods by constructing a corresponding feature map. For example, random forests can be written as kernel methods by partitions \cite{davies2014random}, and deep neural networks can be written as kernel methods by using the last hidden layer as the feature map \cite{snoek_bayesopt_2015, hinton2007using, calandra2016manifold}. Each of these feature maps can be constructed before the Bayesian learning of the Gaussian process (e.g., by pre-training on the same or a separate dataset), providing additional modeling expressiveness and representational capacity. The ability of a \gls{GP} model to incorporate these adaptive feature maps becomes especially important in high-dimension applications, where effective dimension reduction is necessary to circumvent the curse of dimensionality and to ensure good finite-sample performance \citep{bach_breaking_2016}.


\textbf{Contributions}. We propose a unified variable selection framework that is compatible with a wide range of machine learning models that can be defined by, or be closely approximated by, a differentiable feature map. Notable members include neural networks and random forests (\cref{sec:feature_gp_examples}). Our approach defines variable importance as the norm of the function's partial derivative, as was previously studied in the context of frequentist nonparametric regression \cite{rosasco_nonparametric_2013}. We extend it to apply to a much wider class of models than previously considered (\cref{sec:prelim}), propose a principled Bayesian approach to quantify the variable selection uncertainty in finite data (\cref{sec:varimp}), and derive rigorous Bayesian nonparametric theorems to guarantee the method's consistency and asymptotic optimality (Section \ref{sec:theory}). 
To incorporate powerful non-differentiable models into our framework, we also show how to apply this approach to partition-based methods (e.g., decision trees) by leveraging its (soft) feature representation (\cref{sec:fdt}). 
This leads to the first derivative-based Bayesian variable selection approach for tree-type models that is both theoretically grounded and empirically powerful, strongly outperforming other variable selection approaches tailor-designed for random forest (e.g., impurity or random-forest knockoff \citep{breiman_classification_1984, candes_panning_2017}).
We conduct extensive empirical validation of our approach and compare its performance to that of many existing methods across a wide range of data generation scenarios. Results show a clear advantage of the proposed approach, especially in complex scenarios or when the input is a mixture of discrete and continuous features (Section \ref{sec:exp}).

\vspace{-1.em}
\section{Preliminaries}
\label{sec:prelim}
\vspace{-.5em}
\textbf{Problem Setup}. We consider the classic nonparametric regression setting with $d$-dimensional features $\bx=(\bx^1, \dots, \bx^d) \in \Xsc = \real^d$ and a continuous response $y \in \real$. The features $\bx$ are allowed to have a flexible nonlinear effect on $y$, such that:
\begin{align}
y = f_0(\bx) + e_i, \quad \text{where } e_i \overset{i.i.d.}{\sim} \mathcal{N}(0, \sigma^2), 
\label{data}
\end{align}
with homoscedastic noise level $\sigma^2$. The data dimension $d$ is allowed to be large but assumed to be constant and does not grow with the sample size $n$. 
Here the data-generating function $f_0$ is a flexible nonlinear function that resides in a \gls{RKHS} $\Hsc_0:\Xsc_0 \rightarrow \real$ induced by a certain positive definite kernel function $k_0$, and the input space of the true function $\Xsc_0$ spans only a small subset of the input features $(\bx^1, \dots, \bx^d)$, i.e., $\Xsc_0 \subset \Xsc$.

To this end, the goal of \textit{global} variable selection is to produce a variable importance score $\psi_j$ for each of the input features $(\bx^1, \dots, \bx^d)$ such that it can be used as a classification signal for whether $\bx^j \in \Xsc_0$. As a result, the variable selection decision can be made by threhsolding $\psi_j >s$ with a pre-defined threshold $s$. The quality of a variable selection signal $\psi_j$ can be evaluated comprehensively using a standard metric such as the \gls{AUROC}, which measures the Type-I and Type-II errors of variable selection decision $I(\psi_j > s)$ over a range of thresholds $s$.

\vspace{-.5em}
\subsection{Quantifying Model Uncertainty via Featurized \gls{GP}}
\label{sec:gp}
\vspace{-.5em}
In the nonlinear regression scenario given by Equation (\ref{data}), a classic approach to uncertainty-aware model learning is the Gaussian process (GP). Specifically, assuming that $f_0$ can be described by a flexible \gls{RKHS} $\Hsc_k$ governed by the kernel function $k$, the \gls{GP} model imposes a Gaussian process prior $f \sim \Gsc \Psc(0, k)$, such that the function evaluated at any collection of examples follows a multivariate normal ($\mathcal{MVN}$) distribution
\begin{align*}
\mathbf{f} \equiv (f(\bx_1), \ldots, f(\bx_n))^\top \sim \mathcal{MVN}(\bm_{n \times 1}, \bK_{n \times n}),
\end{align*}
with mean $\bm_i = m(\bx_i)$ and covariance matrix $\bK_{i, j}=k(\bx_i, \bx_j)$. The choice of the prior mean $\bm$ and kernel $k$ enable prior specification directly in function space. For example, the Mat\'{e}rn kernel places a prior over $\lceil\nu\rceil-1$ times differentiable functions, with length-scale $l^2$ and amplitude variance $\sigma^2$. As $\nu\to\infty$, this reduces to the common radial basis function (RBF) kernel $k(\bx_i, \bx_j) = \sigma^2 \exp(\|\bx_i - \bx_j \|^2_2 / l^2)$. 


Under the above construction, the posterior predictive distribution of $f$ evaluated at new observations $\bx^*_1,\dotsc,\bx^*_{n^*}$ is also a multivariate normal,
\begin{align}
\label{eq:gp_dual}
\mathbf{f}^*  |\{\bx_i\, y_i\}_{i=1}^n &\sim \mathcal{MVN}(\mathbb{E}[\mathbf{f}^*], \text{Cov}[\mathbf{f}^*]),
\quad \mbox{where}
\\ 
\nonumber
\mathbb{E}[\mathbf{f}^*] = \bm^* + \bK^* (\bK + \sigma^2\bI_n)^{-1}&(\by-\bm);
\quad 
\text{Cov}[\mathbf{f}^*] = \bK^{**} - \bK^*(\bK + \sigma^2 \bI_n)^{-1}\bK^{*\top},
\end{align}
with $\bm^*_i = m(\bx_i^*)$, $\bK^*_{ij}=k(\bx^*_i, \bx_j)$, and $\bK^{**}_{ij}=k(\bx^*_i, \bx^*_j)$.
\cref{eq:gp_dual} is known as the kernel-based representation (or dual representation) of $f$ \cite{rasmussen2006gaussian}. Although mathematically elegant, the posterior (\ref{eq:gp_dual}) is expensive to compute due to the need to invert the $n \times n$ matrix $(\bK + \sigma^2\bI)^{-1}$.

\textbf{Feature-based Representation of GP}.
Alternatively, Mercer's theorem \cite{cristianini_introduction_2000} states that as long as the kernel function $k(\cdot, \cdot)$ can be written as the inner product of a set of basis functions $\phi(\bx)=\{\phi_k(\bx)\}_{k=1}^D$, such that $k(\bx, \bx')=\phi(\bx)^\top \phi(\bx')$, then elements of the \gls{RKHS} $f \in \Hsc_k$ can be written in terms of a linear expansion of basis functions \cite{rasmussen2006gaussian}:
\begin{align}
    f(\bx) = \sum_{k=1}^D \bbeta_k \phi_k(\bx)=\phi(\bx)^\top \bbeta, \text{   where } \bbeta \sim \mathcal{MVN}(\bmu, \bI_{D}).
    \label{eq:gp_primal}
\end{align}
This is known as the feature-based representation (or primal representation) of a Gaussian process. Notice that (\ref{eq:gp_primal}) is not an approximation method but an \textit{exact} reparametrization of the Gaussian process model whose kernel function is induced by feature representation $\phi(\bx)$. 

\textbf{Scalable Posterior Computation via Minibatch Updates}.
The above feature-based representation is powerful in that it reduces the \gls{GP} posterior inference into a Bayesian linear regression problem for $\bbeta$. This brings two concrete benefits. First, the posterior of $\bbeta$ in \cref{eq:gp_primal} adopts a closed-form:
\begin{align}
\label{eq:primal_coef_post}
&\bbeta  \sim \mathcal{MVN}(\mathbb{E}[\bbeta], \text{Cov}[\bbeta]),
\quad \mbox{where} 
\\ 
\nonumber
\mathbb{E}[\bbeta] 
&= \bmu + \Sigma_{\bbeta}\Phi^\top(\by - \Phi\bmu)/\sigma^2;
\quad 
\text{Cov}[\bbeta] = \Sigma_{\bbeta} = (\Phi^\top\Phi/\sigma^2 + \bI)^{-1},
\end{align}

where $\Phi = (\phi(\bx_1)^\top,\dotsc,\phi(\bx_n)^\top)^\top  \in \mathbb{R}^{n\times D}$ is the feature matrix evaluated on the training data \cite{rasmussen2006gaussian}. For large-scale applications, \cref{eq:primal_coef_post} enables us to compute the exact posterior of $\bbeta$ in a mini-batch fashion. For example, the posterior matrix $\text{Cov}[\bbeta] = \Sigma_{\bbeta}$ can be updated using the Woodbury identity:
\begin{align}
&\Sigma_{\bbeta, t+1}
= \Sigma_{\bbeta, t}-\Sigma_{\bbeta, t}\Phi_m^\top 
(\sigma^2 \bI+\Phi_m\Sigma_{\bbeta, t}\Phi_m^\top)^{-1}\Phi_m\Sigma_{\bbeta, t}.
\label{eq:primal_cov_post}
\end{align}
where $\Phi_m$ is the $D$-dimension batch-specific feature matrix evaluated on the mini-batch. Similarly, the posterior mean $\mathbb{E}[\bbeta]$ can be computed by accumulating the $D \times 1$ vector $\Phi^\top(\by - \Phi\bmu)=\sum_m \Phi_m^\top(\by_m - \Phi_m\bmu)$, and compute the posterior mean according to \cref{eq:primal_coef_post} at the end.


The posterior distribution of $\bbeta$ induces a Gaussian process posterior for the prediction function $\mathbf{f}^* = \Phi^* \bbeta$, where $\Phi^*$ is the feature map evaluated on the test data, with mean $\mathbb{E}[\mathbf{f}^*]=\Phi^*\bmu + \Phi^*\Sigma_{\bbeta}\Phi^\top(\by - \Phi\bmu)/\sigma^2$ and covariance $\text{Cov}[\mathbf{f}^*] = \Phi^*\Sigma_{\bbeta}\Phi^{* \top}$. This distribution is equivalent to the kernel-based representation (\ref{eq:gp_dual}) but reduces the computational complexity from cubic time $O(n^3)$ to a linear time $O(n)$ and is minibatch compatible (i.e., \cref{eq:primal_cov_post}).

\textbf{Incorporating Modern ML Model Classes}. The second key advantage of the feature-based representation (\ref{eq:gp_primal}) is its generality: a wide range of machine learning models can be written in term of the feature-based form $f(\bx)=\phi(\bx)^\top\bbeta$ \cite{rahimi_random_2007, davies2014random, lee2017deep}, making the Gaussian process a unified framework for quantifying model uncertainty for a wide array of modern ML models. \cref{sec:feature_gp_examples} summarizes important examples including GAMs, decision trees, random-feature models, deep neural networks and their ensembles. Furthermore, when a deterministically-trained $\hat{\bbeta}$ is available (e.g., via a sophisticated adaptive shrinkage procedure that is not available in Bayesian context), we can incorporate this as prior knowledge into \gls{GP} modeling by setting $\bmu = \hat{\bbeta}$ (\cref{eq:gp_primal}).

\subsection{Bayesian Nonparametric Guarantees for Probabilistic Learning}

The quality of a Bayesian learning procedure is commonly measured by the learning rate of its posterior distribution $\Pi_n=\Pi(\cdot \mid \{\bx_i, y_i\}_{i=1}^n)$. Intuitively, the rate of this convergence is measured by the size of the smallest shrinking balls around $f_0$ that contain most of the posterior probability. Specifically, we consider the size of the set $A_n = \{g \mid \Vert g-f_0\Vert_n^2\leq M \epsilon_n\}$ 
such that $\Pi_n(A_n) \to 1$ \citep{ghosal_convergence_2007, polson_posterior_2018}. The concentration rate $\epsilon_n$ here indicates how fast the small ball $A_n$ concentrates towards $f_0$ as the sample size increases. Below we state the formal definition of posterior convergence \cite{ghosal_convergence_2007}.
\begin{definition}[Posterior Convergence]
\label{def:inj}
For $f_0: \Xsc \to \mathbb{R}$ where $\Xsc=\real^d$, denote $\Hsc_0$ the true \gls{RKHS} induced by a kernel function $k_0$, and denote $\Hsc_{\phi}$ induced by the feature function $\phi: \Xsc \rightarrow \real^D$. Let $f_0 \in \Hsc$ be the true function, and
let $\mathbb{E}_0$ denote the expectation with respect to the true data-generation distribution. Assuming $\Hsc_\phi$ is dense in $\Hsc$, then, the posterior distribution $\Pi_n(f)$ concentrates around $f_0$ at the rate $\epsilon_n$ if there exists an $\epsilon_n \to 0$ such that, for any $M_n \to \infty$,
\begin{align}
    \mathbb{E}_0\Pi_n(f: \Vert f - f_0\Vert_n^2 \geq M_n \epsilon_n) \to 0.
    \label{eq:post_converg}
\end{align}
\end{definition}
Notice that we allow the model space $\Hsc_\phi$ and the true function space $\Hsc$ to be different, but the $\Hsc_\phi$ must be \textit{dense} in the $\Hsc$ for the convergence to happen. Fortunately, this condition is shown to hold for a wide variety of ML models, including random features, random forests, and neural networks \citep{biau2012analysis, hornik1989multilayer, rahimi2008uniform, schmidt-hieber_nonparametric_2020, rovckova2020posterior}. The notion of posterior convergence can also be used to discuss the learning quality of other probabilistic estimates (e.g., variable importance $\psi_j$). In that case, we can simply replace $(f, f_0)$ in (\ref{eq:post_converg}) by their variable importance counterparts. This is the focus of \cref{sec:theory}.


\vspace{-.5em}
\section{Methods}
\label{sec:methods}
\vspace{-0.5em}

\subsection{Quantifying Variable Importance under Uncertainty}
\label{sec:varimp}
In this work, we consider quantifying the \textit{global} importance of a variable based on the norm of the corresponding partial derivative. This is motivated by the observation that, if a function $f$ is differentiable, the relative importance of a variable $\bx^j$ at a point $\bx$ can be captured by the magnitude of the partial derivative function, $|\frac{\partial}{\partial \bx^j} f(\bx)|$ \cite{rosasco_nonparametric_2013}.
This proposed quantity requires the consideration of two issues. First, instead of quantifying the relevance of a variable on a single input point, we need to define a proper global notion of variable importance. Therefore, it is natural to integrate the magnitude of the partial derivative over the input space $\bx \in \Xsc: \Psi_{j}(f)=\Vert\frac{\partial}{\partial \bx^j} f\Vert_{P_\Xsc}^{2}=\int_{\bx \in \Xsc}|\frac{\partial}{\partial \bx^j} f(\bx)|^{2}~dP_\Xsc(\bx)$. Second, since $P_\Xsc(\bx)$ is not known from the training observations, $\Psi_{j}(f)$ can be approximated by its empirical counterpart,
\vspace{-1em}
\begin{align}
\psi_{j}(f)=\Vert\frac{\partial}{\partial \bx^j} f\Vert_{n}^{2}=\frac{1}{n} \sum_{i=1}^{n}|\frac{\partial}{\partial \bx^j} f(\bx_{i})|^{2}. 
\label{eq:varimp}
\end{align}
Notice that $\psi_j(f)$ is an estimator that is derived from the prediction function $f$ estimated using finite data. Consequently, to make a proper decision regarding the importance of an input variable $\bx^j$, it is important to take into account uncertainty in $f$. To this end, by leveraging the featurized \gls{GP} representation introduced in Section \ref{sec:varimp}, we show that this can be done easily for a wide range of ML models $f(\bx)=\phi(\bx)^\top\bbeta$ by studying the posterior distribution of $\psi_j$.

\textbf{Posterior Distribution of Variable Importance}. 
After we obtain the posterior distribution of $\bbeta$ (\ref{eq:primal_coef_post}), the posterior distribution of variable importance can be derived according to \cref{eq:varimp}:
\begin{align}
\psi_{j}(f)=\frac{1}{n}|\frac{\partial}{\partial \bx^j} f(\bX)|^\top |\frac{\partial}{\partial \bx^j} f(\bX)|=\frac{1}{n}\bbeta^\top \frac{\partial \Phi}{\partial \bx^j}  \frac{\partial \Phi^\top}{\partial \bx^j} \bbeta,
\label{eq:varimp_post}
\end{align}
where $\frac{\partial \Phi}{\partial \bx^j} \in \mathbb{R}^{D\times n}$ is the derivative of the feature map with respect to $\bx^j$, across $n$ training samples. The posterior distribution of $\psi_{j}(f)$ adopts a closed form as a generalized chi-squared distribution (see  \cref{sec:posterior_app} for derivation).
In practice, we can sample $\psi_j$ conveniently from its posterior distribution by computing $\frac{\partial}{\partial \bx^j}f(\bX)=\big(\frac{\partial \Phi}{\partial \bx^j} \big)^\top \bbeta^{(s)}$, where $\bbeta^{(s)}$ are Monte Carlo samples from the closed-form posterior (\ref{eq:primal_coef_post}).

There are two ways in which uncertainty aids the model selection process. First, the posterior survival function $P(\psi_j(f) > s)$ of the variable importance utilizes the full posterior distribution of $\psi_j(f)$ to identify the probability that the variable $\bx^j$ exceeds a given threshold $s$. By increasing $s \in (0, \infty)$, $P(\psi_j > s)$ provides a intuitive sense of how model's belief about the importance of variable $\bx^j$ changes as the criteria $s$ becomes more stringent, similar to the regularization path in LASSO methods \citep{friedman2010regularization} but with the incorporation of posterior uncertainty about the variable importance. See Appendix \ref{sec:exp_bang} for an application to a Bangladesh birth cohort study. Second, by integrating the survival function over the threshold, i.e., $\int_{s>0} P(\psi_j(f) > s)~ds$, we obtain the posterior mean of $\psi_j(f)$, and this too incorporates uncertainty in $f$. To see this, notice by using the ``trace trick'' we can write
\begin{align}
    \mathbb{E}[\psi_j(f)] 
    &= \mathbb{E}\left[ \tr\left(
    \bbeta^\top \frac{\partial \Phi}{\partial \bx^j} \frac{\partial \Phi^\top}{\partial \bx^j} \bbeta \right)
    \right] 
    = \mathbb{E}[\bbeta]^T \frac{\partial \Phi}{\partial \bx^j} \frac{\partial \Phi^\top}{\partial \bx^j} \mathbb{E}[\bbeta]
    + \tr\left( \frac{\partial \Phi}{\partial \bx^j} \frac{\partial \Phi^\top}{\partial \bx^j} \text{Cov}[\bbeta] \right),
\end{align}
where all expectations are taken with respect to the posterior.
Therefore, the posterior mean of $\psi_j(f)$ depends on the covariance structure of $\bbeta$, and how it interacts with the eigenspace of the partial derivative functions (encoded by $\frac{\partial \Phi}{\partial \bx^j} \frac{\partial \Phi^\top}{\partial \bx^j}$). In Section \ref{sec:exp} we provide an extensive investigation of \gls{AUROC} scores using the posterior mean of $\psi_j(f)$ for variable selection.

In \cref{sec:alg_summary}, we summarize the algorithms for computing the posterior distributions of the featurized Gaussian process (\cref{eq:primal_coef_post}) and for the posterior distributions of variable importance  (\cref{eq:varimp_post}), and discuss their space and time complexity.

\subsection{Theoretical Guarantees}
\label{sec:theory}

From a theoretical perspective, the variable importance measure $\psi_j$ introduced in (\ref{eq:varimp}) can be understood as a quadratic functional of the Gaussian process model $f$ \cite{efromovich1996optimal}. To this end, rigorous Bayesian nonparametric guarantees can be obtained for $\psi_j$'s ability in learning the true variable importance in finite samples (i.e., posterior convergence, \cref{thm:1}) and its statistical optimality from a frequentist perspective, in providing a low-variance estimator that attains the Cram\'{e}r-Rao bound
(i.e., Bernstein von-Mises phenomenon, \cref{thm:2}). 

\textbf{Posterior Convergence}. We first show that, for a ML model $f$ that can learn the true function $f_0$ with rate $\epsilon_n$  (in the sense of \cref{def:inj}), 
the entire posterior distribution of the variable importance measure $\psi_j(f)$ converges consistently to a point mass at the true $\Psi_j(f_0)$ at a speed that is equal or faster than $\epsilon_n$. 
\begin{theorem}[Posterior Convergence of Variable Importance $\psi_j$]
\label{thm:1}
Suppose $y_i = f_0(\bx_i) + e_i, \; e_i \overset{i.i.d.}{\sim} \mathcal{N}(0, \sigma^2)$, and denote as $\mathbb{E}_0$ the expectation with respect to the true data-generation distribution centered around $f_0$. For the \gls{RKHS} $\Hsc_{\phi}$ induced by the feature function $\phi: \Xsc \rightarrow \real^D$ and $f \in \Hsc_{\phi}$, if:
\begin{enumerate}[(1)]
    \item The posterior distribution $\Pi_n(f)$ converges toward $f_0$ at a rate of $\epsilon_n$;
    \item The differentiation operator $D_j: f \to \frac{\partial}{\partial \bx^j} f$  is bounded: $\Vert D_j \Vert_{op}^2 = \inf \{C\geq 0: \Vert D_j f\Vert_2^2 \leq C \Vert f \Vert_2^2, \text{ for all } f \in \Hsc_{\phi}\}$;
\end{enumerate}
 Then the posterior distribution for $\psi_j(f)=\Vert\frac{\partial }{\partial \bx^j}f\Vert_n^2$ contracts toward $\Psi_j(f_0)=\Vert\frac{\partial }{\partial \bx^j}f_0\Vert_{P_\Xsc}^{2}$ at a rate not slower than $\epsilon_n$. That is, for any $M_n \to \infty$, 
\[\mathbb{E}_0 \Pi_n\left[ \underset{j\in \{1, \ldots, d\}}{\sup}| \psi_j(f) - \Psi_j(f_0)| \geq M_n \epsilon_n \right] \to 0. \]
\end{theorem}
Proof is in \cref{sec:proof_convergence}. \cref{thm:1} is a generalization of the classic result of quadratic functional convergence under linear models and sparse neural networks to a much wider range of ML models in the context of Bayesian variable selection   \citep{efromovich1996optimal, liu_variable_2021, wang2020uncertainty}. It confirms the important fact that, for a ML model $f$ that can accurately learn the true function $f_0$ under finite data, we can  consistently recover the true variable importance at a fast rate by using the proposed variable importance estimate $\psi_j(f)$, despite the potential lack of identifiablity in the model parameters (e.g., weights in a neural network). 

From a practical point of view, \cref{thm:1} reveals that the finite-sample performance of variable importance $\psi_j(f)$ depends on two factors: (1) the finite-sample generalization performance of the prediction function $f$, and (2) the mathematical property of $f$ in terms of its Lipschitz condition. Therefore, to ensure effective variable selection in practice, the practitioner should take care to select a model class $f$ that has a theoretical guarantee in capturing the target function $f_0$, empirically delivers strong generalization performance under finite data, and is well-conditioned in terms of the behavior of its partial derivatives. To this end, we note that, under the featurized Gaussian process $f=\phi(\bx)^\top\bbeta$ discussed in this work, users are free to choose a performant model class (e.g., random forest, random-feature or DNN) whose feature representation spans an \gls{RKHS} $\Hsc_{\phi}$ that is \textit{dense} in the infinite-dimensional function space (therefore $f$ enjoys a convergence guanrantee) \citep{biau2012analysis, hornik1989multilayer, rahimi2008uniform, schmidt-hieber_nonparametric_2020, rovckova2020posterior}, and is empirically more effective than the \gls{GP} methods based on classic kernels such as RBF. (We discuss the Lipschitz condition of these models in \cref{sec:thm_discuss_lipschitz}) Indeed, as we will verify in experiments (\cref{sec:exp}), there does not exist an ``optimal" model class that performs universally well across all data settings (i.e., no free lunch theorem \citep{wolpert1997no}). 
This highlights the importance of having a general-purpose framework for variable selection that can flexibly incorporate the most effective model for the task at hand.

\textbf{Statistical Efficiency $\&$ Uncertainty Quantification.} Next, we verify the uncertainty quantification ability of the variable importance measure $\psi_j(f)$ under featurized \gls{GP}, by showing that it exhibits the \gls{BvM} phenomenon. That is, its posterior measure $\Pi_n(\psi_j(f))$ converges towards a Gaussian distribution that is centered around the truth $\Psi_j(f_0)$, so that its $(1 - \alpha)\%$ level credible intervals achieve the nominal coverage probability for the true variable importance. More importantly, the \gls{BvM} theorem verifies that the posterior distribution of $\psi_j(f)$ is \textit{statistically optimal}, in the sense that its asymptotic variance attains the \gls{CRB} that cannot be improved upon \citep{bickel2012semiparametric}.
\begin{theorem}[Bernstein-von Mises Theorem for Variable Importance $\psi_j$]
\label{thm:2}
Suppose $y_i = f_0(\bx_i) + e_i, \; e_i \overset{i.i.d.}{\sim} \mathcal{N}(0, \sigma^2), \; i=1, \dots, n$. 
Denote $D_j: f \to \frac{\partial}{\partial \bx^j} f$ the differentiation operator and $H_j=D_j^\top D_j$ the inner product of $D_j$, such that:
\begin{align}
\psi_j(f)=\Vert D_j (f) \Vert_n^2=  \frac{1}{n}\langle D_j f, D_j f \rangle = \frac{1}{n}f^\top H_j f. \label{eq:def_psi}
\end{align}
Assuming conditions (1)-(2) in \cref{thm:1} hold, and additionally:
\begin{enumerate}
\item[(3)] $f_0$ is square-integrable over the support $\Xsc$ and $\Vert f_0\Vert_2=1$;
\item[(4)] $\text{\texttt{rank}}(H_j) = o_p(\sqrt{n})$;
\end{enumerate}
Then
\[\sqrt{n}(\psi_j(f) - \psi_j(f_0)) \overset{d}{\rightarrow} \mathcal{N}(0, 4\sigma^2\Vert H_j f_0 \Vert_n^2). \]
\end{theorem}
Proof is in \cref{sec:proof_bvm}. \cref{thm:2} provides a rigorous theoretical justification for $\psi_j(f)$'s ability to quantify its uncertainty about the variable importance. More importantly, it verifies that 
$\psi_j(f)$ has the good frequentist property that it quickly converges to a minimum-variance estimator at a fast speed, which is important for obtaining good variable selection performance in practice. Compared to the previous \gls{BvM} results that tend to focus on a specific Bayesian ML model, \cref{thm:2} is considerably more general (i.e., applicable to a much wider range of models) and comes with a simpler set of conditions \citep{rockova2020semi, wang2020uncertainty, liu_variable_2021}. Specifically, (3) is a standard assumption in nonparametric analysis. It ensures the true function $f_0$ does not diverge towards infinity and makes learning impossible \citep{castillo2015bernstein}. The unit norm assumption $\Vert f_0\Vert_2=1$ is only needed to simplify the exposition of the proof, and the theorem can be trivially extended to $\Vert f_0\Vert_2=C$ for any $C>0$. The most interesting condition is (4). Let's denote $\Hsc_j$ the space of partial derivatives functions $\deriv{\bx^j}f$ of the model functions $f \in \Hsc_\phi$.  Then intuitively, (4) says to attain the \gls{BvM} phenomenon, the effective dimensionality of the derivative function space $\Hsc_j$ 
(as measured by $\text{\texttt{rank}}(H_j)=\text{\texttt{rank}}(D_j)$) cannot be too large. Since effective dimensionality of the derivative space is bounded above by that of the original \gls{RKHS} $f \in \Hsc_\phi$, (4) essentially states that the effective dimensionality of the model space $\Hsc_\phi$ cannot grow too fast with data size (i.e., $o_p(\sqrt{n})$). Fortunately, this condition is satisfied by a wide range of ML models including trees and deep networks \citep{rockova2020semi, wang2020uncertainty}. See \cref{sec:thm_discuss_bvm} for further discussion.

\vspace{-.5em}
\section{Experiment Analysis}
\label{sec:exp}

In this section, we investigate the finite-sample performance of the derivative norm metric $\psi_j$ for variable selection (\ref{eq:varimp}) under a wide variety of ML methods. 
We illustrate the breath of our framework by applying it to tree ensembles (\cref{sec:fdt}), where a principled and gradient-based uncertainty-aware variable selection approach has been previously unavailable. We also apply it to linear models and (approximation) kernel machines, which are standard approaches to variable selection in  data science practice \citep{tibshirani_regression_1996, bobb_bayesian_2015}.
Over a wide range of complex and realistic data scenarios (e.g., discrete features, interactions, between-feature correlations) derived from socioeconomic and healthcare datasets, we investigate the method's statistical performance in accurately recovering the ground-truth features (in terms of the Type I and Type II errors), and compare it to other well-established approaches in each of the model classes (\cref{tb:methods}). Our main observations are:
\vspace{-0.5em}
\begin{enumerate}[wide, labelwidth=!, labelindent=0.5pt]
\item[\textbf{O1}:] \textbf{Importance of generality}. There does not exist a model class that performs universally well across all data scenarios (i.e., no free lunch theorem \citep{wolpert1997no}, Figure \ref{fig:auc_main}, \ref{fig:cat_all}-\ref{fig:tst_mi}). This highlights the importance of an unified variable selection framework that incorporates a wide range of models, so that practitioners have the freedom of choosing the most suitable model class for the task at hand.
\item[\textbf{O2}:] \textbf{Good prediction translates to effective variable selection}. Comparing between different model classes, the ranking of models' predictive accuracy is generally consistent with the ranking of their variable selection performance under $\psi_j$ (i.e., better prediction translates to better variable selection, as suggested in \cref{thm:1}).
\item[\textbf{O3}:] \textbf{Statistical efficiency of $\psi_j$}. Comparing within each model class, the derivative norm metric $\psi_j$ generally outperforms other measures of variable importance. The advantage is especially pronounced in small samples and for correlated features. This empirically verifies that $\psi_j$ has good finite-sample statistical efficiency even under complex data scenarios (as suggested in \cref{thm:2}).
\end{enumerate}
\begin{table}[ht]
    \centering
    \resizebox{0.75\columnwidth}{!}{%
    \begin{tabular}{|c|c|c|}
     \hline\hline
      \textbf{Model Class}  &  
      \textbf{(Ours)} & \textbf{Baselines} \\
      \midrule
      Tree Ensembles   & RF-FDT & RF-Impurity, RF-Knockoff, BART \\
      \hline
      (Appr.) Kernel Methods & RFNN & BKMR, BAKR\\
      \hline
      Linear Models & GAM & BRR, BL \\
     \hline\hline
    \end{tabular}
    }
    \caption{Summary of methods considered in the experiments.}
    \label{tb:methods}
\vskip -0.2in
\end{table}

\textbf{Models $\&$ Methods}. 
We consider three main classes of models (\cref{tb:methods}): (I) \textbf{\gls{RF}}. Given a trained forest, we quantify variable importance using $\psi_j$ by translating it to an ensemble of \gls{FDT} 
(\cref{sec:fdt}), and compare it to three baselines: \gls{impurity} \citep{breiman_classification_1984}, \gls{knockoff} \citep{candes_panning_2017}, and \gls{BART}. (II) \textbf{(Approximate) Kernel Methods}. We apply $\psi_j$ to a random-feature model that approximates a Gaussian process with a RBF kernel \cite{rahimi_random_2007}, and set the number of features to $\sqrt{n} \, log(n)$ to ensure proper approximation of the exact RBF-GP \cite{rudi_generalization_2018}. We term this approach \gls{RFNN}, and compared it to both  \gls{BKMR}  \cite{bobb_bayesian_2015} based on a \gls{GP} with exact RBF kernel and spike-and-slab prior and \gls{BAKR} based on random-feature model with a projection-based feature importance measure and an adaptive shrinkage prior \citep{crawford_bayesian_2018}. (III) \textbf{Linear Models}. We apply $\psi_j$ to a featurized \gls{GP} representation of the \gls{GAM}, with the prior center $\bmu$ set at the frequentist estimate of the original GAM model obtained from a sophiscated REML procedure \citep{wood2006generalized}. We compare it to two baselines: \gls{BRR} \cite{hoerl_ridge_1970} and \gls{BL} \cite{park_bayesian_2008}. \cref{sec:exp_detail} provides further detail.
Previously, \citep{liu_variable_2021} studied the specialization of our framework to the deep neural networks (DNNs), we don't repeat that work here as DNN is not yet a standard data science model for tabular data.

To quantify variable importance while accounting for posterior uncertainty the variable importance $\psi_j(f)$, we examine its posterior survival function $\int_{s>0} P(\psi_j(f) > s)~ds$ (i.e., the posterior likelihood of $\psi_j(f)$ greater than the threshold $s$) integrated over the full range of thresholds $s$. 
For other methods, we use their default metrics to quantify variable importance
(e.g., variable inclusion probabilities in \gls{BART}, \gls{BKMR}. \cref{sec:exp_detail}).

\textbf{Datasets and Tasks.}
We consider two synthetic benchmarks and three real-world socio-economic and healthcare datasets, encapsulating challenging phenomena such as between-feature correlations and interaction effects. For the synthetic benchmark, we generate data under the Gaussian noise model $y \sim \mathcal{N}(f_0, 0.01)$ for four types of outcome-generation functions $f_0$ (\texttt{linear}, \texttt{rbf}, \texttt{matern32} and \texttt{complex}, see Appendix \ref{sec:exp_data} for a full description) with number of causal variables $d^\star = 5$. Two types of feature distribution are considered: (1) \textbf{synthetic-continuous}: all features follow $\bx^j \sim Unif(-2, 2)$; (2) \textbf{synthetic-mixture}: two of the causal features and two of the non-causal features are distributed as $Bern(0.5)$ and the rest are distributed as $Unif(-2, 2)$; We vary sample size $n \in \{100, 200, 500, 1000\}$ and data dimension $d \in \{25, 50, 100\}$, leading to 96 total scenarios.

For real-world data, we consider (1) \textbf{adult}: 1994 U.S. census data of 48842 adults with 8 categorical and 6 continuous features \cite{kohavi_scaling_nodate}; (2) \textbf{heart}: a coronary artery disease dataset of 303 patients from Cleveland clinic database with 7 categorical and 6 continuous features \cite{detrano_international_1989}; and (3) \textbf{mi}: disease records of myocardial infarction (MI) of 1700 patients from Krasnoyarsk
interdistrict clinical hospital during 1992-1995, with 113 categorical and 11 continuous features \cite{golovenkin_trajectories_2020}. 
All datasets exhibit non-trival correlation structure among features (Appendix \cref{fig:corr_adult}-\ref{fig:corr_mi}).
Since the ground-truth causal features on these datasets are not known, in order to rigorously evaluate variable selection performance, we follow the standard practice in causal ML to simulate the outcome based on causal features selected from data \citep{yao2021survey}. We use the four outcome-generating functions as described previously and evaluate over the same data size $\times$ dimension combinations, leading to 144 total scenarios\footnote{In the setting where required data dimension is higher than that of the real data, we generate additional synthetic features from $Unif(-2, 2)$. We use $n\in\{50,100,150,257\}$ for \textbf{heart} due to data size restriction.}. We repeat the simulation $20$ times for each scenario, and use \gls{AUROC} to measure the variable selection performance (in terms of Type I and II errors) of each method.

In the \cref{sec:exp_bang}, we further evaluate the method on a well-studied environmental health dataset (Bangladesh birth cohort study \citep{kile2014prospective}) with respect to the real outcome (infant development scores). We visualize the "Bayesian" regularization path as introduced earlier. The selected variables correspond well with the established toxicology pathways in the literature \citep{gleason2014contaminated}.
\begin{figure}[ht]
\begin{center}
\centerline{\includegraphics[width=1.0\columnwidth]{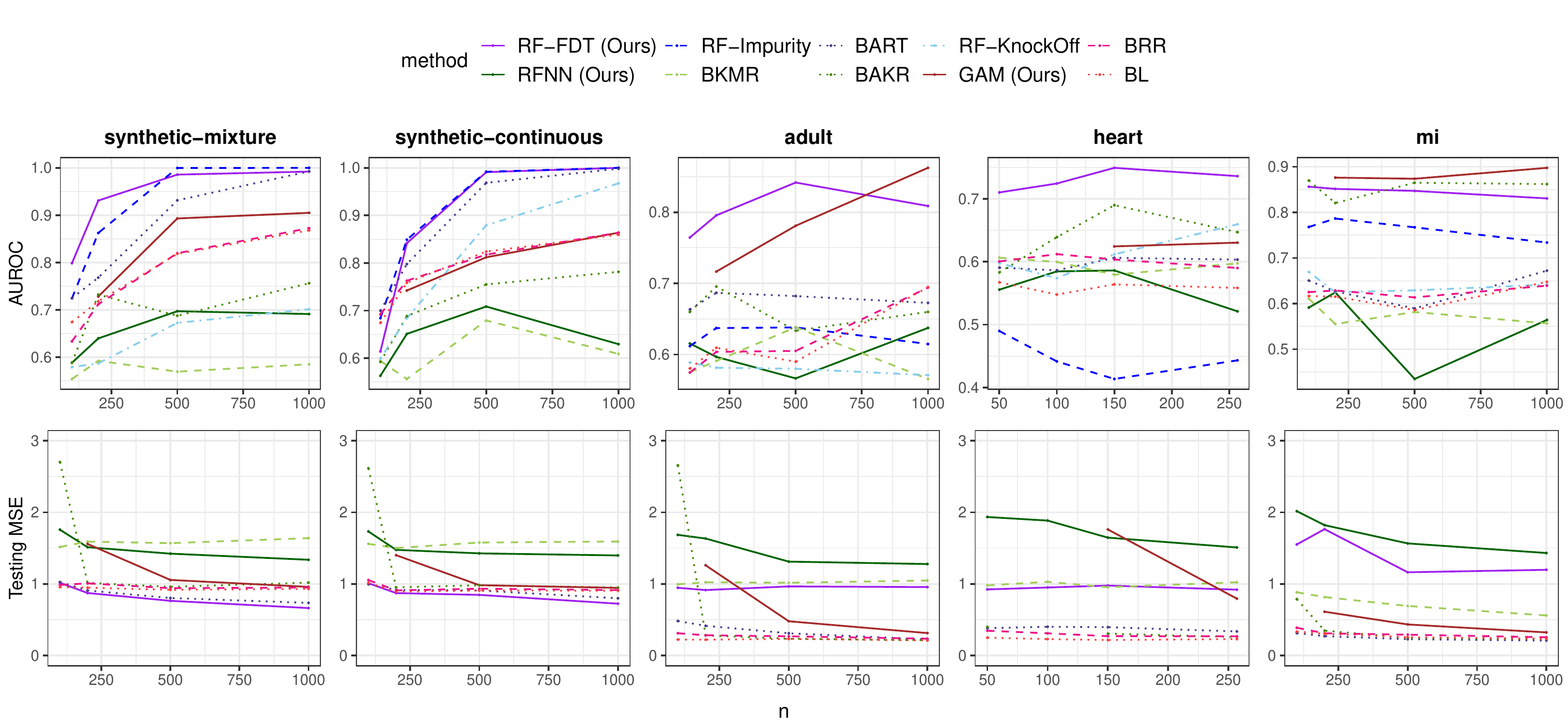}}
\caption{
Method performance in variable selection (measured by \gls{AUROC}, row 1) and prediction (measured by test MSE, row 2) under \texttt{matern32} data-generation function and with input dimension 100 (5 causal features). 
The ranking of the variable \gls{FDT} (solid purple) outperforms other methods in most of the data settings, and \gls{GAM} outperforms in the setting of large data size and high percentage of categorical features (\textbf{adult} and \textbf{mi}). The rankings of performance are roughly consistent between prediction and variable selection.)
}
\label{fig:auc_main}
\end{center}
\vskip -0.4in
\end{figure}

\subsection{Results}
\label{sec:exp_result}
\vspace{-0.5em}
\cref{fig:auc_main} shows the methods' performance in variable selection \textbf{(Row 1)} and prediction \textbf{(Row 2)}\footnote{For the prediction plots, a method will not be visualized if they share the model fit with another method (\gls{impurity} and \gls{knockoff}), or if the method does not produce valid result due to small sample size (\gls{GAM}).} in an exemplary setting, where the true function $f_0$ is \texttt{matern32} with an input dimension of  100. It represents the tabular data setting that we are the most interested in: nonlinear feature-response relationship with  interaction effects and  high input dimension. This is because $f_0$ is sampled from a \gls{RKHS} induced by M\'{a}tern $\frac{3}{2}$ kernel, which contains a large space of continuous and at least once differentiable functions \citep{rasmussen2006gaussian}. 
We delay complete visualizations for all 240 scenarios to \cref{sec:exp_result_app}.

Recalling the three observations introduced earlier: 

\textbf{O1 ("No free lunch")}: No method performs universally well. For example, \gls{BAKR} performs robustly in correlated datasets (\textbf{heart} and \textbf{mi}), but poorly otherwise. Kernel approaches (\textbf{RFNN} and \gls{BKMR}) performs competitively in low dimension, but their performance deteriorates quickly as dimension increases (\cref{fig:cat_all}-\ref{fig:cont_all}). \gls{FDT} generally is the strongest method in small samples and high dimensions, but can be outperformed by \gls{GAM} in large samples and data with high percentage of categorical features (\textbf{adult} and \textbf{mi}). This highlights the importance of an unified framework that allows users to select the most appropriate model for variable selection depending on the data setting.

\textbf{O2 ("Good prediction implies effective variable selection")}: Comparing among the gradient-based methods under each model class (i.e., \gls{FDT}, \gls{GAM} and \gls{RFNN}. Solid lines in \cref{fig:auc_main}), we see that their rankings in prediction (row 2) is largely consistent with rankings in variable selection. It's worth noting that this pattern is occasionally violated (e.g., \gls{GAM} in \textbf{adult}, n=500 and \textbf{heart}, n=250), but that does not contradict our conclusion (\cref{thm:1}) since the convergence rate of the  prediction function only forms an \textit{upper bound} for the convergence rate of $\psi_j$.

\textbf{O3 (Statistical efficiency of $\psi_j$)}: Comparing between variable selection methods from the same class (especially for tree models. i.e., \gls{FDT} v.s. \gls{impurity} / \gls{knockoff} / \gls{BART}), we see that \gls{FDT} is competitive or strongly outperforms its baselines in variable selection, despite being based on exactly the same fitted model (\gls{impurity} / \gls{knockoff}), or not accounting for the uncertainty in the tree growing process (\gls{BART}). This pattern is consistent in most data settings, and the advantage is especially pronounced in high dimensions, small data sizes, and correlated datasets (\cref{sec:exp_result_app}, \cref{fig:cat_all}-\ref{fig:mi}). This provides strong empirical evidence for the fact that $\psi_j$ is a statistically efficient estimator for variable selection with good finite-sample behavior (as suggested in \cref{thm:2}), and can deliver strong variable selection performance for tabular data when combined with a performant ML model like random forest.

\cref{sec:exp_result_app} contains further discussion.

\vspace{-1em}
\section{Discussion and Future Directions}
\label{sec:discussion}
\vspace{-.5em}

The modern data analysis pipeline typically involves fitting multiple models, comparing their performance, and iterating as necessary. When variable selection is involved, the practitioner may ask \textit{are the variable importances across models measuring the same behavior?} And, \textit{what if the most suitable model does not have a satisfactory variable selection procedure?} By framing model choice as kernel choice --- which we emphasize includes kernels corresponding machine learning methods like neural networks and random forests in addition to the long list of traditional kernels --- we propose a unified variable selection procedure that is compatible across models and we prove strong guarantees for this procedure. 

\textbf{Limitations}. Our proposed framework provides principled uncertainty quantification by performing exact Bayesian inference on the weights $\bbeta$ of a feature map $\phi(\bx)$. We do not consider uncertainty in the feature map itself. This means, for example,  that if the feature map is given by the last hidden layer of a neural network trained by maximizing the posterior, then our model class corresponds to the \textit{neural linear model}. This model is different from a fully Bayesian neural network, which performs posterior inference also on the kernel hyperparameters (i.e., the hidden weights) \cite{ober_nlms_2019, snoek_bayesopt_2015, thakur_luna_2021}. Likewise, the kernel induced by the featurized decision tree studied here does not consider uncertainty in the tree's partitioning process. Yet, this does not seem to be a significant limitation in our experiments (e.g., \gls{FDT} outperforms \gls{BART}), although this point still merits further investigation in the future.

In our experiments, we focused on kernels based on tree ensembles, kernel methods and linear models. In the future, it would be worth expanding this framework to other model classes (e.g., MARS \cite{friedman_mars_1991} or neural network) and  estimating the importance of interaction effects and higher-order terms.
We would also like to apply this method to large-scale scientific studies (e.g., epidemiology study based on extremely large EHR datasets) 
where an uncertainty-aware nonlinear variable selection method is typically impossible due to challenges with scalability.

\textbf{Societal Impacts}. 
The method proposed in this paper provides a theoretically-grounded approach for quantifying variable importance that is applicable to a wide range of ML models. We expect it to provide a set of powerful tools for practitioners to understand the importance of input variables in their ML models with limited data, which is especially important for scientific investigations such as epidemiology and computational biology. However, we recognize that this approach can potentially be utilized by bad actors to probe the input-variable uncertainty of an existing ML system, and use it to engineer more targeted white-box adversarial attacks. To this end, we recommend system developers to incorporate this approach into the formal verification procedure of a ML system, so as to monitor and understand the model uncertainty with respect to input variables, and devise proper improvement and prevention strategies (e.g., data augmentation or randomized smoothing targeted at specific variables) accordingly.

\clearpage
\bibliography{ref}

\clearpage
\appendix

\section{Additional Background and Technical Derivations}
\subsection{Neural Network Representation of Decision Tree}
\label{sec:nn_dt}
For each node in a learned decision tree, we know the feature the node is splitting on and its corresponding threshold. \cite{karthikeyan_learning_2021} provides a neural network representation of a decision tree:
\begin{align}
\label{nn:dt}
f(\bx| \bW, \bb, \bbeta) &= \sum_{l=0}^{D} \phi_l(\bx|\bW, \bb) \bbeta_l, \ \text{where}\nonumber \\
\phi_l(\bx|\bW, \bb) 
&= \sigma_{\mbox{\texttt{\tiny{step}}}} \big(\sum_{i=0}^{h-1} \sigma_{\mbox{\texttt{\tiny{step}}}}\big((\bx^\top \bw_{i, I(i, l)}+b_{i, I(i, l)})S(i, l) \big)-h\big).
\end{align}
In the above equations, $\bbeta_l \in \mathbb{R}$ is the prediction given by the $l^{th}$ leaf node, $h$ is the height of the tree and $D$ is the number of leaf nodes. 
$I(i, l)$ denotes the index of the $l^{th}$ leaf's predecessor in the $i^{th}$ level of the tree. $\bw_{ij} \in \mathbb{R}^d$ indicates the feature the node is splitting on using one hot encoding, with only one element being $1$ or $-1$ and the rest being $0$. $b_{ij} \in \mathbb{R}$ is the corresponding threshold (or the threshold multiplied by $-1$). The $-1$ is to guarantee that $\bx^\top \bw_{i, j}+b_{i, j}>0$ so that when multiplied by
\begin{align*}
S(i, l)=\begin{cases}
-1 & \text{ if } l^{th} \text{ leaf } \in \text{ left subtree of node } I(i, l),   \\
+1 &  \text{ otherwise, } 
\end{cases}
\end{align*}
the direction of $(\bx^\top \bw_{i, I(i, l)}+b_{i, I(i, l)})S(i, l)$ can be kept. $\sigma_{\mbox{\texttt{\tiny{step}}}}(\cdot)$ is the step function,
\begin{align*}
\sigma_{\mbox{\texttt{\tiny{step}}}}(a)=1, \text{ if } a \geq 0, \text{ and } \sigma_{\mbox{\texttt{\tiny{step}}}}(a)=0, \text{ if } a<0.
\end{align*}

Therefore, the model space can be regarded as a three-layer neural network with $\sigma_{\mbox{\texttt{\tiny{step}}}}$ as activation function, with $\bW$ as hidden weights and $\bb$ as hidden bias.

\subsection{Derivation of Posterior Distribution of Variable Importance}
\label{sec:posterior_app}

Recall from Equation \ref{eq:primal_coef_post} that the posterior distribution of $\bbeta$ is $\mathcal{MVN}(\mathbb{E}[\bbeta], \text{Cov}[\bbeta])$, which can be computed in closed form. This induces a distribution over the variable importance $\psi_j(f)$:
\begin{align*}
\psi_{j}(f)&=\frac{1}{n}|\frac{\partial}{\partial \bx^j} f(\bX)|^\top |\frac{\partial}{\partial \bx^j} f(\bX)|\\
&=\frac{1}{n}\bbeta^\top \big(\frac{\partial}{\partial \bx^j} \phi(\bX) \big) \big(\frac{\partial}{\partial \bx^j} \phi(\bX) \big)^\top \bbeta \\
&=\frac{1}{n}\bbeta^\top \bQ\bLambda\bQ^\top \bbeta \tag{Eigen-decomposition on $\big(\frac{\partial}{\partial \bx^j} \phi(\bX) \big) \big(\frac{\partial}{\partial \bx^j} \phi(\bX) \big)^\top$}\\
&= \frac{1}{n}\sum_{i=1}^D \lambda_i(\bq_i^\top \bbeta)^2 \tag{$\lambda_i$ is eigenvalue, $\bq_i$ is eigenvector}\\
&= \frac{1}{n}\sum_{i=1}^D (\lambda_i  V_i) \cdot Z_i, \tag{$V_i=\bq_i^\top \text{Cov}(\bbeta) \bq_i$}
\end{align*}

where $Z_i := (\bq_i^\top\bbeta)^2 /V_i \sim \chi_1^{2}(\mu_i)$ are independent random variables that follows a noncentral $\chi^2$ distribution with 1 degree of freedom and parameter $\mu_i=(\bq_i^\top \mathbb{E}[\bbeta])^2$. The values $\{\lambda_i\cdot V_i\}_{i=1}^D$ are scalar constants weighting each noncentral $\chi^2$ random variable $Z_i$. As a result, the full distribution is a well-known distribution of a linear combination of non-central $\chi^2$ distributions \citep{harville1971distribution}. This distribution has mean $\sum_{i=1}^D (\lambda_i  V_i) \cdot (1+\mu_i)$, variance $\frac{2}{n}\sum_{i=1}^D (\lambda_i  V_i)^2 \cdot (1+2\mu_i)$, and it can be sampled efficiently from by using the linear combination representation as introduced above.

\newpage
\subsection{Algorithm Summary}
\label{sec:alg_summary}

Given a fixed\footnote{Namely, the feature function $\phi(\bx)$ is either fixed by construction like random feature models or kernel machine using classic kernels (RBF, Mat\'{e}rn, etc). Or $\phi(\bx)$ is already learned elsewhere (i.e., pre-trained on the same or a separate dataset) like random forests or neural networks.} feature function $\phi: \Xsc \rightarrow \real^D$, we present algorithm summaries for (1) Computing the posterior distribution of $\bbeta$ in the feature-based representation of a Gaussian process,
and (2) Computing the posterior distribution of the integrated partial derivative metric. 

First consider (1), it involves computing two closed-form updates (for posterior mean and variance) over the training data in mini-batches for 1 epoch. The algorithm has a linear complexity with respect to data size. 

\begin{algorithm}[htb]
   \caption{Posterior Computation, Feature-based Representation of Gaussian Process}
   \label{alg:gp_primal}
\begin{algorithmic}[1]
   \State {\bfseries Input:} Training data mini-batches $\{(\bX_m, \by_m)\}_{m=1}^M$. Fixed feature function $\phi:\Xsc \rightarrow \real^D$.
   \State {\bfseries Output:} Posterior mean and variance $\mathbb{E}[\bbeta]_{D \times 1}$, $\text{Cov}[\bbeta]_{D \times D}$.
   \State {\bfseries Initialize:} Feature-label product matrix $\bP = \bzero_{D \times 1}$, covariance matrix $\bSigma = \bzero_{D\times D}$
   \For{$m=1$ {\bfseries to} $M$}
   \State Compute minibatch feature representation $\bPhi_m = [\phi(\bx_1), \dots, \phi(\bx_{n_m})]_{n_m \times D}$
   \State Update $\bP = \bP +  \Phi_m^\top(\by_m - \Phi_m\bmu)/ \sigma^2$ 
   \State Update $\bSigma = \bSigma - \bSigma\Phi_m^\top(\sigma^2\bI + \bPhi_m \bSigma \bPhi_m^\top)^{-1}\Phi_m\bSigma$ 
   $\qquad\qquad\qquad\qquad  \triangleright$ \cref{eq:primal_cov_post}
   \EndFor
  \State Compute $\text{Cov}[\bbeta] = \bSigma_{\bbeta}=\bSigma$
  $\qquad\qquad\qquad\qquad\;\;\;  \triangleright$ \cref{eq:primal_coef_post}
  \State Compute $\mathbb{E}[\bbeta] = \bmu + \bSigma_{\bbeta}\bP$
  $\qquad\qquad\qquad\qquad\quad\;  \triangleright$ \cref{eq:primal_coef_post}
\end{algorithmic}
\end{algorithm}

As shown, during mini-batch computation, the algorithm computes the posterior mean and precision matrix by linearly accumulating the statistic $\Phi_m^\top(\by_m - \Phi_m\bmu)$, and performs one computation in the end to obtain the $\mathbb{E}[\bbeta]$. As a result, the space complexity of the algorithm is $O(D^2)$ (for the covariance matrix) and time complexity of the algorithm is $O(nD^3)$ for the matrix inversion. In large-scale applications, the model dimension $D$ is usually fixed and is significantly smaller than the data size $n$, leading to a linear-time algorithm. Notice that in actual implementation, this algorithm can be made much more efficient (i.e., $O(nD^2)$) by changing how covariance matrix is computed. We introduce this improved algorithm at the end of this section in \cref{alg:gp_primal_fast}.

Now consider (2). Given the posterior of $\bbeta$ from Algorithm \ref{alg:gp_primal}, the posterior distribution of the integrated partial derivative metric
$\psi_{j}(f)=\Vert\frac{\partial}{\partial \bx^j} f\Vert_{n}^{2}=\frac{1}{n}\bbeta^\top \frac{\partial \Phi}{\partial \bx_i^j}  \frac{\partial \Phi^\top}{\partial \bx_i^j} \bbeta$ can be computed conveniently by sampling $\bbeta$ from its posterior. 

\begin{algorithm}[htb]
   \caption{Posterior Computation, Integrated Partial Derivative Metric}
   \label{alg:ipd_primal}
\begin{algorithmic}[1]
   \State {\bfseries Input:} Data $\bX^*$ with size $n^*$. Posterior distribution $\mathcal{M}\mathcal{V}\mathcal{N}(\mathbb{E}[\bbeta]_{D \times 1},\text{Cov}[\bbeta]_{D \times D})$.
   \State {\bfseries Output:} Posterior samples of $\psi_{j}(f)$ of size $K$: $\{\psi_j(f)_k \}_{k=1}^K$
   \State Sample $\{\bbeta_k\}_{k=1}^K \sim \mathcal{M}\mathcal{V}\mathcal{N}(\mathbb{E}[\bbeta],\text{Cov}[\bbeta])$
   \State Compute partial derivative feature matrix $[\frac{\partial \Phi}{\partial \bx^j}]_{D \times N^*} = [\partial \phi(\bx_1)^\top, \dots, \partial \phi(\bx_{N^*})^\top]^\top$
   \State Compute $\bG_{j, D\times D} = \frac{\partial \Phi}{\partial \bx^j}  \frac{\partial \Phi^\top}{\partial \bx^j}$
   \State Compute $\psi_{j}(f)_k = \frac{1}{N^*} \bbeta_k^\top \bG_j \bbeta_k$ for $k=1,\dots,K$
   $\qquad\qquad\qquad\qquad\quad\;  \triangleright$ \cref{eq:varimp_post}
\end{algorithmic}
\end{algorithm}

When the data size is large, the $\bG_j$ matrices can usually be computed as part of \cref{alg:gp_primal} by accumulating gradient partial derivative matrices $\bG_j = \bG_j + \frac{\partial \Phi_m}{\partial \bx^j}  \frac{\partial \Phi_m^\top}{\partial \bx^j}$. The time complexity of the algorithm is $O(D^2 n^*)$ which is again a linear-time algorithm with respect to data size $n^*$. When the data size is extremely large, one can consider reduce computational burden by subsampling from $\bX^*$, which is equivalent to performing a Monte Carlo approximation to the integration over the empirical measure (\cref{eq:varimp}).

Finally, we present a more efficient implementation of \cref{alg:gp_primal}, which improved the run time from $O(nD^3)$ to $O(nD^2)$ by changing how covariance matrix is computed during minibatch accumulation:
\begin{algorithm}[htb]
   \caption{Posterior Computation, Feature-based Representation of Gaussian Process  (Version 2)}
   \label{alg:gp_primal_fast}
\begin{algorithmic}[1]
   \State {\bfseries Input:} Training data mini-batches $\{(\bX_m, \by_m)\}_{m=1}^M$. Fixed feature function $\phi:\Xsc \rightarrow \real^D$.
   \State {\bfseries Output:} Posterior mean and variance $\mathbb{E}[\bbeta]_{D \times 1}$, $\text{Cov}[\bbeta]_{D \times D}$.
   \State {\bfseries Initialize:} Feature-label product matrix $\bP = \bzero_{D \times 1}$, precision matrix $\bS = \bI_{D\times D}$
   \For{$m=1$ {\bfseries to} $M$}
   \State Compute minibatch feature representation $\bPhi_m = [\phi(\bx_1), \dots, \phi(\bx_{n_m})]_{n_m \times D}$
   \State Update $\bP = \bP +  \Phi_m^\top(\by_m - \Phi_m\bmu)/ \sigma^2$ 
   \State Update $\bS = \bS + \Phi_m^\top\Phi_m / \sigma^2$ 
   \EndFor
  \State Compute $\text{Cov}[\bbeta] = \bSigma_{\bbeta} = \bS^{-1}$
  $\qquad\qquad\qquad\qquad  \triangleright$ \cref{eq:primal_coef_post}
  \State Compute $\mathbb{E}[\bbeta] = \bmu + \bSigma_{\bbeta}\bP$
  $\qquad\qquad\qquad\qquad\quad\;  \triangleright$ \cref{eq:primal_coef_post}
\end{algorithmic}
\end{algorithm}

As shown, during mini-batch computation, the algorithm computes the posterior mean and precision matrix by linearly accumulating two statistics $\Phi_m^\top(\by_m - \Phi_m\bmu)$ and $\Phi_m^\top\Phi_m / \sigma^2$, and performs one matrix inversion in the end to obtain the covariance matrix $\bSigma_{\bbeta}$ (hence even more efficient than the Woodbury update formula introduced in \cref{alg:gp_primal}, which requires an inversion for every single update step). As a result, the space complexity of the algorithm is $O(D^2)$ (for the covariance matrix) and time complexity of the algorithm is $O(nD^2 + D^3)$. Since in practice, the model dimension $D$ is usually fixed and much smaller than $n$, the time complexity is in fact $O(n D^2)$,

\newpage
\section{Featurized Representation of ML Models}
\label{sec:feature_gp_examples}
The second key advantage of the feature-based representation (\ref{eq:gp_primal}) is its generality: a wide range of machine learning models can be written in term of the feature-based form $f(\bx)=\phi(\bx)^\top\bbeta$ \cite{rahimi_random_2007, davies2014random, lee2017deep}, making the Gaussian process a unified framework for quantifying model uncertainty with a wide array of modern machine learning models. 
This section enumerates a few important examples:

\textbf{Generalized Additive Models (GAM)}. For a regression task with $d$ input features, a generalized additive model (GAM) has the form
$f(\bx)=\beta_0 + \sum_{j=1}^d \beta_j h_j(\bx^j)$, where $h_j's$ are flexible functions (e.g., splines) with bounded norm \cite{hastie2009elements}. GAM induces a $d$-dimensional feature representation (\citep{hastie2009elements}, Chapter 9):
$$\phi(\bx)_{d \times 1} = [1, h_1(\bx^1), \dots, h_d(\bx^d)],
$$ 
where $h_j's$ are usually spline functions that are differentiable. In the special case where all $h_j's$ are identity functions, GAM reduces to a linear model, and the corresponding $f=\phi(\bx)^\top\bbeta$ becomes a \gls{GP} with linear kernel.

\textbf{Decision Trees}.
By partitioning the whole feature space into $D$ cells $\Xsc= \cup_{j=1}^D \Xsc_j$, a decision tree model essentially induces a one-hot feature map, e.g., 
$$\phi(\bx)_{D \times 1}=[0, \dots, 1, \dots, 0],$$
where each element is a indicator function $\mathbbm{1}(\bx \in \Xsc_j)$ for whether the data point $\bx$ falls into the $j^{th}$ cell (\cref{feature-exp}). This connection is crucial for extending Gaussian process treatment to tree models. \cref{sec:fdt} introduce this formulation in more detail. Following the same construction, the features learned by the majority of partition-based learning methods (e.g., CART, PRIM, etc.) can be used to construct Gaussian process kernels.

\textbf{Random Feature Models}. 
The random-feature model takes the form: 
$$\phi(\bx)_{D \times 1} = \sqrt{2} \sigma(\bW^\top \bx + \bb),$$
where $\bW_{d \times D}$ and $\bb_{D \times 1}$ are frozen weights initialized from i.i.d. samples from certain fixed distributions, and $\sigma$ is an activation function. For example, in the case of classic random Fourier features whose inner product approximates the RBF kernel, we have $\sigma(\cdot)=\cos(\cdot), \bW \stackrel{iid}{\sim} N(0, 1), \bb \stackrel{iid}{\sim} Unif(0, 2\pi)$ \cite{liu2021random}. Although first introduced as a scalable approximation to \gls{GP} models equipped with certain kernels (e.g., radial basis function (RBF)), the modern literature treats it as a standalone class of models with its own unique set of theoretic guarantees \citep{Mei2019TheGE, pmlr-v119-jacot20a, NIPS2008_0efe3284}. 

\textbf{(Deep) Neural Networks}.
For a trained $L$-layer neural network of the from $f(\bx)=\bbeta^\top g_L \cdot g_{L-1} \dots \cdot g_0(\bx)$ with $g_l(\bx)=\sigma_l(\bW_l^\top\bx + \bb_l)$, the last-layer representation function 
$$\phi(\bx)=g_L \cdot g_{L-1} \dots \cdot g_1(\bx)$$ 
can be understood as the feature map. Then, the feature map can be used to construct the Gaussian process kernel $k(\bx, \bx')=\phi(\bx)^\top \phi(\bx')$. This approach was studied extensively in prior literature, due to a neural network's appealing ability in learning an effective representation for the task at hand \citep{hinton2007using, calandra2016manifold}. Works like \citep{wilson2016deep, wilson2016stochastic, liu2020simple} further extended this in the context of modern deep learning.

\textbf{Ensembles.} An ensemble model of linear models, trees, or neural networks can be written as a mixture of Gaussian processes. Specifically, an ensemble model can be written as $f(\bx)=\sum_{m=1}^M \alpha_m h_m(\bx)$, where $h_m's$ are weak learners such as linear models, trees, or neural networks, and $\alpha_m$ are model weights that are either learned or set to uniform $\frac{1}{M}$. This formulation covers well-known examples such as AdaBoost, boosted trees, and random forests 
\cite{
hastie2009elements}. 
As introduced above, since many classic weak learners $h_m=\phi_m(\bx)^\top\bbeta_m$ induces a Gaussian process with kernel $k_m$ via their feature representation $k_m(\bx, \bx') = \phi_m(\bx')^\top\phi_m(\bx)$, the full ensemble model 
induces a mixture of Gaussian processes with fixed mixing weights dictated by the ensemble weights $\{\alpha\}_{m=1}^M$. That is, the ensemble induces a Bayesian model $f'(\bx)=\sum_{m=1}^M \alpha_m h'(\bx)$ where $\alpha_m$'s are fixed constants and $h_m'(\bx)$'s are Gaussian process models with prior $\Gsc\Psc(0, k_m)$. In the actual implementation, we fit each of the individual \gls{GP} model $h_m'(\bx)$ following exactly how it is done in the original ensemble model. For example, for random forest models, we fit each $h_m'(\bx)$ models independently with respect to the original label $y$. While not a focus of this work, for gradient boosting models, we fit $h'_m$'s recursively with respect to the residual $y - \sum_{l<m}\alpha_l h'_l(\bx)$ \citep{sigrist2020gaussian}.

\clearpage
\section{Proof for Posterior Convergence}
\label{sec:proof_convergence}

\textbf{Proof for \cref{thm:1}}\\
Recall the list of technical conditions: 
\begin{enumerate}[i)]
\item \textbf{(Convergence of Prediction Function $f$)} The posterior distribution $\Pi_n(f)$ converges toward $f_0$ at a rate of $\epsilon_n$. (Note that in nonparametric learning setting, this rate is not faster than $O_p(n^{-\frac{1}{2}})$ which is the optimal parametric rate);
\item \textbf{(Well-conditioned Derivative Functions)} $D_j: f \to \frac{\partial}{\partial \bx^j} f$ the differentiation operator is bounded: $\Vert D_j \Vert_{op}^2 = \inf \{C\geq 0: \Vert D_j f\Vert_2^2 \leq C \Vert f \Vert_2^2, \text{ for all } f \in \Hsc_{\phi}\}$; 
\end{enumerate}

\begin{proof}
Denote $A_n = \{f : \Vert f - f_0\Vert^2_n > M_n \epsilon_n \}$ and $B_n = \{f : |\psi_j(f) - \Psi_j(f_0)| > M_n \epsilon_n \}$, then showing the statement in \cref{thm:1} is equivalent to showing $\Pi_n(B_n) \rightarrow 0$.

Specifically, we assume below two facts hold:
\begin{enumerate}
\item[\textbf{Fact 1}.] $|\psi_j(f) - \psi_j(f_0)|\leq \Vert D_jf - D_jf_0\Vert_n^2$
\item[\textbf{Fact 2}.]  
$\sup_{j \in \{1, \dots, d\}}|\psi_j(f_0) - \Psi_j(f_0)| \lesssim \Vert f - f_0\Vert^2_n$
\end{enumerate} 
Because if the above facts hold, we then have
\begin{align*}
\sup_{j\in \{1, \dots, d\}}|\psi_j(f) - \Psi_j(f_0)| &\leq 
\sup_{j\in \{1, \dots, d\}}|\psi_j(f) - \psi_j(f_0)| + \sup_{j\in \{1, \dots, d\}}|\psi_j(f_0) - \Psi_j(f_0)| \\
&\leq \sup_{j\in \{1, \dots, d\}}\Vert D_jf - D_jf_0\Vert^2_n + \sup_{j\in \{1, \dots, d\}}|\psi_j(f_0) - \Psi_j(f_0)|\\
& \leq \sup_{j\in \{1, \dots, d\}}\Vert D_jf - D_jf_0\Vert^2_2 + O_p(n^{-\frac{1}{2}}) + 
\sup_{j\in \{1, \dots, d\}}|\psi_j(f_0) - \Psi_j(f_0)| \\
& \leq C  \Vert f - f_0\Vert^2_2 + O_p(n^{-\frac{1}{2}}) +
\sup_{j\in \{1, \dots, d\}}|\psi_j(f_0) - \Psi_j(f_0)| \tag{$D_j$ is bounded}\\
& \leq C  \Vert f - f_0\Vert^2_n + O_p(n^{-\frac{1}{2}}) +
\sup_{j\in \{1, \dots, d\}}|\psi_j(f_0) - \Psi_j(f_0)|\\
& \lesssim \Vert f - f_0\Vert^2_n.
\end{align*}
It then follows that:
\begin{align*}
\mathbb{E}_0\Pi_n \Big(\sup_{j\in \{1, \dots, d\}}|\psi_j(f) - \Psi_j(f_0)| \geq M_n \epsilon_n \Big) 
& \lesssim
\mathbb{E}_0\Pi_n \Big( \Vert f - f_0\Vert^2_n \geq M'_n \epsilon_n \Big)
\rightarrow 0.
\end{align*}

We now show Facts 1 and 2 are true. 
\begin{itemize}
\item \textbf{Fact 1} follows simply from the triangular inequality:
\begin{align*}
|\psi_j(f) - \psi_j(f_0)| 
&= 
\Big| \Vert D_j f \Vert_n^2 - \Vert D_j f_0 \Vert_n^2 \Big|\\
&= 
\max\Big\{
\Vert D_j f \Vert_n^2 - \Vert D_j f_0 \Vert_n^2,\;
\Vert D_j f_0 \Vert_n^2 - \Vert D_j f \Vert_n^2
\Big\}
\leq \Vert D_j f -  D_j f_0 \Vert_n^2.
\end{align*}

\item \textbf{Fact 2} follows from standard Bernstein-type concentration inequality (see, e.g., Lemma 18 of \cite{rosasco_nonparametric_2013}). Specifically, for $| D_jf_0(\bx)|^2$ a random variable with respect to probability measure $P(\bx)$ that is bounded by $L$. Given $n$ iid samples $\{| D_jf_0(\bx_i)|^2\}_{i=1}^n$, recall that $\psi_j(f_0) = \frac{1}{n}\sum_{i=1}^n | D_jf_0(\bx_i)|^2$ and $\Psi(f_0) = \mathbb{E}(| D_jf_0|^2)$, then with probability $1 - \eta$:
\begin{align*}
|\psi_j(f_0) - \Psi(f_0)| \leq 
n^{-\frac{1}{2}} * \big(2\sqrt{2} * L * \log(2/\eta)\big),
\end{align*}
that is, $|\psi_j(f_0) - \Psi(f_0)| \rightarrow 0$ at the rate of $O(n^{-\frac{1}{2}})$. Notice that $O(n^{-\frac{1}{2}})$ is the optimal parametric rate that cannot be surpassed by the convergence speed of the ReLU networks (recall the typical convergence rate is $\epsilon_n \asymp n^{-\frac{\beta}{2\beta + \delta}}* \log(n)^\gamma$ for some $\delta > 0$ and $\gamma > 1$). Therefore we have:
$$
\sup_{j \in \{1, \dots, d\}}|\psi_j(f_0) - \Psi_j(f_0)| \lesssim \Vert f - f_0\Vert^2_n.
$$
\end{itemize}
\end{proof}

\begin{remark}
The sample $L_2$ norm and the expected $L_2$ norm are closed to each other at the rate of $O(n^{-\frac{1}{2}})$. This will happen when $\bx$'s are random, coming from a distribution.
\end{remark}

\begin{remark}
Although not listed explicitly in the main theorem, we also impose a weak technical condition (i.e., Non-trivial Gradient Function) on model function $f$ and true function $f_0$ to avoid certain pathological situations:
\begin{enumerate}
    \item[iii)] \textbf{(Non-trivial Derivative Functions)} Denote $j^* \in \{1, \dots, d^*\}$ the index of the causal variables, and recall $P_\Xsc(\bx)$ the distribution of the input features $\bx$. Then there exists $\epsilon > 0$ such that for all $j^* \in \{1, \dots, d^*\}$, $||D_{j^*} f_0(\bx)||_2^2 > \epsilon$ and $||D_{j^*} f(\bx)||_2^2 > \epsilon$ with non-zero probability.
\end{enumerate}
Note that this condition is weak in that it only requires the partial derivative under model function $f$ and $f_0$ are not zero almost everywhere. For differentiable functions under continuous features, this should be satisfied by definition. This basic technical condition is intended to remove two pathological situations. The first is non-differentiable models (e.g., tree models), whose gradient is zero almost everywhere in the feature space. The second case are the discrete features, where the traditional sense of partial derivative is not well defined. In \cref{sec:non_diff}, we discuss how to incorporate non-differentiable models and discrete features into our framework. Briefly, a non-differentiable model (e.g., partition-based models) can be made differentiable by employing a differentable approximation. For discrete features, we can compute the discrete version of the differentiable operator, e.g., $D_j f(\bx) = f(\bx^j=1, \bx^{-j}) - f(\bx^j=0, \bx^{-j})$ for binary feature where $\bx_{d \times 1}=[\bx^j, [\bx^{-j}]^\top_{(d-1) \times 1}]^\top$ (known as \textit{contrast} in statistics). Notice that this discrete differentiation operator $D_j f(\bx)$ is a linear function of the original prediction function $f$. As a result, the posterior convergence of $\psi_j$ with respect to this operator is again guaranteed by the convergence of the prediction $f$.
\end{remark}

\begin{remark}
Note that our result focuses on posterior concentration of variable importance $\psi_j$, not of prediction function $f$. In fact, the convergence of $\psi_j$ depends on the convergence of the prediction function $f$, as introduced in the assumptions of \ref{thm:1}. In practice, it is up to the practitioners to select a proper prediction model $f$ that has a convergence guarantee for the task at hand. Specifically, we showed that for any model, if its prediction function has a posterior concentration guarantee, its variable importance has a convergence guarantee as well. 
To this end, we notice that majority of popular machine learning methods (e.g., random features, neural networks, tree ensembles) has a posterior concentration guarantee for target functions in certain general function space (e.g., the space of $\alpha$-H\"{o}lder space), given the recent advances in the approximation and convergence guarantees of parametric (finite-dimensional) ML models in both frequentist and Bayesian settings \cite{rovckova2020posterior, wang2020uncertainty, liu_variable_2021, schmidt-hieber_nonparametric_2020}. 

Furthermore, we note that although the ML models covered in our work are not traditional universal kernels \cite{micchelli2006universal}, most of them (e.g., random features, neural networks, tree ensembles) do come with a universal approximation guarantee for an appropriately defined function class \cite{rahimi2008uniform, biau2012analysis, schmidt-hieber_nonparametric_2020}. As a result, the kernel functions defined by these models provide basis functions that span function spaces that are often dense in an infinite-dimensional RKHS, implying that the resulting model can approximate $f_0$ to arbitrary precision \cite{rahimi2008uniform}. Please see \cite{rahimi2008uniform, hornik1989multilayer, biau2012analysis} for specific results for random features, neural networks and random forests.
\end{remark}

\newpage
\section{Proofs for Asymptotic Normality}
\label{sec:proof_bvm}

\begin{lemma}{Functional Delta Method (univariate)} 
\label{lem:3}
Suppose $\Psc_n$ is the empirical distribution of a random sample $X_1, \dots, X_n$ from a distribution $P$, and $\phi$ is a function that maps the distribution of interest into some space. Define the Gateaux derivative 
\[\phi_P '(\delta_x - P) = \frac{\mathrm{d} }{\mathrm{d} t}\mid_{t = 0} \phi((1 - t)P + t\delta_x) = IF_{\phi, P}(x),\]
which is also the Influence Function, and $\gamma^2 = \int IF_{\phi, P}(x)^2 dP$. If integration and differentiation can be exchanged, then 
\[\int \phi_P '(\delta_x - P) dP = 0. \]
Further, if $\sqrt{n} R_n \overset{P}{\rightarrow}0$, where 
\[R_n = \phi(\Psc_n) - \phi(P) - \frac{1}{n}\sum_{i} \phi_P'(\delta_{x_i} - P), \]
then from the Central Limit Theory that
\[\sqrt{n}(\phi(\Psc_n) - \phi(P)) \overset{d}{\rightarrow} \mathcal{N}(0, \gamma^2).\]
\end{lemma}

\vspace{2em}

\begin{lemma}{Functional Delta Method (multivariate)} 
\label{lem:4}
Suppose $\mathbf{\Psc}_n$ is the empirical distribution of a random sample $X_1, \dots, X_n$ from a distribution $P$, and $\bphi: \mathbb{R}^d \to \mathbb{R}^k$. Define the Gateaux derivative 
\[\bphi_P '(\delta_x - P) = \frac{\mathrm{d} }{\mathrm{d} t}\mid_{t = 0} \bphi((1 - t)P + t\delta_x) = IF_{\bphi, P}(x),\]
which is also the Influence Function, and $[\bV_0]_{i, j} = \int \langle [IF_{\bphi, P}(x)]_{i}, [IF_{\bphi, P}(x)]_{j} \rangle dP$. If integration and differentiation can be exchanged, then 
\[\int \bphi_P '(\delta_x - P) dP = 0. \]
Further, if $\sqrt{n} \bR_n \overset{P}{\rightarrow}0$, where 
\[\bR_n = \bphi(\mathbf{\Psc}_n) - \bphi(P) - \frac{1}{n}\sum_{i} \bphi_P'(\delta_{x_i} - P), \]
then from the Central Limit Theory that
\[\sqrt{n}(\bphi(\mathbf{\Psc}_n) - \bphi(P)) \overset{d}{\rightarrow} \mathcal{MVN}(0, \bV_0).\]
\end{lemma}

\vspace{1em}
\textbf{Proof for \cref{thm:2}}\\
To make our assumptions explicit, we list out a collection of easily-satisfied technical conditions. 
\begin{enumerate}[(1)]
\item $f$ is a consistent estimator of $f_0$;
\item $ D_j$ is bounded: $\Vert D_j \Vert_{op}^2 = \inf \{C\geq 0: \Vert D_j f\Vert_2^2 \leq C \Vert f \Vert_2^2, \text{ for all } f \in \Hsc_{\phi}\}$.
\item $f_0$ is square-integrable over the support of $X$ and $\Vert f_0\Vert_2=1$;
\item $\text{\texttt{rank}}(H_j) = o_p(\sqrt{n})$;
\end{enumerate}
\begin{proof}
Since $H_j = D_j^\top D_j$, Condition (2) is equivalent to the largest eigenvalue of $H_j$ being bounded, i.e., $\lambda_{max}(H_j)=O_p(1)$. From the definition in \cref{eq:def_psi}, we have 
\[\psi_j'(f) = \frac{\partial}{\partial f} \psi_j(f)= \frac{2}{n}H_jf. \]
Define a mean functional $m: F \to E(F)$, where $F$ is the distribution. Then in our case, $f_0 = E(F) = m(F)$. According to \cref{lem:3}, we have 
\[\psi_j(f_0) = \psi_j(E(F)) = \psi_j(m(F)) = \phi(F),\]
i.e., $\phi(\cdot) = \psi_j(m(\cdot))$.
Therefore,
\begin{align*}
\phi_F '(\delta_y - F)  &= \psi_j'(m(\delta_y - F)) \\
&= \frac{\mathrm{d} }{\mathrm{d} t}\mid_{t = 0} \psi_j(m((1 - t)F + t\delta_y)) \\
&= \frac{\mathrm{d} }{\mathrm{d} t}\mid_{t = 0} \psi_j((1 - t)f_0 + ty) \\
&= \frac{\mathrm{d} }{\mathrm{d} t}\mid_{t = 0} \frac{1}{n}[(1 - t)f_0 + ty]^\top H_j [(1 - t)f_0 + ty] \\
&= \frac{2}{n}(y - f_0)^\top H_j f_0 \\
&= IF_{\phi, F}(y).
\end{align*}
On the other hand, 
\begin{align*}
\gamma^2 &= \int IF_{\phi, F}(y)^2 dF\\
&= 4\int \frac{1}{n}\cdot f_0^\top H_j (y - f_0) (y - f_0)^\top H_j f_0 \cdot \frac{1}{n}dF \\
&= 4\sigma^2 \Vert H_j f_0 \Vert_n^2.
\end{align*}
Moreover, we have 
\[\int \phi_F '(\delta_y - F) dF = \frac{2}{n}\int (y - f_0)^\top H_j f_0 dF = 0, \]
and 
\begin{align}
\sqrt{n}R_n &= \sqrt{n}[\phi(\Fsc_n) - \phi(F) - \frac{1}{n}\sum_{i} \phi_F'(\delta_{y_i} - F)] \nonumber \\
&= \sqrt{n}[\psi_j(f) - \psi_j(f_0)\nonumber  - \frac{1}{n}\cdot \frac{2}{n}\sum_{i}(y_i - f_{0, i})^\top [H_j f_0]_i]\nonumber \\
&= \sqrt{n}[\frac{1}{n}\cdot (f^\top H_j f - f_0^\top H_jf_0)\nonumber - \frac{1}{n}\cdot \frac{2}{n}(y - f_0)^\top H_j f_0]\nonumber \\
&= \frac{1}{\sqrt{n}}[f^\top H_j f - f^\top H_j f_0 + f^\top H_j f_0  - f_0^\top H_jf_0 - \frac{2}{n}(y - f_0)^\top H_j f_0]\nonumber \\
&= \frac{1}{\sqrt{n}}[(f - f_0)^\top H_j (f + f_0) - \frac{2}{n}(y - f_0)^\top H_j f_0]\nonumber \\
&= \frac{1}{\sqrt{n}}[(f - f_0)^\top H_j (f - f_0) + 2(f - f_0)^\top H_jf_0 - \frac{2}{n}(y - f_0)^\top H_j f_0]\nonumber \\
&= \frac{1}{\sqrt{n}}[\hat{\epsilon}_n^\top H_j \hat{\epsilon}_n + 2\hat{\epsilon}_n^\top H_jf_0 - \frac{2}{n}(y - f_0)^\top H_j f_0] \label{unit-2} \\
&= \frac{1}{\sqrt{n}}o_p(\sqrt{n})\label{unit-3} \\
&= o_p(1) \overset{P}{\rightarrow} 0, \nonumber 
\end{align}
where $\hat{\epsilon}_n = f - f_0$. We can prove the result from  \cref{unit-2} to \cref{unit-3} as following:
Denote $k=\text{\texttt{rank}}(H_j)$, then the eigendecomposition of $H_j$ is $H_j = U_j \Lambda U_j^\top$, with $U_j=[\bu_1, \dots, \bu_k]$ a $n\times k$ orthogonal matrix and $\Lambda$ a $k\times k$ diagonal matrix with elements $\{\lambda_i\}_{i=1}^k$ being the eigenvalues of $H_j$, then define 
\[\bv=U_j^\top \hat{\epsilon}_n=\begin{bmatrix}
\bu_1^\top \hat{\epsilon}_n\\ 
\vdots\\ 
\bu_k^\top \hat{\epsilon}_n
\end{bmatrix}. \]
Therefore,
\begin{align*}
\hat{\epsilon}_n^\top H_j \hat{\epsilon}_n &= \bv^\top \Lambda \bv = \sum_{i=1}^k \lambda_i v_i^2\\
&\leq \lambda_{max}(H_j) \sum_{i=1}^k v_i^2\\
&=  \lambda_{max}(H_j) \sum_{i=1}^k \bu_i^\top \hat{\Sigma}_n \bu_i\\
&\leq \lambda_{max}(H_j) \sum_{i=1}^k \lambda_{max}(\hat{\Sigma}_n)\\
&= k \cdot \lambda_{max}(H_j) \cdot \lambda_{max}(\hat{\Sigma}_n)\\
&= o_p(\sqrt{n}) \cdot O_p(1) \cdot O_p(1)\\
&= o_p(\sqrt{n}),
\end{align*}
where $E(\hat{\epsilon}_n) = \mathbf{0}$, $\cov(\hat{\epsilon}_n) = \hat{\Sigma}_n$, and $\lambda_{max}(\hat{\Sigma}_n)$ is the largest eigenvalue of $\hat{\Sigma}_n$.

On the other hand, $2\hat{\epsilon}_n^\top H_jf_0=o_p(\sqrt{n})$ because $f$ is a consistent estimator of $f_0$. Moreover, since $y_i - f_{0, i}=O_p(1)$, we know $\frac{2}{n}(y - f_0)^\top H_j f_0=o_p(\sqrt{n})$. So,
\begin{align}
\hat{\epsilon}_n^\top H_j \hat{\epsilon}_n + 2\hat{\epsilon}_n^\top H_jf_0 - \frac{2}{n}(y - f_0)^\top H_j f_0=o_p(\sqrt{n}) \label{op}
\end{align}

Therefore, by \cref{lem:3}, we have 
\[\sqrt{n}(\psi_j(f) - \psi_j(f_0)) \overset{d}{\rightarrow} \mathcal{N}(0, 4\sigma^2\Vert H_j f_0 \Vert_n^2). \]
\end{proof}

\begin{theorem}[Asymptotic Distribution of Variable Importance (multivariate)]
Suppose $y_i = f_0(\bx_i) + e_i, \; e_i \overset{i.i.d.}{\sim} \mathcal{N}(0, \sigma^2), \; i=1, \dots, n$. Denote $\bpsi=[\psi_1, \dots, \psi_d]$ for $\psi_j$ as defined in \cref{eq:def_psi}. 
If the following conditions are satisfied:
\begin{enumerate}[i)]
\item $\text{\texttt{rank}}(H_j) = o_p(\sqrt{n}), j=1, \dots, d$;
\item $f_0$ is square-integrable over the support of $X$ and $\Vert f_0\Vert_2=1$;
\item $f$ is a consistent estimator of $f_0$;
\item $ D_j$ is bounded: $\Vert D_j \Vert_{op}^2 = \inf \{C\geq 0: \Vert D_j f\Vert_2^2 \leq C \Vert f \Vert_2^2, \text{ for all } f \in \Hsc_{\phi}\}$.
\end{enumerate}

Then $\bpsi(f)$ asymptotically converges toward a multivariate normal distribution surrounding $\bpsi(f_0)$, i.e.,
\[\sqrt{n}(\bpsi(f) - \bpsi(f_0)) \overset{d}{\rightarrow} \mathcal{MVN}(\mathbf{0}, \bV_0), \]
where $\bV_0$ is a $d\times d$ matrix such that $[\bV_0]_{j1, j2}=4\sigma^2 \langle H_{j1} f_0, H_{j2} f_0 \rangle_n$.
\end{theorem}

\begin{proof}
Define a mean function $m: F \to E(F)$, where $F$ is the distribution. Then in our case, $f_0 = E(F) = m(F)$. According to \cref{lem:4}, we have 
\[[\bpsi (f_0)]_j = \psi_j(E(F)) = \psi_j(m(F)) = [\bphi(F)]_j,\]
i.e., $\bphi(\cdot) = \bpsi(m(\cdot))$ and $[\bphi(\cdot)]_j = \psi_j(m(\cdot))$, where $\bphi: \Rsc \to \Rsc^P$.
Therefore,
\begin{align*}
[\bphi_F '(\delta_y - F)]_j  &= \psi_j'(m(\delta_y - F)) \\
&= \frac{\mathrm{d} }{\mathrm{d} t}\mid_{t = 0} \psi_j(m((1 - t)F + t\delta_y)) \\
&= \frac{\mathrm{d} }{\mathrm{d} t}\mid_{t = 0} \psi_j((1 - t)f_0 + ty) \\
&= \frac{\mathrm{d} }{\mathrm{d} t}\mid_{t = 0} \frac{1}{n}[(1 - t)f_0 + ty]^\top H_j [(1 - t)f_0 + ty] \\
&= \frac{2}{n}(y - f_0)^\top H_j f_0 \\
&= [IF_{\bphi, F}(y)]_j.
\end{align*}
On the other hand, 
\begin{align*}
[\bV_0]_{j1, j2} &= \int \langle [IF_{\bphi, F}(y)]_{j1}, [IF_{\bphi, F}(y)]_{j2} \rangle dF\\
&= 4\int \frac{1}{n} \cdot f_0^\top H_{j1} (y - f_0) (y - f_0)^\top H_{j2} f_0 \cdot \frac{1}{n} dF \\
&= 4\sigma^2 \langle H_{j1} f_0, H_{j2} f_0 \rangle_n.
\end{align*}
Moreover, we have 
\[[\int \bphi_F '(\delta_y - F) dF]_j = \frac{2}{n}\int (y - f_0)^\top H_j f_0 dF = 0, \]
and 
\begin{align}
[\sqrt{n}\bR_n]_j &= \sqrt{n}[[\bphi(\Fsc_n)]_j - [\bphi(F)]_j - \frac{1}{n}\sum_{i} [\bphi_F'(\delta_{y_i} - F)]_j]\nonumber \\
&= \frac{1}{\sqrt{n}}[f^\top H_j f - f_0^\top H_jf_0 - \frac{2}{n}(y - f_0)^\top H_j f_0]\nonumber \\
&= \frac{1}{\sqrt{n}}[(f - f_0)^\top H_j (f - f_0) + 2(f - f_0)^\top H_jf_0 - \frac{2}{n}(y - f_0)^\top H_j f_0]\nonumber \\
&= \frac{1}{\sqrt{n}}[\hat{\epsilon}_n^\top H_j \hat{\epsilon}_n  + 2\hat{\epsilon}_n^\top H_jf_0 - \frac{2}{n}(y - f_0)^\top H_j f_0] \label{mult-4} \\
&= \frac{1}{\sqrt{n}}o_p(\sqrt{n}) \label{mult-5} \\
&= o_p(1) \overset{P}{\rightarrow} 0, \nonumber 
\end{align}
where $\hat{\epsilon}_n = f - f_0$ and the reason from \cref{mult-4} to \cref{mult-5} is because of \cref{op}. Therefore, by \cref{lem:4}, we have 
\[\sqrt{n}(\bpsi(f) - \bpsi(f_0)) \overset{d}{\rightarrow} \mathcal{MVN}(\mathbf{0}, \bV_0), \]
where $\bV_0$ is a $d\times d$ matrix such that $[\bV_0]_{j1, j2}=4\sigma^2\langle H_{j1} f_0, H_{j2} f_0 \rangle_n$.
\end{proof}

\clearpage
\section{Additional Theoretical Discussions}
\label{sec:thm_discuss}

\subsection{Lipschitz condition of ML models}
\label{sec:thm_discuss_lipschitz}


The condition of the differentiation operator $D_j: f \to \frac{\partial}{\partial \bx^j} f$  being bounded is guaranteed if the functional $f$ is differentiable and Lipschitz, so that $\frac{|f(\bx_1) - f(\bx_2)|}{||\bx_1 - \bx_2||_2} \leq C$ where $||\bx_1 - \bx_2||_2$ is the $L_2$ metric in $\Xsc$. Fortunately, a wide range of machine learning models (under proper regularity condition) satisfy the Lipschitz condition. Below we consider a few important examples:

\textbf{Generalized Additive Models (GAM)}. The generalized additive models is often written as the sum of smooth functions, 
$$f(\bx)=\beta_0 + \sum_{j=1}^d \bbeta_j  h_j(\bx^j).$$ As a result,  $f$ is Lipschitz if every individual smooth function $h_j$ is Lipschitz. To this end, we notice that in the GAM algorithm, the $h_j$'s are commonly estimated under a smoothness constraint in terms of its second derivatives \citep{wood2006generalized} $\psi_{2,j}=\int_\Xsc |\frac{\partial^2}{\partial (\bx^j)^{2}} f(\bx^j)|^2 d\bx$, which essentially imposes an upper bound on the first-order partial derivatives $\deriv{\bx^j}f(\bx^j)$ (assuming bounded support). As a result, the Lipschitz of GAM function is guanranteed by the virtue of its smoothing constraints.

\textbf{Decision Trees}. Interestingly, we can understand the Lipschitz condition of a tree-type model by investigating its model structure from a neural network lens. Specifically, 
for a depth-$L$ tree model with $D$ leaf nodes,
\cite{karthikeyan_learning_2021} shows that a it can be written in the form of a neural network layer:
\begin{align*}
f(\bx) &= \sum_{k=1}^{D} q_k(\bx) \bbeta_k, 
\quad \mbox{where} \quad
q_k(\bx) = \sigma_{\mbox{\texttt{\tiny{step}}}} \big(\sum_{l=1}^{L} \sigma_{\mbox{\texttt{\tiny{step}}}}\big((\bx^\top \bw_{k, I(l, k)}+b_{k, I(l, k)})S(l, k) \big)-h\big).
\end{align*}
Here $q_k(\bx)$ is a re-parametrization for the indicator function of whether $\bx$ belongs to the $k^{th}$ leaf node, i.e., $\prod_{l=1}^L \sigma_{\mbox{\texttt{\tiny{step}}}} \Big[ (\bx^\top \bw_{k, I(l, k)}+b_{k, I(l, k)})S(l, k) \Big]$. (See \cref{sec:nn_dt} or Section 3 of \cite{karthikeyan_learning_2021} for full detail.)
Briefly, $\sigma_{\mbox{\texttt{\tiny{step}}}}(x)=I(x>0)$ is the step function, $I(l, k)$ indicates the index for the ancestor node for the $k^{th}$ leaf at depth $l$, and $S(l,k) \in \{-1, 1\}$ is a sign function for whether $k^{th}$ leaf is the right subtree of node $I(l, k)$. As a result, $q_k(\bx)$ measures whether $\bx$ satisfies every ancestry decision rules $I\big[S(l,k)(\bx^\top \bw_{k, I(l, k)} - b_{k, I(l, k)})>0\big]$ at every level $l \in \{1, \dots, L\}$, where $\bw_{k, I(l, k)}$ is a $d \times 1$ one-hot vector indicating the index of feature being selected by that node. 

As a result, the tree model can be viewed as a wide 1-hidden layer neural network model with bounded activation function $\sigma_{\mbox{\texttt{\tiny{step}}}}$ and hidden weights bounded within $[-1, 1]$, which leads to a Lipschitz function. Furthermore, the function $f(\bx)$ remains Lipschitz if we replace the non-differentiable $\sigma_{\mbox{\texttt{\tiny{step}}}}$ with a differentiable activation function that is Lipschitz (e.g., \cref{sec:non_diff}).

\textbf{Random Feature Models}. The random feature methods are also structured the same way as $f(\bx) = \sigma(\bW^\top \bx + \bb)$, where $\bW$ are frozen weights that are independently sampled from distribution with finite second moments (e.g., Gaussian distribution), and $\sigma$ is a trignomitric function ($sin$ and $cos$), or common activation functions that are used in the neural networks \citep{pmlr-v84-choromanski18a, liu2021random}. As a result, $f(\bx)$ is also Lipschitz with high probability. In practice, the Lipschitz condition can be guaranteed in absolute terms by truncating the individual terms in $\bW$ to be within a range $[-C, C]$ (e.g., $C=4.$ for $W \stackrel{iid}{\sim} N(0, 1)$), which often leads to almost identical performance.

\textbf{(Deep) Neural Networks}. Both deep neural networks and random-feature models can be written as a composition of functions: 
$$f(\bx)=\bbeta^\top g_L \cdot g_{L-1} \dots \cdot g_1(\bx),
\quad \mbox{where} \quad
g_l(\bx)=\sigma(\bW_l^\top \bx + \bb_l).$$
As a result, due to chain rule, $f$ is Lipschitz if each of its individual layer $g_l$ is Lipschitz \citep{virmaux2018lipschitz}. Similarly, since the layer function $g_l$ is a composition of the linear function $\bW_l^\top \bx + \bb_l$ and a non-linear activation $\sigma$, $g_l$ is guanranteed to be Lipschitz if both the linear function is bounded with high probability, and the activation function $\sigma$ is also Lipschitz. In the context of neural network learning, this is often satisfied by the common practice of imposing $L_1$ or $L_2$ regularization to neural network weights, and by using standard choices of activation functions such as ReLU, leaky ReLU, tanh, etc \citep{virmaux2018lipschitz, liu_gaussian_2019}. 



\subsection{Discussion on \glsfirst{BvM} phenomenon}
\label{sec:thm_discuss_bvm}

\paragraph{Dimensionality of the Derivative Function Space.}

Denote $\Hsc$ the space of model functions spanned by the basis functions $\{b_k(\bx)\}_{k=1}^D$, such that $f(\bx)=\sum_{k=1}^D \alpha_k b_k(\bx)$. Then, the space of partial derivative function is $\Hsc_j = \{\deriv{\bx^j} f | f \in \Hsc \}$. Furthermore, for every element in $\Hsc_j$, we have:
\begin{align*}
    \deriv{\bx^j} f = \sum_{k=1}^D \alpha_k \cdot [\deriv{\bx^j} b_k(\bx)].
\end{align*}
That is, the derivative function space $\Hsc_j$ can be spanned by $\{\deriv{\bx^j}  b_k(\bx)\}_{k=1}^D$, the partial derivatives of the basis functions for the original model space $\Hsc$. Furthermore, since differentiation is a linear operator, the set of linearly independent functions in $\{\deriv{\bx^j}  b_k(\bx)\}_{k=1}^D$ should be equivalent to that in $\{b_k(\bx)\}_{k=1}^D$. As a result, the effective dimensionality of the derivative function space $\Hsc_j$ can be controlled by the effective dimensionality of the model space $\Hsc$. As an aside, for a model space $\Hsc_\phi$ induced by the feature representation $\phi: \Xsc \rightarrow \real^D$, its effective dimensionality can be measured by the rank of the feature matrix $\text{\texttt{rank}}(\Phi)$ for $\Phi=[\phi(\bx_1)^\top, \dots, \phi(\bx_n)^\top]^\top$. Alternatively, in the nonparametric literature, the effective dimensionality can also be measured by model-specific notions of "parameter count", such as the number of leaf partitions of a tree model, or the number of non-zero hidden weights of a deep neural network \citep{schmidt-hieber_nonparametric_2020}.

\paragraph{Effective Dimensionality of Statistical ML Models.}

The \gls{BvM} result (\cref{thm:2}) contains a key condition (4) $H_j = o_p(\sqrt{n})$. As stated in the main text, this condition can be satisfied if the effective dimensionality of model space $\Hsc_{\phi}$ does not grow faster than $o_p(\sqrt{n})$ with respect to the data.

Combined with the posterior convergence condition (i.e., (1)-(2) from \cref{thm:1}), (4) provides a more precise characterization of the convergence behavior of the model $f \in \Hsc_\phi$ for the \gls{BvM} phenomenon to occur.  Loosely, (1)-(2) states that the model $f$ should balance its bias-variance tradeoff well enough so that the overall error rate is controlled at the rate $\epsilon_n$. Then, (4) goes one step further and states that within this bias-variance tradeoff, the variance term must be well managed, which is guaranteed by bounding the model complexity at the rate of $o_p(\sqrt{n})$.

As a matter of fact, for a wide class of ML models, a $o_p(\sqrt{n})$ bound on model complexity is not a stringent requirement, as it only prescribes a growth rate of model complexity with respect to data size. For example, the effective data size can be $C*\sqrt{n}$ for an bounded but very large $C$). 
Interestingly, condition (4) is in fact equivalent or looser than some of the previous \gls{BvM} results obtained for specific ML models.
For example, the decision tree models (e.g., BART) obtains a optimal rate when its number of partitions grow at a rate of $O((n / log n)^{d/2\gamma + d})$ for learning the space of $\gamma$-H\"{o}lder continuous functions with $\gamma > d/2$ \citep{rockova2020semi}, which leads to a $p(\sqrt{n/logn}) < o(\sqrt{n})$ bound on complexity. A similar result also holds for deep learning models, where the number of non-zero model weights is controlled at $O(n^{d/(2\gamma + d)})$ for $\gamma > \frac{d}{2}$ (\citep{wang2020uncertainty}, Theorem 3.2), which also leads to a rate of $o(\sqrt{n})$. 


\clearpage
\section{Incorporating Non-differentiability}
\label{sec:non_diff}

\subsection{Incorporating Non-differentiable Model: \glsfirst{FDT}}
\label{sec:fdt}
Several techniques have been proposed to learn a (soft) tree-structured model using gradient-optimization methods. However, either their accuracies do not match the
state-of-the-art tree learning methods \cite{yang_deep_2018}  or result in models that do not obey the tree structure \cite{irsoy_soft_2012, frosst_distilling_2017, biau_neural_2019, tanno_adaptive_2019}. We propose to translate a learned tree into its exact feature representation, and leverage this representation to unlock a rigorous uncertainty-aware variable selection method that was previously not available for this class of models.

\paragraph{Feature-based Representation of a Decision Tree}
For a certain decision tree $m$ in a learned random forest, consider the following feature map $\phi: \mathbb{R}^d \to \mathbb{R}^{D}$:
\begin{enumerate}
\item The decision tree partitions the whole feature space into $D$ cells $\Xsc=\cup_{k=1}^D \Xsc_k$. Label the cells of the generated partition by $1, 2, \ldots, D$ in arbitrary order.
\item To encode a data point $\bx \in \mathbb{R}^d$, look up the label $y$ of the cell that $\bx$ falls into and set $\phi(\bx)$ to be the (column) indicator vector of whether $\bx \in \Xsc_k$, i.e., $\phi(\bx)=\{\mathbbm{1}(\bx \in \Xsc_k)\}_{k=1}^D$.
\end{enumerate}

The dimensionality $D$ of $\phi$ equals the number of leaf nodes, and each feature mapping $\phi(\bx)$ takes the one-hot form. This feature map $\phi$ induces a kernel 
\begin{align}
k_{dt}(\bx, \bx') := \phi(\bx)^\top \phi(\bx')\nonumber = \begin{cases}
1 & \text{ if } \bx, \bx' \text{ in the same partition cell}\\
0 & \text{ otherwise } 
\end{cases}
\end{align}

As a result, the feature mapping $\phi(\bx)$ defines a featurized decision tree.

\begin{figure}[ht]
\begin{center}
\centerline{\includegraphics[width=0.6\columnwidth]{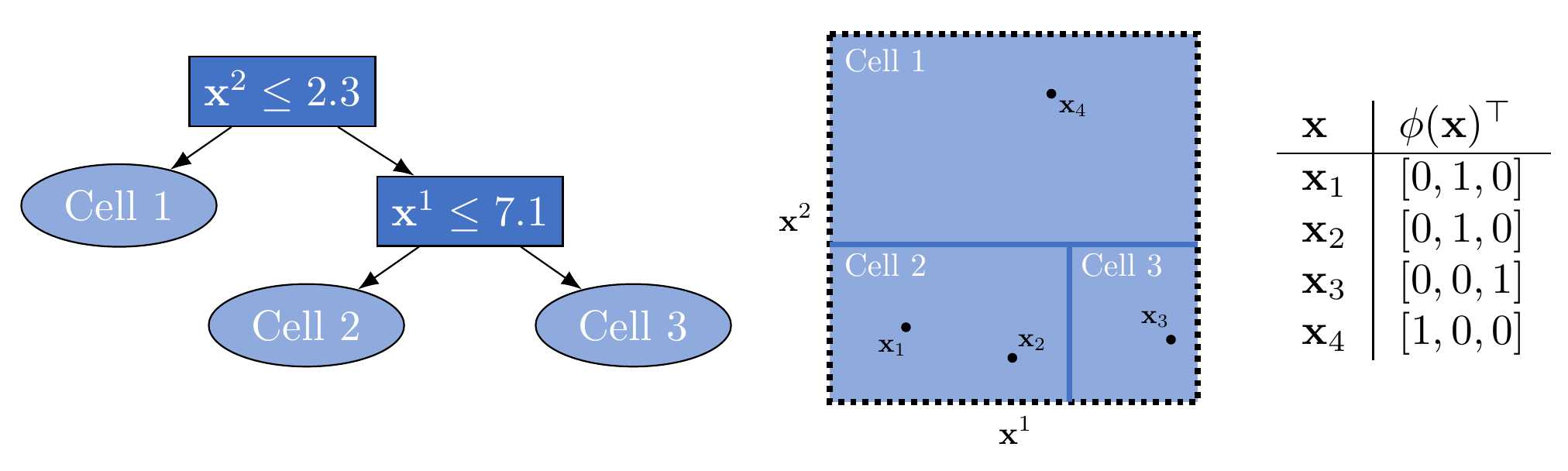}}
\caption{Feature expansion of a decision tree evaluated on $4$ data points in $\mathbb{R}^2$. The middle panel shows the partition of $\mathbb{R}^2$ defined by the decision tree on the left. On the right is the associated feature map.}
\label{feature-exp}
\end{center}
\vskip -0.2in
\end{figure}

As introduced in \cref{sec:varimp}, the solution for $\bbeta$ is $(\Phi^\top \Phi + \sigma^2 \bI_{D})^{-1}\Phi^\top \by$. Note that under the decision tree kernel, $\Phi^\top \Phi=diag(n_1, \dots, n_D)$ is a diagonal matrix of the number of training samples in each leaf cell. Therefore, the time complexity to invert the matrix $(\Phi^\top \Phi + \sigma^2 \bI_{D})$ is $O(D)$. 

\paragraph{Differentiable Approximation}
\label{diff_approx}
The random features generated by \cref{feature-exp} can be written as 
\begin{align*}
\phi(\bx)
&= (\mathbbm{1}(\bx^2 \leq 2.3), \mathbbm{1}(\bx^2 > 2.3, \bx^1 \leq 7.1), \mathbbm{1}(\bx^2 > 2.3, \bx^1 > 7.1)) \\
&= (\mathbbm{1}(\bx^2 \leq 2.3), \mathbbm{1}(\bx^2 > 2.3)\cdot \mathbbm{1}(\bx^1 \leq 7.1),  \mathbbm{1}(\bx^2 > 2.3)\cdot \mathbbm{1}(\bx^1 > 7.1)).
\end{align*}

To calculate variable importance, the indicator function needs to be approximated by a smooth function, so that we can take the derivative with respect to each feature. In this work, we consider approximating the indicator function using the sigmoid function \cite{irsoy_soft_2012}:
\begin{align*}
\mathbbm{1}(x > a) \approx i_c(x > a) = \frac{1}{1+\exp(-c\cdot(x-a))},
\end{align*}
and analogously, $\mathbbm{1}(x \leq a) = 1 - i_c(x > a)$. Here $c$ is a hyperparameter that controls the smoothness of the approximation. A larger $c$ leads to a better approximation to the random forest algorithm, but may result in a non-smooth prediction function which may be undesirable for approximating an continuous regression function $f_0$. 

\subsection{Incorporating Discrete Features}
Compared to the empirical derivative norm, a more principled way to measure the variable importance of a discrete feature is \emph{contrast}, which is the square of the difference in predictions when fixing the feature to a certain value versus fixing it to the other value, while keeping the other features the same. Specifically, we can consider defining a discrete version of the derivative:
\begin{align}
    D_j f = f(\bx^j=1, \bx^{-j}) - f(\bx^j=0, \bx^{-j}),
    \label{eq:disc_diff}
\end{align}
where $\bx^{-j}$ denotes all features with $\bx^{j}$ removed. 

Then, in the case where the feature takes two values, we can set one of them as the reference group with value $0$ and the other group with value $1$,

\vspace{-1em}
\begin{align*}
\Psi_{j}(f)&=\Vert D_j f \Vert_{2}^{2}
= \int_{\bx \in \Xsc}| D_j f|^{2}dP(\bx)
\\ 
&= \int_{\bx \in \Xsc}|f(\bx^j=1, \bx^{-j}) - f(\bx^j=0, \bx^{-j})|^{2}dP(\bx),
\end{align*}

Since $P(\bx)$ is not known from the training observations, $\Psi_{j}(f)$ can be approximated by its empirical counterpart:

\begin{align*}
\psi_{j}(f) &=\Vert D_j f \Vert_{n}^{2}
=\frac{1}{n} \sum_{i=1}^{n}|D_j f|^{2}. 
\\
&=\frac{1}{n} \sum_{i=1}^{n}|f(\bx_i^j=1, \bx_i^{-j}) - f(\bx_i^j=0, \bx_i^{-j})|^{2}. 
\end{align*}

In the case where the feature takes multiple groups, we can calculate the pairwise contrasts and take the $L_2$ norm. Empirically, using contrast for discrete feature improves the performance of variable selection. As contrast is a linear function of the original prediction function $f$, the posterior convergence of $\psi_j$ with respect to this operator is guaranteed by the convergence of the prediction function $f$. Similarly, the \gls{BvM} phenomenon is guaranteed when $D_j f$ is bounded and $H_j = D_j^\top D_j$ has rank $o_p(\sqrt{n})$ (i.e., the similar set of conditions in \cref{thm:2} but with the original $D_j$ replaced by its discrete counterpart \cref{eq:disc_diff}).


\clearpage
\section{Further Experiment Detail}
\label{sec:exp_detail}

\subsection{Methods}
We consider three main classes of models (\cref{tb:methods}).

\begin{enumerate}[I]
    \item \textbf{\glsfirst{RF}}
    \begin{itemize}
        \item \gls{FDT}: Given a trained forest, we quantify variable importance using $\psi_j$ by translating it to an ensemble of \gls{FDT} (\cref{sec:fdt}). We use a variant of random forest here, extra trees \cite{geurts_extremely_2006} since it performs better. We use $50$ trees to build the forest and maximum number of leaf nodes for each tree is $\sqrt{n} \, log(n)$. Throughout our experiment, we fix $c=1$ for continuous features calculated using integrated partial derivatives and fix $c=0.1$ for discrete features calculated using contrasts. 
        We use \texttt{scikit-learn} package in Python to train the random forest.
        \item \gls{impurity} \citep{breiman_classification_1984}: It measures variable importance with their impurity based on the average reduction of the loss function were the variable to be removed. We also use extra trees here. We use $50$ trees to build the forest and maximum number of leaf nodes for each tree is $\sqrt{n} \, log(n)$. We use \texttt{scikit-learn} package in Python to train the random forest.
        \item \gls{knockoff} \citep{candes_panning_2017}: It uses random forest statistics to assess variable importance in our case. We use \texttt{knockoff} package in R to calculate the statistic.
        \item Bayesian additive regression trees (\gls{BART}) \cite{chipman_bart_2010}: It produces a measure of variable importance by tracking variable inclusion proportions, enabling variable selection with a user-defined threshold. We use \texttt{bartMachine} package in R to train the model.
    \end{itemize}
    \item \textbf{(Approximate) Kernel Methods}
    \begin{itemize}
        \item Random Feature Neural Networks (\gls{RFNN}): We apply $\psi_j$ to a random-feature model that approximates a Gaussian process with a RBF kernel \cite{rahimi_random_2007}, and set the number of features to $\sqrt{n} \, log(n)$ to ensure proper approximation of the exact RBF-GP \cite{rudi_generalization_2018}. We choose the lengthscale parameter of RBF-GP from a list of lengthscale candidates $\{5, 10, 16, 23\}$ based on the prediction performance on testing data.
        \item Bayesian kernel machine regression (\gls{BKMR}) \cite{bobb_bayesian_2015}: It is based on a \gls{GP} with exact RBF kernel and spike-and-slab prior, using posterior inclusion probabilities to perform variable selection. We use \texttt{bkmr} package in R to train the model and the number of iterations of the MCMC sampler is set to be $4000$.
        \item Bayesian Approximate Kernel Regression (\gls{BAKR}) \citep{crawford_bayesian_2018}: It is based on random-feature model with a projection-based feature importance measure and an adaptive shrinkage prior, using squared estimates of the parameter coefficients to perform variable selection. We use \texttt{BAKR} repository from the author's GitHub to train the model and the number of iterations of the MCMC sampler is set to be $2000$.
    \end{itemize}
    \item  \textbf{Linear Models}
    \begin{itemize}
        \item \gls{GAM}:  We apply $\psi_j$ to a featurized \gls{GP} representation of the \gls{GAM}, with the prior center $\bmu$ set at the frequentist estimate of the original GAM model obtained from a sophiscated REML procedure \citep{wood2006generalized}. We use \texttt{mgcv} package in R to train the model.
        \item Bayesian Ridge Regression (\gls{BRR}) \cite{hoerl_ridge_1970}: It applies a fixed prior for each feature, using squared estimates of the parameter coefficients to perform variable selection. We use \texttt{BGLR} package in R to train the model and the number of iterations of the MCMC sampler is set to be $2000$.
        \item Bayesian Lasso (\gls{BL}) \cite{park_bayesian_2008}: It developed a Bayesian way to access the Lasso estimate which allows tractable full conditional distributions, using squared estimates of the parameter coefficients to perform variable selection. We use \texttt{BGLR} package in R to train the model and the number of iterations of the MCMC sampler is set to be $2000$.
    \end{itemize}
\end{enumerate}

The results in this paper were obtained using R 4.1.0 or Python 3.7. All experiments were run on a Linux-based high performance computing cluster using SLURM-managed CPU resources.

\clearpage
\subsection{Data}
\label{sec:exp_data}

\paragraph{Outcome-generating function} As discussed earlier, we generate data under the homoscedastic Gaussian noise model $y \sim \mathcal{N}(f_0(\bx), 0.01)$ for different sparse functions $f_0$ and features $\bx$. Given $n\in\{100, 200, 500, 1000\}$ observations in $d\in\{25,50,100 \}$ dimensions, the goal is to model $f_0$ while identifying the $d^*=5$ features on which $f_0$ depends. To this end we report mean squared error (MSE) to quantify prediction performance and \gls{AUROC} scores to quantify variable selection performance. 

We consider four settings of the data-generation function $f_0$:
\begin{enumerate}[1)]
\item \texttt{linear}: a simple linear function $f_0(\bx) = \bx^1 - \bx^2 + \bx^3 + 0.5 \bx^4 + 2\bx^5$;
\item \texttt{rbf}: a Gaussian RBF kernel with length-scale $1$. This kernel represents the space of functions that are smooth (i.e., infinitely differentiable) and have reasonable complexity (i.e., does not have fast-varying fluctuations that are difficult to model);
\item \texttt{matern32}: a mat\'{e}rn $\frac{3}{2}$ kernel with length-scale $1$. Compared to RBF, it has the same degree of complexity but is less smooth, in the sense that it represents the space of once-differentiable functions, but is not necessarily infinitely differentiable;
\item \texttt{complex}: a complicated and non-smooth multivariate function that is outside the \gls{RKHS} $\Hsc$: $f_0(\bx) = \frac{\sin(\max(\bx^1, \bx^2)) + \arctan(\bx^2)}{1+\bx^1+\bx^5}+\sin(0.5\bx^3)(1+\exp(\bx^4-0.5\bx^3))+{\bx^3}^2+2\sin(\bx^4)+4\bx^5$, which is non-continuous in terms of $\bx^1, \bx^2$ but infinitely differentiable in terms of $\bx^3, \bx^4, \bx^5$. 
\end{enumerate}

\paragraph{Synthetic Benchmarks} We create synthetic benchmark datasets of varying number of observations $n$ and number of features $d$. 
The \textbf{synthetic-continuous} dataset uses only continuous features, and the \textbf{synthetic-mixture} dataset uses a mixture of continuous and discrete features. The synthetic features are drawn either from $Bern(0.5)$ (if discrete) or $Unif(-2,2)$ (if continuous).
Additionally, each feature is either causal (i.e., used by $f_0$) or non-causal. For each simulation setting, there are always $d^*=5$ causal features. Specifically, in the \textbf{synthetic-continuous} dataset, all features are continuous, while in the \textbf{synthetic-mixture} dataset, there are 2 discrete and 3 continuous causal features, while there are 2 discrete non-causal features (all the rest of non-causal features are continous). 


For each sample size - data dimension scenario, we use the same set of generated features across the repeated simulation runs. 



\paragraph{Socio-economic and Healthcare Data} 
\begin{itemize}
\item \textbf{adult}: 1994 U.S. census data of 48842 adults with 8 categorical and 6 continuous features \cite{kohavi_scaling_nodate}. The data is publicly available\footnote{\url{https://archive.ics.uci.edu/ml/machine-learning-databases/adult/}} and does not contain personally identifiable information or offensive content. We concatenated the training data (\texttt{adult.data}) and testing data (\texttt{adult.test}), and remove all observations with missing features.  Additionally, we removed the redundant feature "\texttt{education}", and performed suitable re-categorization for discrete features: For "\texttt{race}", we encoded "White" as $0$ and the rest as $1$; for "\texttt{sex}", we encoded "Female" as $1$ and "Male" as $0$; for "\texttt{relationship}", we encoded "Husband" as $0$, "Not-in-family" as $1$ and the rest as $2$; for "\texttt{workclass}", we encoded "Private" as $0$, "Self-emp-not-inc" as $1$ and the rest as $2$; for "marital\_status", we encoded "Married-civ-spouse" as $0$, "Never-married" as $1$ and the rest as $2$; for "\texttt{occupation}", we encoded "Prof-specialty" as $0$, "Craft-repair" as $1$ and the rest as $2$; for "\texttt{native\_country}", we encoded "United-States" as $0$, "Mexico" as $1$ and the rest as $2$. The final features in the dataset are: (\texttt{"race", "sex", "education\_num", "hours\_per\_week", "age", "relationship", "workclass", "fnlwgt", "capital\_gain", "capital\_loss", "marital\_status", "occupation", "native\_country"}). If the data dimension is higher than 13, additional features will be generated from $Unif(-2, 2)$. 

\item \textbf{heart}: a coronary artery disease dataset of 303 patients from Cleveland clinic database with 7 categorical and 6 continuous features \cite{detrano_international_1989}. The data is publicly available\footnote{\url{https://archive.ics.uci.edu/ml/machine-learning-databases/heart-disease/processed.cleveland.data}} and does not contain personally identifiable information or offensive content. All observations with missing features are removed before analysis. 

The list of features used in the final datasets are  (\texttt{"sex", "exang", "thal", "oldpeak", "age", "ca", "cp", "chol", "trestbps", "thalach", "fbs", "restecg", "slope"}). If the data dimension is higher than 13, additional features will be generated from $Unif(-2, 2)$.

\item \textbf{mi}: disease records of myocardial infarction (MI) of 1700 patients from Krasnoyarsk interdistrict clinical hospital during 1992-1995, with 113 categorical and 11 continuous features \cite{golovenkin_trajectories_2020}. The data is publicly available\footnote{ \url{https://archive.ics.uci.edu/ml/machine-learning-databases/00579/}} and does not contain personally identifiable information or offensive content. 
We imputed missing values using the \texttt{IterativeImputer} method from \texttt{scikit-learn} package and with a \texttt{BayesianRidge} regressor. Specifically, it imputes each feature with missing values as a function of other features in a round-robin fashion: At each step, a feature column is designated as output $y$ and the other feature columns are treated as inputs $X$. A regressor is fit on $(X, y)$ for known $y$. Then, the regressor is used to predict the missing values of $y$. This is done for each feature in an iterative fashion, and then is repeated for 10 imputation rounds. The results of the final imputation round are returned.

The listed of features used in the analysis are as below: 
\texttt{("sex", "ritm\_ecg\_p\_01", "age", "s\_ad\_orit", "d\_ad\_orit", "ant\_im", "ibs\_post", "k\_blood", "na\_blood", "l\_blood", "inf\_anam", "stenok\_an", "fk\_stenok", "ibs\_nasl", "gb", "sim\_gipert", "dlit\_ag", "zsn\_a", "nr11", "nr01", "nr02", "nr03", "nr04", "nr07", "nr08", "np01", "np04", "np05", "np07", "np08", "np09", "np10", "endocr\_01", "endocr\_02", "endocr\_03", "zab\_leg\_01", "zab\_leg\_02", "zab\_leg\_03", "zab\_leg\_04", "zab\_leg\_06", "s\_ad\_kbrig", "d\_ad\_kbrig", "o\_l\_post", "k\_sh\_post", "mp\_tp\_post", "svt\_post", "gt\_post", "fib\_g\_post", "lat\_im", "inf\_im", "post\_im", "im\_pg\_p", "ritm\_ecg\_p\_02", "ritm\_ecg\_p\_04", "ritm\_ecg\_p\_06", "ritm\_ecg\_p\_07", "ritm\_ecg\_p\_08", "n\_r\_ecg\_p\_01", "n\_r\_ecg\_p\_02", "n\_r\_ecg\_p\_03", "n\_r\_ecg\_p\_04", "n\_r\_ecg\_p\_05", "n\_r\_ecg\_p\_06",  "n\_r\_ecg\_p\_08", "n\_r\_ecg\_p\_09", "n\_r\_ecg\_p\_10", "n\_p\_ecg\_p\_01", "n\_p\_ecg\_p\_03", "n\_p\_ecg\_p\_04", "n\_p\_ecg\_p\_05", "n\_p\_ecg\_p\_06", "n\_p\_ecg\_p\_07", "n\_p\_ecg\_p\_08", "n\_p\_ecg\_p\_09", "n\_p\_ecg\_p\_10", "n\_p\_ecg\_p\_11", "n\_p\_ecg\_p\_12", "fibr\_ter\_01", "fibr\_ter\_02", "fibr\_ter\_03", "fibr\_ter\_05", "fibr\_ter\_06", "fibr\_ter\_07", "fibr\_ter\_08", "gipo\_k", "giper\_na", "alt\_blood", "ast\_blood", "kfk\_blood", "roe", "time\_b\_s", "r\_ab\_1\_n", "r\_ab\_2\_n", "r\_ab\_3\_n", "na\_kb", "not\_na\_kb", "lid\_kb", "nitr\_s", "na\_r\_1\_n", "na\_r\_2\_n", "na\_r\_3\_n", "not\_na\_1\_n", "not\_na\_2\_n", "not\_na\_3\_n", "lid\_s\_n", "b\_block\_s\_n", "ant\_ca\_s\_n", "gepar\_s\_n", "asp\_s\_n", "tikl\_s\_n", "trent\_s\_n")}.
\end{itemize}
We standardize (by subtracting from mean and dividing by standard deviation) all features except for 2 discrete causal features and 2 discrete non-causal features.

\newpage
\begin{figure*}[ht]
\vskip 0.2in
\begin{center}
\centerline{\includegraphics[width=1.3\textwidth]{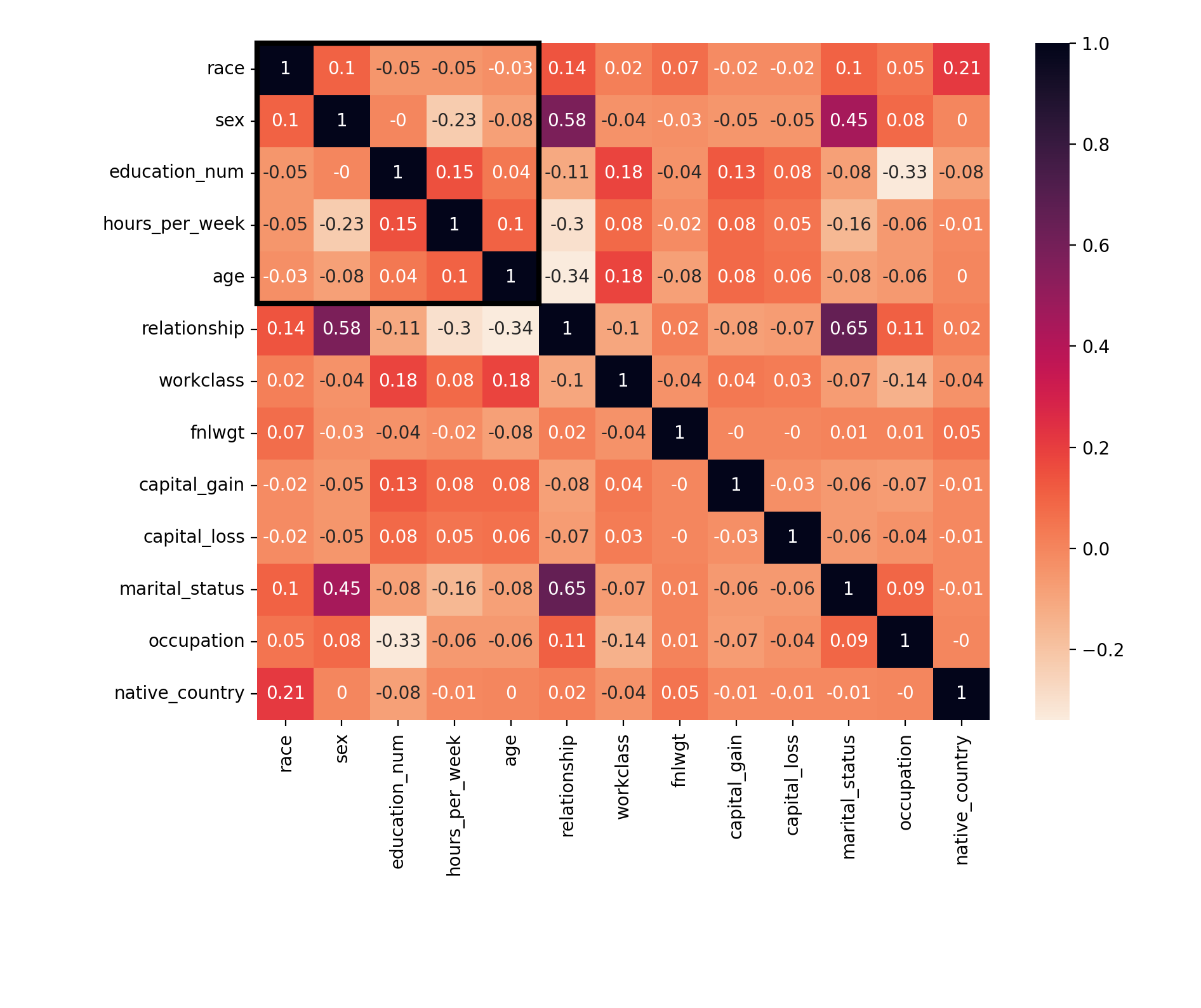}}
\caption{Correlation matrix for \textbf{adult} dataset, where the upper left black box indicates the five causal features.}
\label{fig:corr_adult}
\end{center}
\vskip -0.2in
\end{figure*}

\begin{figure*}[ht]
\vskip 0.2in
\begin{center}
\centerline{\includegraphics[width=1.2\textwidth]{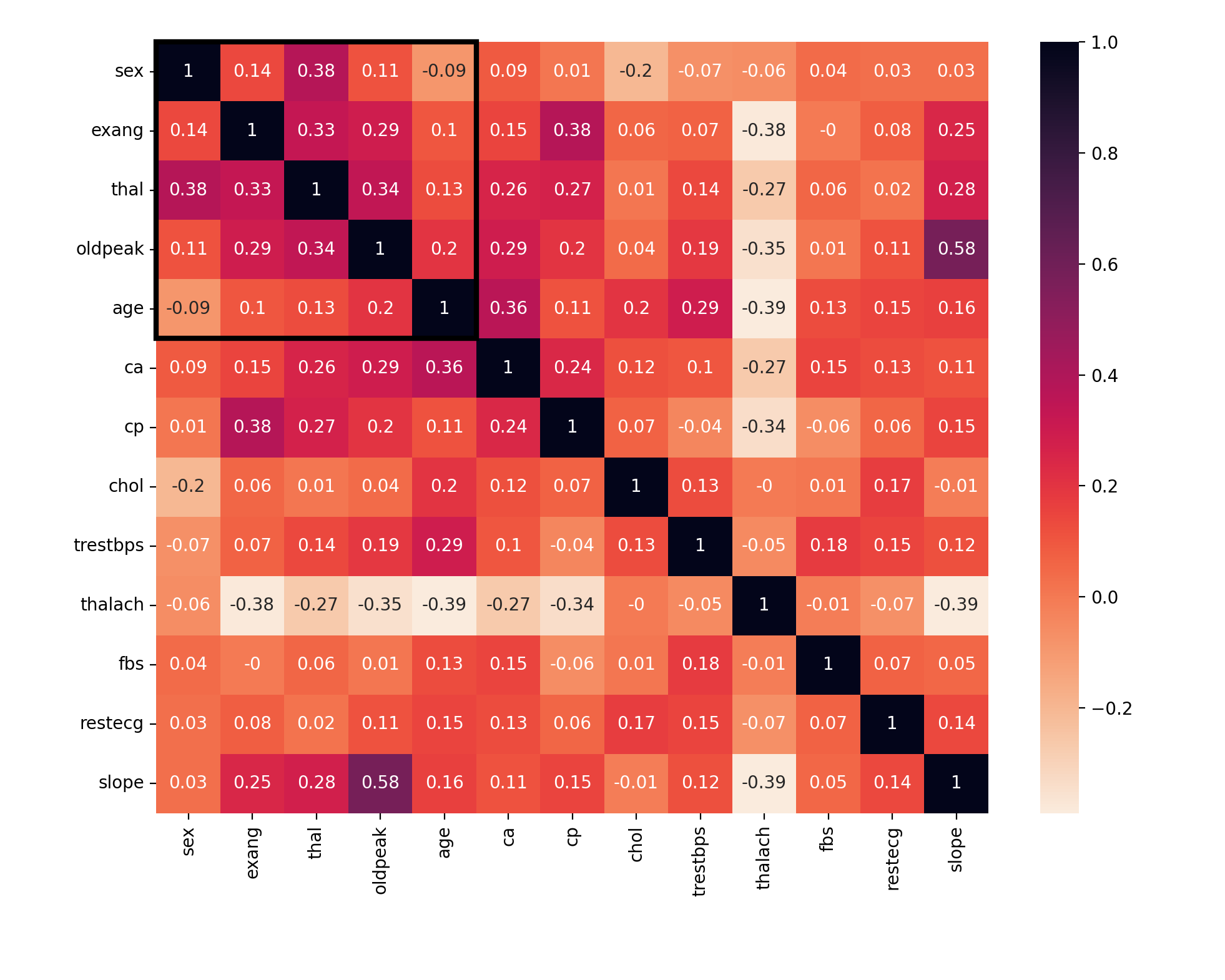}}
\caption{Correlation matrix for \textbf{heart} dataset, where the upper left black box indicates the five causal features.}
\label{fig:corr_heart}
\end{center}
\vskip -0.2in
\end{figure*}

\begin{figure*}[ht]
\vskip 0.2in
\begin{center}
\centerline{\includegraphics[width=1.5\textwidth]{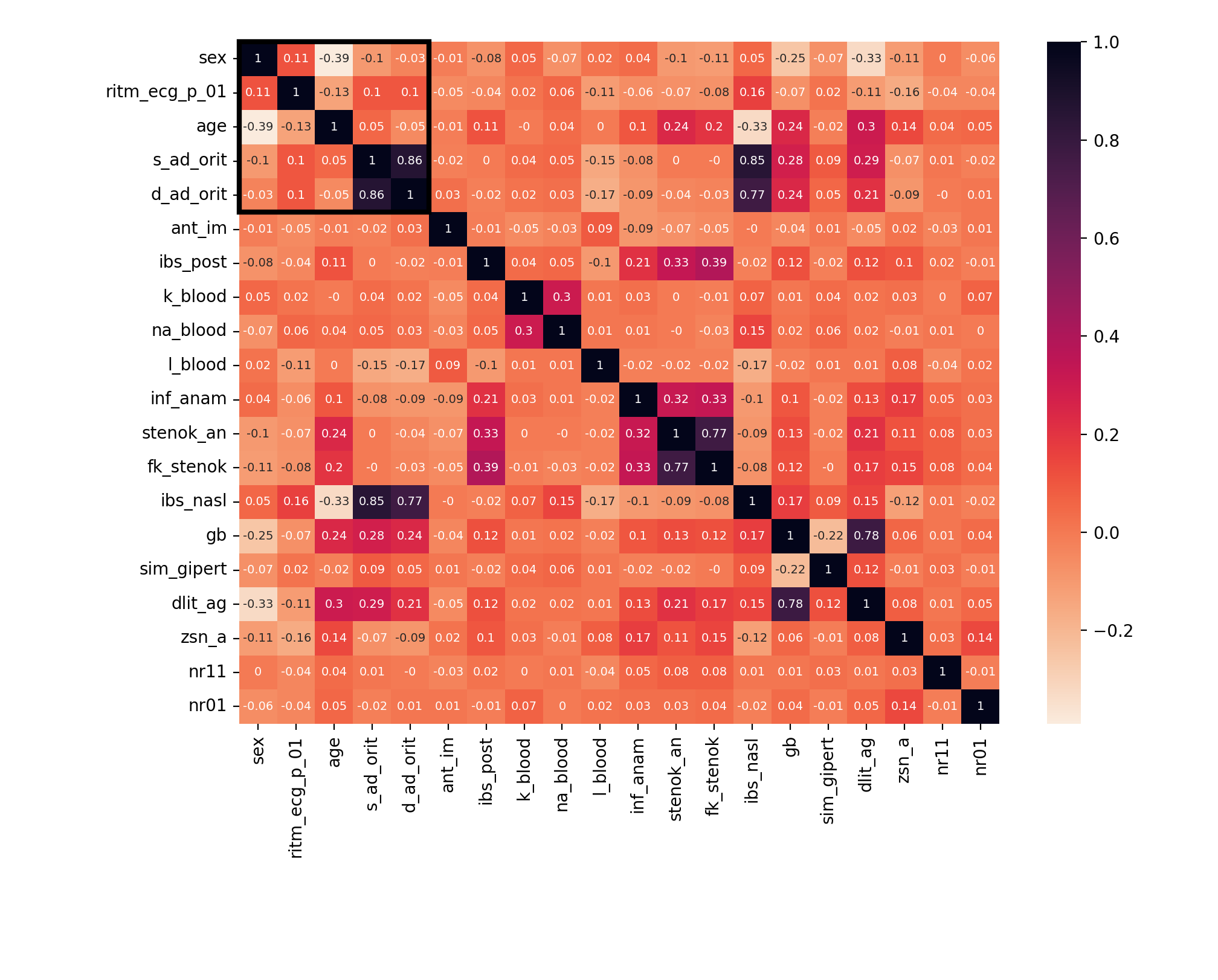}}
\caption{Correlation matrix for the first 20 features in \textbf{mi} dataset, where the upper left black box indicates the five causal features.}
\label{fig:corr_mi}
\end{center}
\vskip -0.2in
\end{figure*}

\clearpage
\subsection{Error Bars}
\label{sec:err_bar}
Tables \ref{error:mix}-\ref{error:mi} shows the AUROC scores or Testing MSE's for the result presented in the main text. For the Testing MSE tables, a method will not be shown if they share the model fit with another method (\gls{impurity} and \gls{knockoff}), or if the method does not produce valid result due to small sample size (\gls{GAM}).
\begin{table}[ht]
\caption{AUROC scores and their standard deviations for \textbf{synthetic-mixture} dataset.}
\vskip 0.1in
\hskip -2.0cm
\scriptsize
\begin{tabular}{l|llllllllll}
\textbf{n}    & \textbf{RF-FDT (Ours)} & \textbf{RFNN (Ours)} & \textbf{RF-Impurity} & \textbf{BKMR}       & \textbf{BART}       & \textbf{BAKR}       & \textbf{RF-KnockOff} & \textbf{GAM (Ours)} & \textbf{BRR}        & \textbf{BL}         \\ \hline
100  & 0.8(0.09)     & 0.59(0.11)  & 0.72(0.13)  & 0.55(0.07) & 0.72(0.14) & 0.59(0.09) & 0.58(0.06)  & NA(NA)     & 0.63(0.13) & 0.67(0.09) \\
200  & 0.93(0.1)     & 0.64(0.18)  & 0.86(0.2)   & 0.59(0.11) & 0.77(0.12) & 0.73(0.12) & 0.59(0.05)  & 0.73(0.17) & 0.71(0.15) & 0.72(0.14) \\
500  & 0.99(0.03)    & 0.7(0.18)   & 1(0)        & 0.57(0.11) & 0.93(0.08) & 0.69(0.14) & 0.67(0.11)  & 0.89(0.08) & 0.82(0.1)  & 0.82(0.1)  \\
1000 & 0.99(0.03)    & 0.69(0.24)  & 1(0)        & 0.59(0.14) & 0.99(0.03) & 0.76(0.18) & 0.7(0.11)   & 0.9(0.1)   & 0.87(0.1)  & 0.87(0.11)
\end{tabular}
\label{error:mix}
\end{table}

\begin{table}[ht]
\caption{AUROC scores and their standard deviations for \textbf{synthetic-continuous} dataset.}
\vskip 0.1in
\hskip -2.0cm
\scriptsize
\begin{tabular}{l|llllllllll}
\textbf{n}    & \textbf{RF-FDT (Ours)} & \textbf{RFNN (Ours)} & \textbf{RF-Impurity} & \textbf{BKMR}       & \textbf{BART}       & \textbf{BAKR}       & \textbf{RF-KnockOff} & \textbf{GAM (Ours)} & \textbf{BRR}        & \textbf{BL}         \\ \hline
100  & 0.61(0.13)    & 0.56(0.12)  & 0.68(0.13)  & 0.59(0.07) & 0.7(0.11)  & 0.59(0.09) & 0.6(0.11)   & NA(NA)     & 0.69(0.13) & 0.67(0.13) \\
200  & 0.84(0.13)    & 0.65(0.11)  & 0.85(0.13)  & 0.56(0.1)  & 0.8(0.14)  & 0.69(0.12) & 0.68(0.14)  & 0.74(0.12) & 0.76(0.13) & 0.76(0.14) \\
500  & 0.99(0.02)    & 0.71(0.17)  & 0.99(0.02)  & 0.68(0.15) & 0.97(0.03) & 0.76(0.13) & 0.88(0.11)  & 0.81(0.14) & 0.82(0.13) & 0.82(0.13) \\
1000 & 1(0)          & 0.63(0.24)  & 1(0)        & 0.61(0.12) & 1(0.01)    & 0.78(0.12) & 0.97(0.08)  & 0.86(0.11) & 0.86(0.14) & 0.86(0.14)
\end{tabular}
\end{table}

\begin{table}[ht]
\caption{AUROC scores and their standard deviations for \textbf{adult} dataset.}
\vskip 0.1in
\hskip -2.0cm
\scriptsize
\begin{tabular}{l|llllllllll}
\textbf{n}    & \textbf{RF-FDT (Ours)} & \textbf{RFNN (Ours)} & \textbf{RF-Impurity} & \textbf{BKMR}       & \textbf{BART}       & \textbf{BAKR}       & \textbf{RF-KnockOff} & \textbf{GAM (Ours)} & \textbf{BRR}        & \textbf{BL}         \\ \hline
100  & 0.76(0.1)     & 0.62(0.13)  & 0.61(0.16)  & 0.57(0.1)  & 0.66(0.12) & 0.66(0.09) & 0.59(0.13)  & NA(NA)     & 0.57(0.1)  & 0.58(0.11) \\
200  & 0.8(0.09)     & 0.6(0.14)   & 0.64(0.11)  & 0.59(0.06) & 0.69(0.14) & 0.7(0.1)   & 0.58(0.1)   & 0.72(0.14) & 0.6(0.09)  & 0.61(0.07) \\
500  & 0.84(0.07)    & 0.57(0.18)  & 0.64(0.09)  & 0.64(0.11) & 0.68(0.09) & 0.63(0.09) & 0.58(0.13)  & 0.78(0.12) & 0.61(0.1)  & 0.59(0.12) \\
1000 & 0.81(0.1)     & 0.64(0.18)  & 0.61(0.08)  & 0.57(0.11) & 0.67(0.14) & 0.66(0.11) & 0.57(0.09)  & 0.86(0.11) & 0.69(0.11) & 0.69(0.09)
\end{tabular}
\end{table}

\begin{table}[ht]
\caption{AUROC scores and their standard deviations for \textbf{heart} dataset.}
\vskip 0.1in
\hskip -2.0cm
\scriptsize
\begin{tabular}{l|llllllllll}
\textbf{n}    & \textbf{RF-FDT (Ours)} & \textbf{RFNN (Ours)} & \textbf{RF-Impurity} & \textbf{BKMR}       & \textbf{BART}       & \textbf{BAKR}       & \textbf{RF-KnockOff} & \textbf{GAM (Ours)} & \textbf{BRR}        & \textbf{BL}         \\ \hline
50  & 0.71(0.06)    & 0.56(0.14)  & 0.49(0.15)  & 0.61(0.09) & 0.59(0.1)  & 0.58(0.08) & 0.6(0.07)   & NA(NA)     & 0.6(0.06)  & 0.57(0.07) \\
100 & 0.72(0.06)    & 0.58(0.12)  & 0.44(0.11)  & 0.6(0.07)  & 0.59(0.08) & 0.64(0.13) & 0.57(0.08)  & NA(NA)     & 0.61(0.07) & 0.55(0.08) \\
150 & 0.75(0.08)    & 0.59(0.13)  & 0.41(0.12)  & 0.58(0.08) & 0.61(0.1)  & 0.69(0.13) & 0.61(0.11)  & 0.62(0.12) & 0.6(0.05)  & 0.56(0.05) \\
257 & 0.74(0.08)    & 0.52(0.17)  & 0.44(0.12)  & 0.6(0.08)  & 0.6(0.09)  & 0.65(0.07) & 0.66(0.12)  & 0.63(0.13) & 0.59(0.07) & 0.56(0.07) \\
\end{tabular}
\end{table}

\begin{table}[ht]
\caption{AUROC scores and their standard deviations for \textbf{mi} dataset.}
\vskip 0.1in
\hskip -2.0cm
\scriptsize
\begin{tabular}{l|llllllllll}
\textbf{n}    & \textbf{RF-FDT (Ours)} & \textbf{RFNN (Ours)} & \textbf{RF-Impurity} & \textbf{BKMR}       & \textbf{BART}       & \textbf{BAKR}       & \textbf{RF-KnockOff} & \textbf{GAM (Ours)} & \textbf{BRR}        & \textbf{BL}         \\ \hline
100  & 0.86(0.05)    & 0.59(0.13)  & 0.77(0.08)  & 0.61(0.1)  & 0.65(0.12) & 0.87(0.05) & 0.67(0.12)  & NA(NA)     & 0.63(0.11) & 0.62(0.12) \\
200  & 0.85(0.04)    & 0.62(0.08)  & 0.79(0.07)  & 0.56(0.07) & 0.63(0.11) & 0.82(0.08) & 0.62(0.11)  & 0.88(0.05) & 0.63(0.08) & 0.62(0.11) \\
500  & 0.85(0.05)    & 0.43(0.15)  & 0.77(0.06)  & 0.58(0.09) & 0.59(0.08) & 0.86(0.06) & 0.63(0.08)  & 0.87(0.07) & 0.61(0.1)  & 0.59(0.11) \\
1000 & 0.83(0.04)    & 0.56(0.17)  & 0.73(0.07)  & 0.56(0.08) & 0.67(0.09) & 0.86(0.06) & 0.64(0.12)  & 0.9(0.06)  & 0.64(0.07) & 0.65(0.11)
\end{tabular}
\end{table}

\begin{table}[ht]
\caption{Testing MSE's and their standard deviations for \textbf{synthetic-mixture} dataset. A method will not be shown if they share the model fit with another method (\gls{impurity} and \gls{knockoff}), or if the method does not produce valid result due to small sample size (\gls{GAM})}
\vskip 0.1in
\hskip -0.4cm
\scriptsize
\begin{tabular}{l|llllllllll}
\textbf{n}    & \textbf{RF-FDT (Ours)} & \textbf{RFNN (Ours)}  & \textbf{BKMR}       & \textbf{BART}       & \textbf{BAKR}       & \textbf{GAM (Ours)} & \textbf{BRR}        & \textbf{BL}         \\ \hline
100  & 1.01(0.25)    & 1.76(0.28)  & 1.52(0.16) & 1.03(0.28) & 2.7(0.79)  & NA(NA)     & 0.98(0.24) & 0.95(0.24) \\
200  & 0.87(0.16)    & 1.51(0.23)  & 1.59(0.09) & 0.91(0.14) & 1.01(0.11) & 1.56(0.32) & 1.01(0.13) & 0.95(0.12) \\
500  & 0.76(0.13)    & 1.42(0.15)  & 1.57(0.07) & 0.8(0.14)  & 0.96(0.1)  & 1.05(0.2)  & 0.94(0.14) & 0.92(0.12) \\
1000 & 0.66(0.12)    & 1.34(0.17)  & 1.64(0.07) & 0.73(0.17) & 1.02(0.12) & 0.96(0.18) & 0.95(0.16) & 0.93(0.17)
\end{tabular}
\end{table}

\begin{table}[ht]
\caption{Testing MSE's and their standard deviations for \textbf{synthetic-continuous} dataset. A method will not be shown if they share the model fit with another method (\gls{impurity} and \gls{knockoff}), or if the method does not produce valid result due to small sample size (\gls{GAM})}
\vskip 0.1in
\hskip -0.4cm
\scriptsize
\begin{tabular}{l|llllllllll}
\textbf{n}    & \textbf{RF-FDT (Ours)} & \textbf{RFNN (Ours)}  & \textbf{BKMR}       & \textbf{BART}       & \textbf{BAKR}       & \textbf{GAM (Ours)} & \textbf{BRR}        & \textbf{BL}         \\ \hline
100  & 1.01(0.2)     & 1.73(0.26)  & 1.56(0.12) & 1(0.16)   & 2.62(0.6)  & NA(NA)     & 1.05(0.14) & 1.01(0.17) \\
200  & 0.87(0.15)    & 1.48(0.23)  & 1.5(0.12)  & 0.91(0.2) & 0.95(0.18) & 1.4(0.33)  & 0.91(0.17) & 0.89(0.15) \\
500  & 0.85(0.19)    & 1.43(0.12)  & 1.58(0.09) & 0.91(0.2) & 0.98(0.19) & 0.98(0.23) & 0.93(0.19) & 0.91(0.2)  \\
1000 & 0.72(0.15)    & 1.4(0.19)   & 1.59(0.1)  & 0.8(0.18) & 0.95(0.19) & 0.94(0.2)  & 0.91(0.18) & 0.9(0.18) 
\end{tabular}
\end{table}

\begin{table}[ht]
\caption{Testing MSE's and their standard deviations for \textbf{adult} dataset. A method will not be shown if they share the model fit with another method (\gls{impurity} and \gls{knockoff}), or if the method does not produce valid result due to small sample size (\gls{GAM})}
\vskip 0.1in
\hskip -0.4cm
\scriptsize
\begin{tabular}{l|llllllllll}
\textbf{n}    & \textbf{RF-FDT (Ours)} & \textbf{RFNN (Ours)}  & \textbf{BKMR}       & \textbf{BART}       & \textbf{BAKR}       & \textbf{GAM (Ours)} & \textbf{BRR}        & \textbf{BL}         \\ \hline
100  & 0.95(0.4)     & 1.69(0.44)  & 0.99(0.12) & 0.48(0.17) & 2.66(0.87) & NA(NA)     & 0.31(0.11) & 0.22(0.06) \\
200  & 0.91(0.24)    & 1.63(0.31)  & 1.02(0.09) & 0.41(0.14) & 0.28(0.08) & 1.26(0.33) & 0.28(0.08) & 0.22(0.06) \\
500  & 0.96(0.18)    & 1.31(0.14)  & 1.02(0.09) & 0.31(0.09) & 0.23(0.07) & 0.48(0.1)  & 0.27(0.07) & 0.23(0.07) \\
1000 & 0.96(0.18)    & 1.28(0.17)  & 1.05(0.08) & 0.22(0.06) & 0.21(0.04) & 0.31(0.06) & 0.23(0.05) & 0.21(0.04)
\end{tabular}
\end{table}

\begin{table}[ht]
\caption{Testing MSE's and their standard deviations for \textbf{heart} dataset. A method will not be shown if they share the model fit with another method (\gls{impurity} and \gls{knockoff}), or if the method does not produce valid result due to small sample size (\gls{GAM})}
\vskip 0.1in
\hskip -0.4cm
\scriptsize
\begin{tabular}{l|llllllllll}
\textbf{n}    & \textbf{RF-FDT (Ours)} & \textbf{RFNN (Ours)}  & \textbf{BKMR}       & \textbf{BART}       & \textbf{BAKR}       & \textbf{GAM (Ours)} & \textbf{BRR}        & \textbf{BL}         \\ \hline
50  & 0.92(0.21)    & 1.93(0.5)   & 0.98(0.16) & 0.38(0.11) & 0.4(0.1)   & NA(NA)     & 0.34(0.12) & 0.25(0.08) \\
100 & 0.95(0.27)    & 1.88(0.29)  & 1.03(0.1)  & 0.4(0.1)   & 0.49(1.58) & NA(NA)     & 0.31(0.08) & 0.23(0.06) \\
150 & 0.98(0.22)    & 1.65(0.32)  & 0.95(0.13) & 0.39(0.14) & 0.3(0.09)  & 1.76(0.37) & 0.27(0.08) & 0.21(0.08) \\
257 & 0.92(0.25)    & 1.51(0.2)   & 1.02(0.08) & 0.33(0.09) & 0.26(0.06) & 0.79(0.2)  & 0.27(0.06) & 0.23(0.06) \\
\end{tabular}
\end{table}

\begin{table}[ht]
\caption{Testing MSE's and their standard deviations for \textbf{mi} dataset. A method will not be shown if they share the model fit with another method (\gls{impurity} and \gls{knockoff}), or if the method does not produce valid result due to small sample size (\gls{GAM})}
\vskip 0.1in
\hskip -0.4cm
\scriptsize
\begin{tabular}{l|llllllllll}
\textbf{n}    & \textbf{RF-FDT (Ours)} & \textbf{RFNN (Ours)}  & \textbf{BKMR}       & \textbf{BART}       & \textbf{BAKR}       & \textbf{GAM (Ours)} & \textbf{BRR}        & \textbf{BL}         \\ \hline
100  & 1.55(0.94)    & 2.02(0.44)  & 0.88(0.24) & 0.31(0.08) & 0.79(0.43) & NA(NA)     & 0.39(0.1)  & 0.33(0.07) \\
200  & 1.76(2.86)    & 1.82(0.48)  & 0.81(0.21) & 0.27(0.11) & 0.35(0.15) & 0.61(0.27) & 0.31(0.12) & 0.3(0.11)  \\
500  & 1.16(0.37)    & 1.57(0.27)  & 0.69(0.3)  & 0.23(0.06) & 0.24(0.07) & 0.43(0.21) & 0.29(0.09) & 0.25(0.07) \\
1000 & 1.2(1.05)     & 1.43(0.3)   & 0.56(0.3)  & 0.21(0.05) & 0.23(0.05) & 0.32(0.09) & 0.25(0.06) & 0.23(0.06)
\end{tabular}
\label{error:mi}
\end{table}

\clearpage
\section{Experiment Results and Additional Figures}
\label{sec:exp_result_app}

Figures \ref{fig:cat_all}-\ref{fig:mi} and \ref{fig:tst_cat}-\ref{fig:tst_mi} show the \gls{AUROC} scores and MSE results, respectively, across all of the datasets. Here we also summarize additional observations that are not included in the main text. The figure captions contain further descriptions of the results. 





\textbf{Synthetic Benchmarks}. In the synthetic datasets, where all features are independent, FDT, RF, BART, BNN, GAM, BRR and BL perform better and more stable than they do in the real datasets where there's feature correlation. The better performance of FDT compared to \gls{impurity} and \gls{knockoff} illustrates the advantage of the proposed integrated partial derivative metric for variable selection. 
For the \textbf{synthetic-continuous} and \textbf{synthetic-mixture} cases, FDT has higher \gls{AUROC} scores across most scenarios, especially when data are generated having high complexity with quickly-varying local fluctuations (rbf, matern32). Moreover, all 11 methods perform only moderately well in complex data settings. The two tree-based methods, RF and BART also have high \gls{AUROC} scores across scenarios, since the tree-based methods naturally rank by how well the features improve the purity of the node. Note that under low dimension case ($d=25$), BKMR is comparable to FDT when $f_0 \in \Hsc$ (linear, rbf, matern32). However, when it comes to medium- or relatively high-dimension settings ($d=50, 100$), BKMR produces low \gls{AUROC} scores due to suffering from the issue of curse of dimensionality \cite{vaart_information_2011}. RFNN, also a kernel-based method, has similar trend as BKMR. Finally, BAKR performs consistently poorly and has lowest \gls{AUROC} scores in relatively low-dimension setting ($d=25, 50$). Linear models (GAM, BRR and BL) achieve comparable or superior performance under the linear data setting. However, for more complicated data generation functions, BRR and BL consistently perform poorly with low \gls{AUROC} scores. 

\textbf{Socio-economic and Healthcare Datasets} In the \textbf{adult}, \textbf{heart} and \textbf{mi} cases, where the features are correlated, the performances of all 11 methods are worse than in the \textbf{synthetic-mixture} and \textbf{synthetic-continuous} cases (where the features are independent). Their performance tends to saturate earlier and are less stable with respect to the sample size. In relatively low-dimension settings ($d=25, 50$), the standard methods such as BART and BNN have higher \gls{AUROC} scores than FDT. However, when the dimension is higher ($d=100$), FDT consistently performs better. 

\newpage
\begin{figure*}[ht]
\vskip 0.2in
\begin{center}
\centerline{\includegraphics[width=1.0\textwidth]{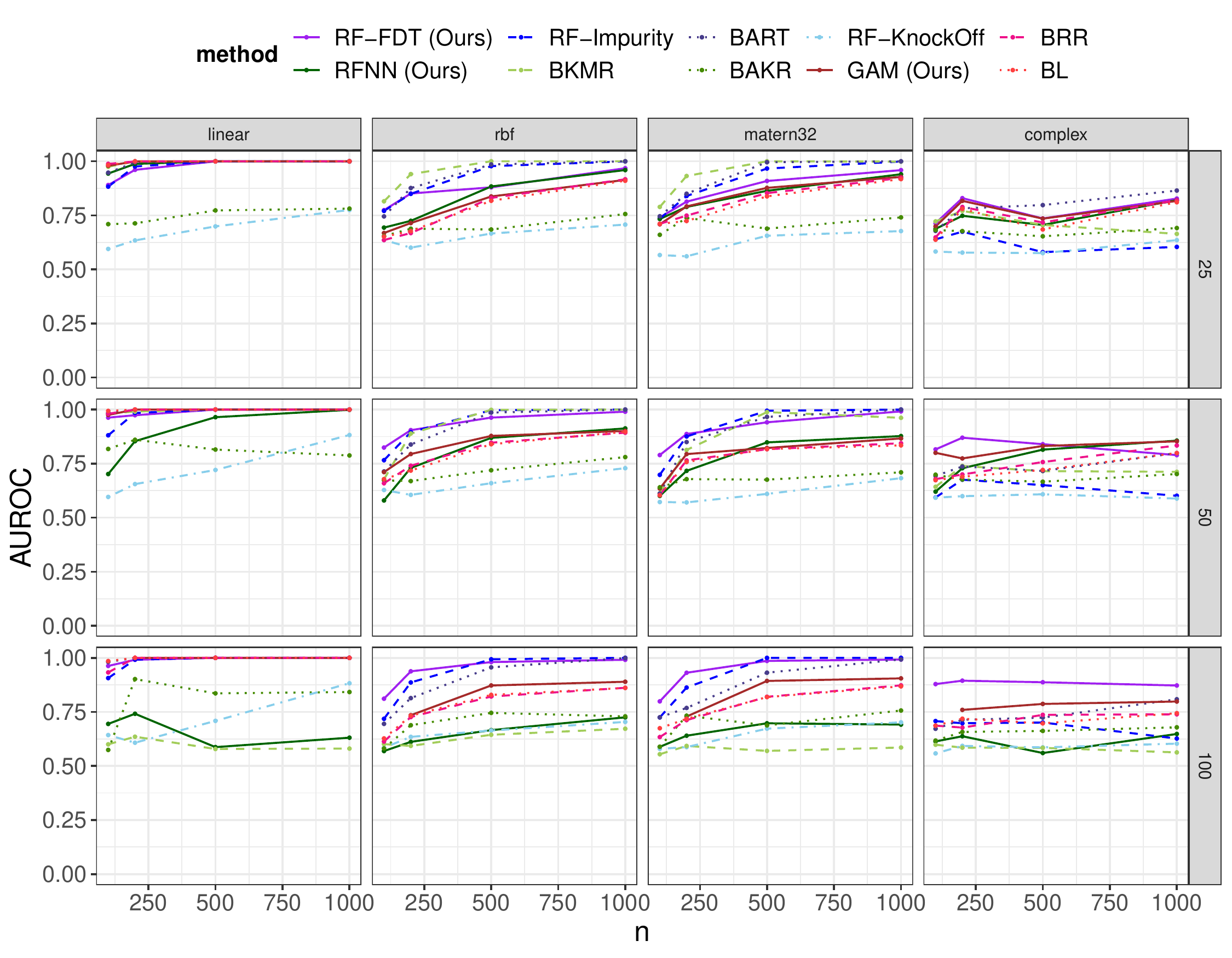}}
\caption{AUROC scores for \textbf{synthetic-mixture} data. FDT generally outperforms other methods in most of the data settings in relatively higher dimension ($d=50, 100$). Knockoff with random forest statistics produce lower \gls{AUROC} scores than in \textbf{synthetic-continuous}, even in linear data settings. Additive models BRR, BL and GAM have mediocre scores under the nonlinear settings. Some model (e.g., GAM) reports missing result in $n>p$ setting due to the restriction of their implementations.}
\label{fig:cat_all}
\end{center}
\vskip -0.2in
\end{figure*}

\begin{figure*}[ht]
\vskip 0.2in
\begin{center}
\centerline{\includegraphics[width=1.0\textwidth]{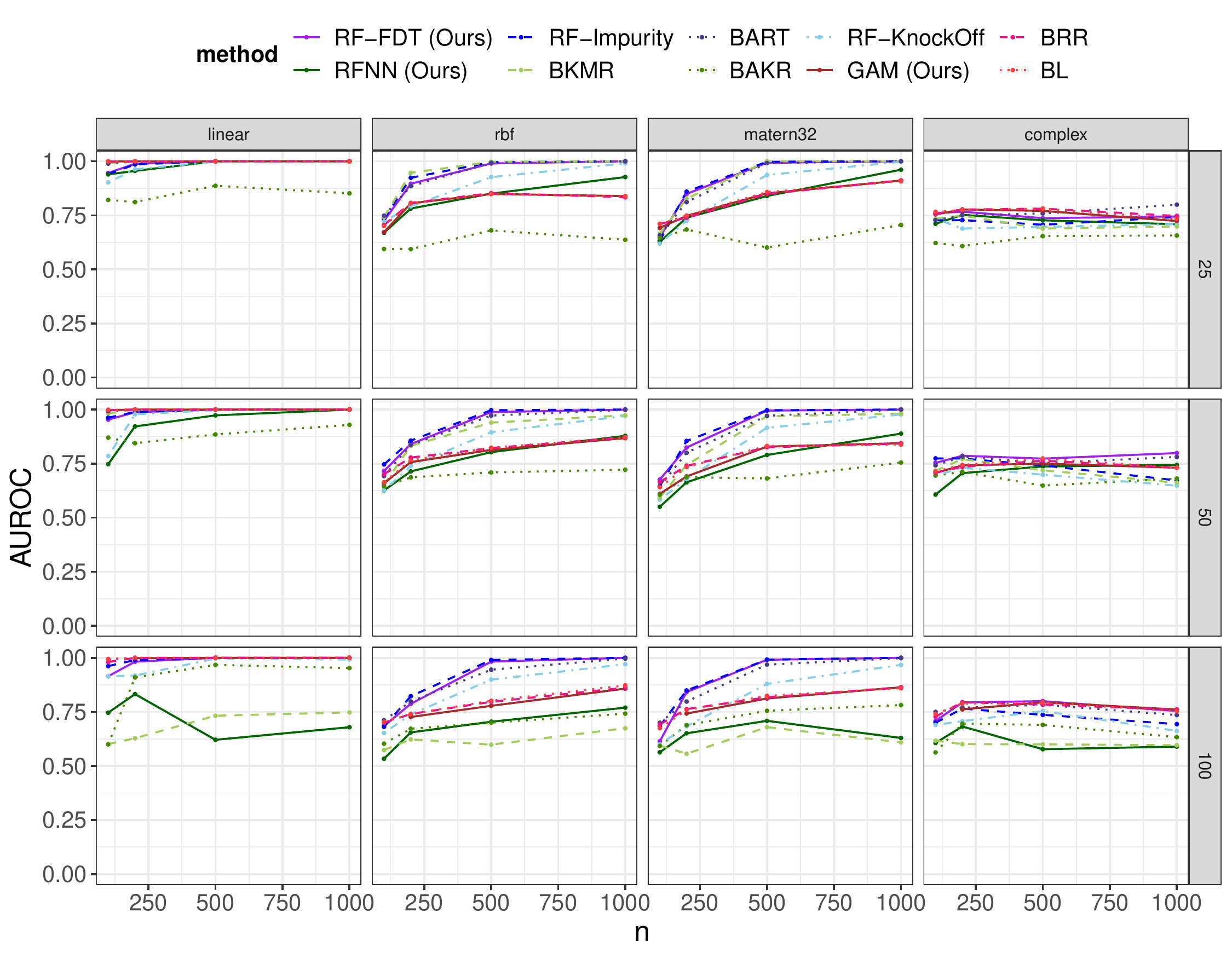}}
\caption{AUROC scores for \textbf{synthetic-continuous} data. FDT generally outperforms other methods in most of the data settings, with BKMR as the comparable one when $d=25$. However, BKMR performs poorly in higher dimension. Tree-based methods RF, BART and Knockoff with random forest statistics have high \gls{AUROC} scores. The performance of BNN is between tree-based methods and additive models while BAKR performs poorly consistently. Additive models BRR, BL and GAM have mediocre scores under the nonlinear settings. Some model (e.g., GAM) reports missing result in $n>p$ setting due to the restriction in their implementations.}
\label{fig:cont_all}
\end{center}
\vskip -0.2in
\end{figure*}

\begin{figure*}[ht]
\vskip 0.2in
\begin{center}
\centerline{\includegraphics[width=1.0\textwidth]{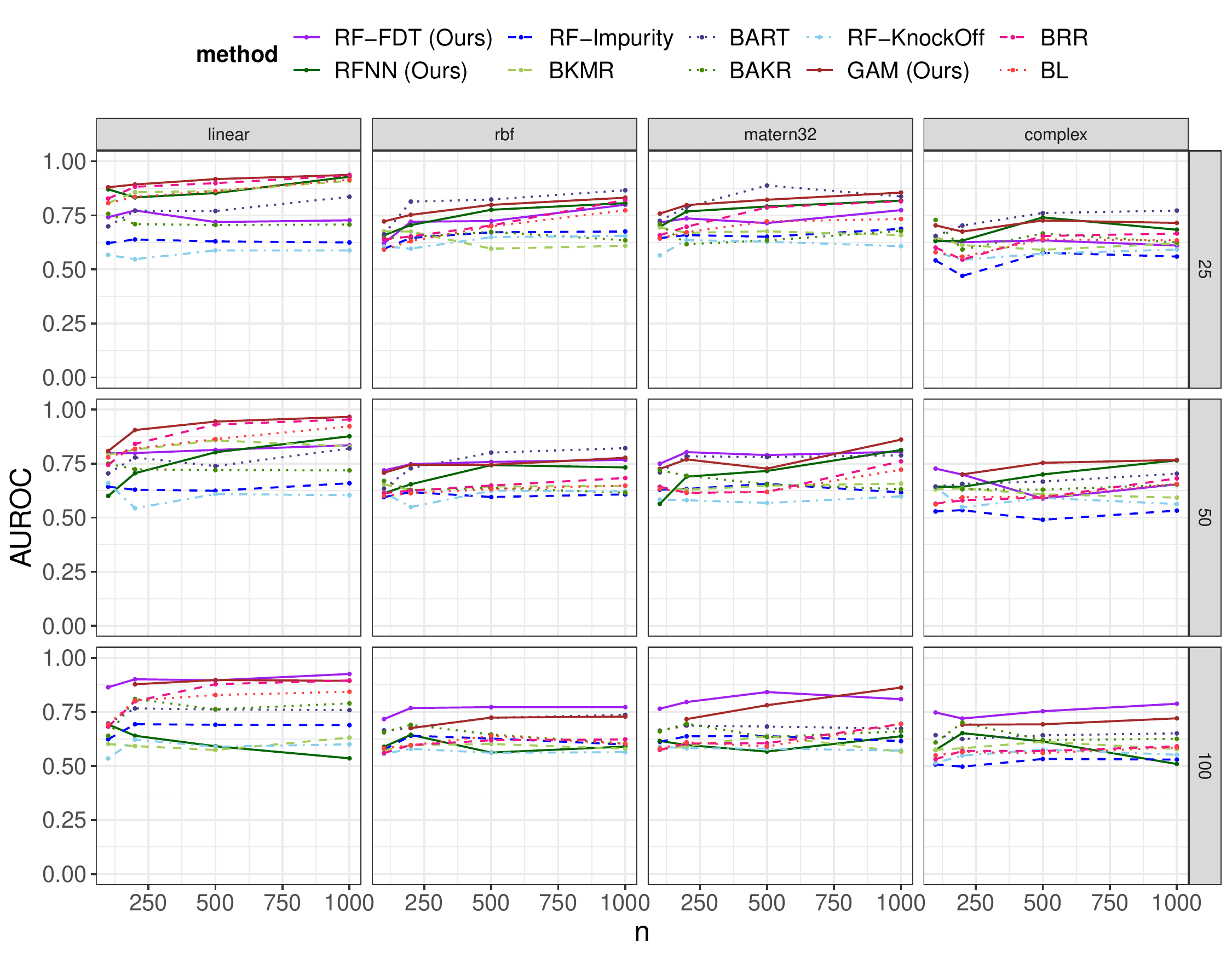}}
\caption{AUROC scores for \textbf{adult} data. In relatively low-dimension settings ($d=25, 50$), the standard methods such as BART and BNN have higher \gls{AUROC} scores than FDT. However, when the dimension is higher ($d=100$), FDT performs better consistently. Some model (e.g., GAM) reports missing result in $n>p$ setting due to the restriction in their implementations.}
\label{fig:adult}
\end{center}
\vskip -0.2in
\end{figure*}

\begin{figure*}[ht]
\vskip 0.2in
\begin{center}
\centerline{\includegraphics[width=1.0\textwidth]{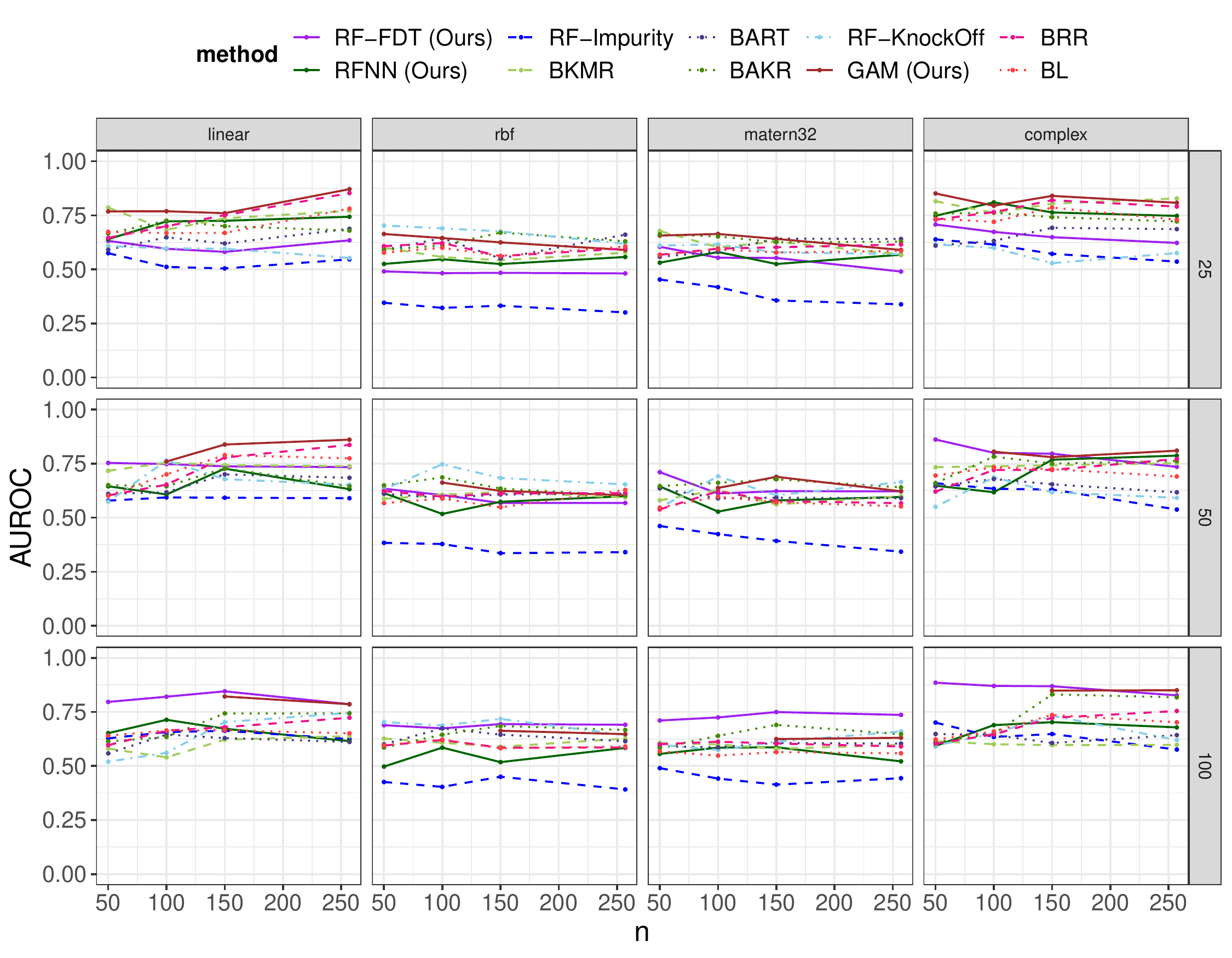}}
\caption{AUROC scores for \textbf{heart} data. In relatively low-dimension settings ($d=25, 50$), the standard methods such as BART and BNN have higher \gls{AUROC} scores than FDT. However, when the dimension is higher ($d=100$), FDT performs better consistently. Some model (e.g., GAM) reports missing result in $n>p$ setting due to the restriction in their implementations.}
\label{fig:heart}
\end{center}
\vskip -0.2in
\end{figure*}

\begin{figure*}[ht]
\vskip 0.2in
\begin{center}
\centerline{\includegraphics[width=1.0\textwidth]{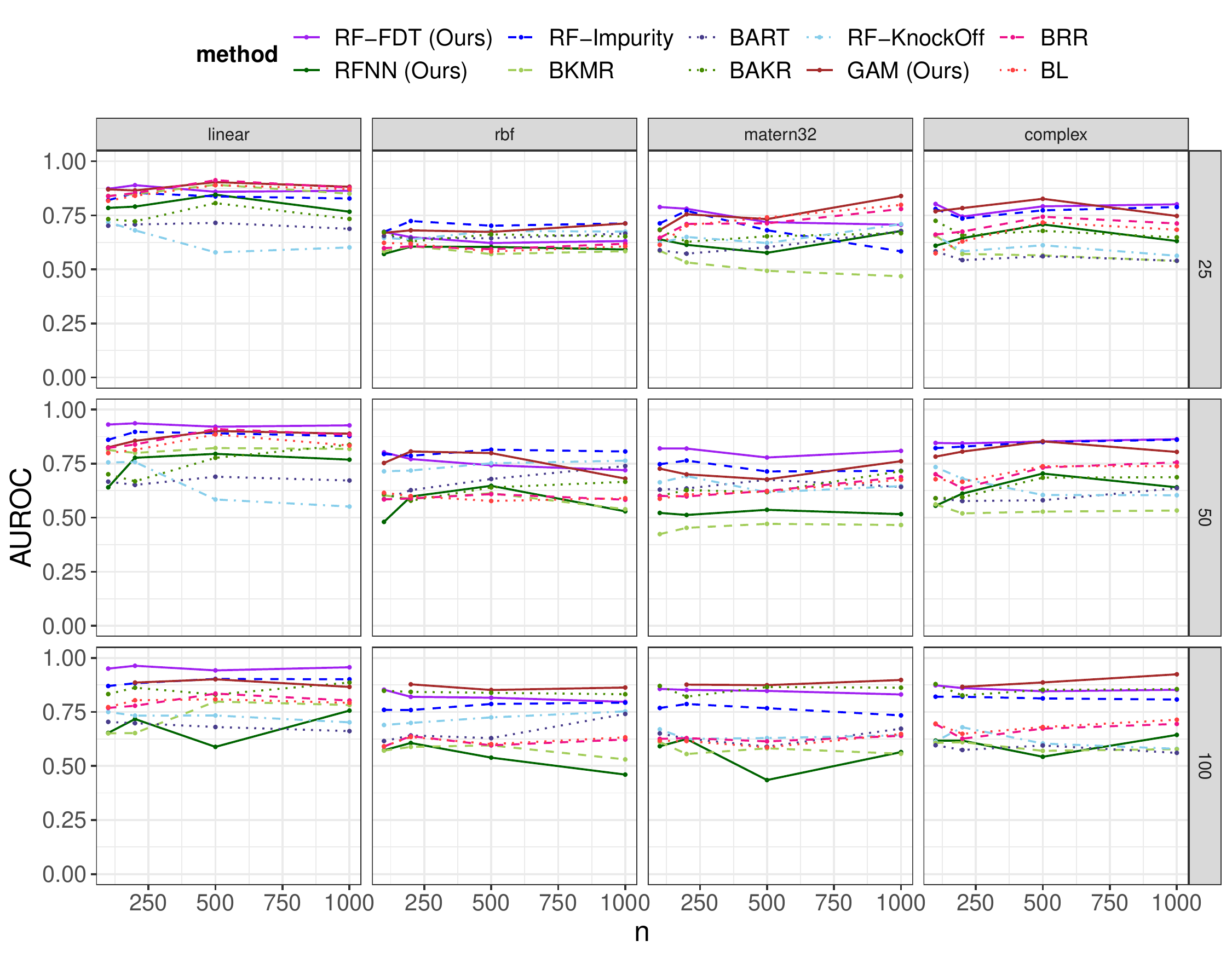}}
\caption{AUROC scores for \textbf{mi} data. In low-dimension setting ($d=25$), the standard methods such as BART and BNN have higher \gls{AUROC} scores than FDT. However, when the dimension is higher ($d=50, 100$), FDT and GAM perform better consistently. Some model (e.g., GAM) reports missing result in $n>p$ setting due to the restriction in their implementations.}
\label{fig:mi}
\end{center}
\vskip -0.2in
\end{figure*}

\begin{figure*}[ht]
\vskip 0.2in
\begin{center}
\centerline{\includegraphics[width=1.0\textwidth]{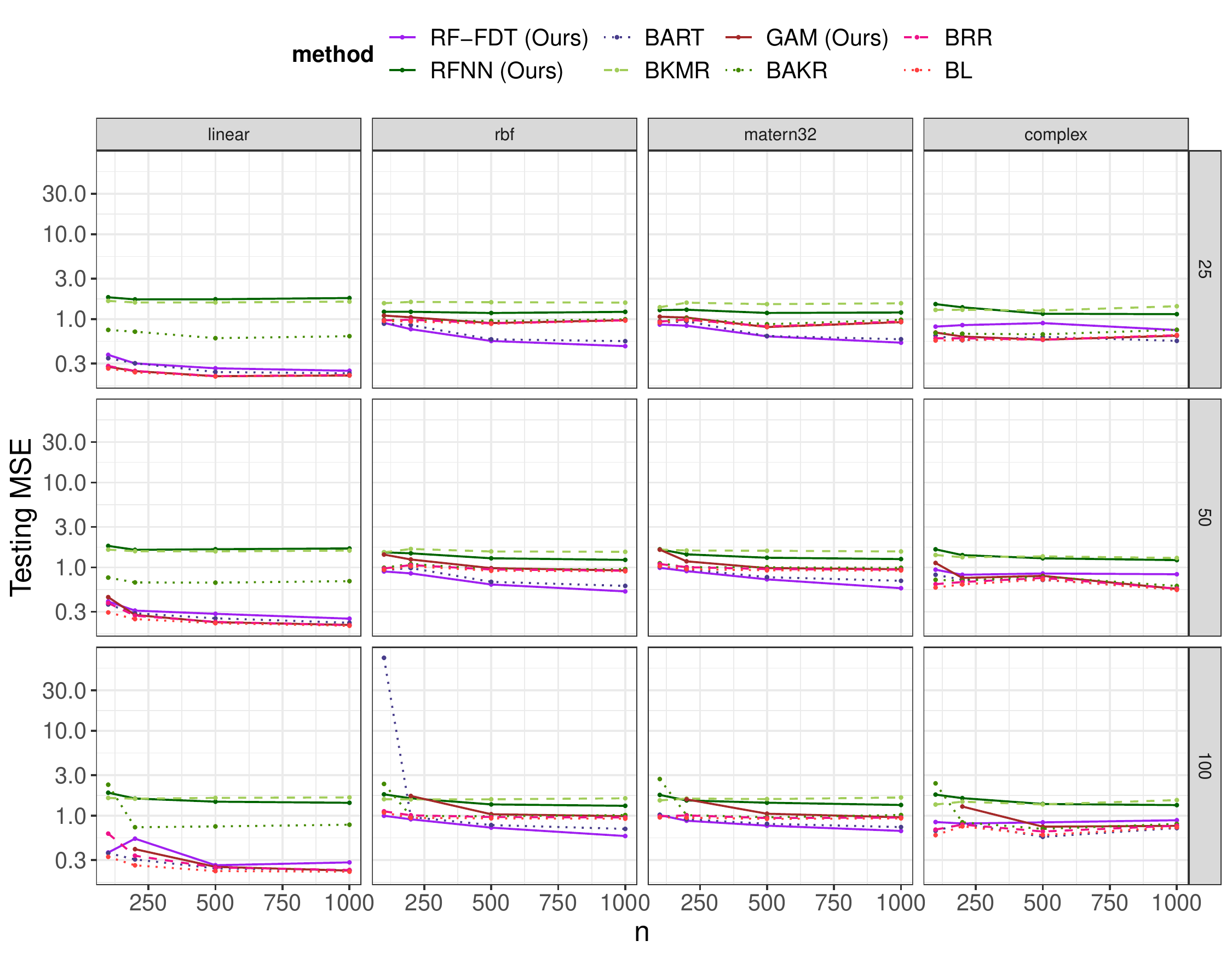}}
\caption{Testing MSE for \textbf{synthetic-mixture} data. FDT generally performs better or competitively with baselines, except in the \texttt{linear} case where \textbf{BL} unsurprisingly does best. \textbf{BKMR} consistently performs worse than other methods, except in the low data, high dimension setting when \textbf{BAKR} performs worst. Some model (e.g., GAM) reports missing result in $n>p$ setting due to the restriction in their implementations. Notice that this dataset contains a setting $n=p$, which can lead to the double descent phenonmenon for some random-feature-based models \citep{d2020triple}.}
\label{fig:tst_cat}
\end{center}
\vskip -0.2in
\end{figure*}

\begin{figure*}[ht]
\vskip 0.2in
\begin{center}
\centerline{\includegraphics[width=1.0\textwidth]{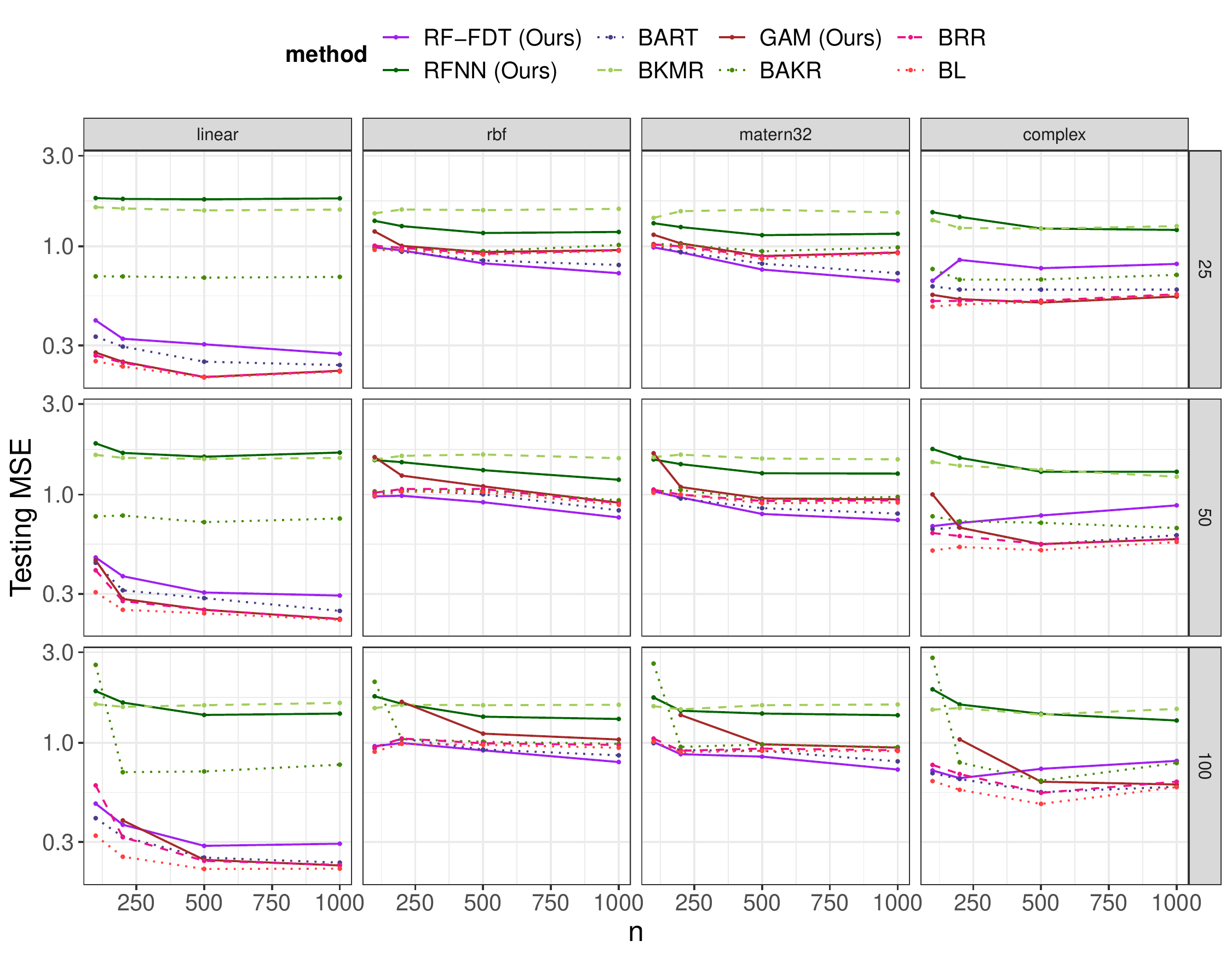}}
\caption{Testing MSE for \textbf{synthetic-continuous} data. A method will not be shown if they share the model fit with another method (\gls{impurity} and \gls{knockoff}), or if the method does not produce valid result due to small sample size (\gls{GAM}).}
\label{fig:tst_cont}
\end{center}
\vskip -0.2in
\end{figure*}

\begin{figure*}[ht]
\vskip 0.2in
\begin{center}
\centerline{\includegraphics[width=1.0\textwidth]{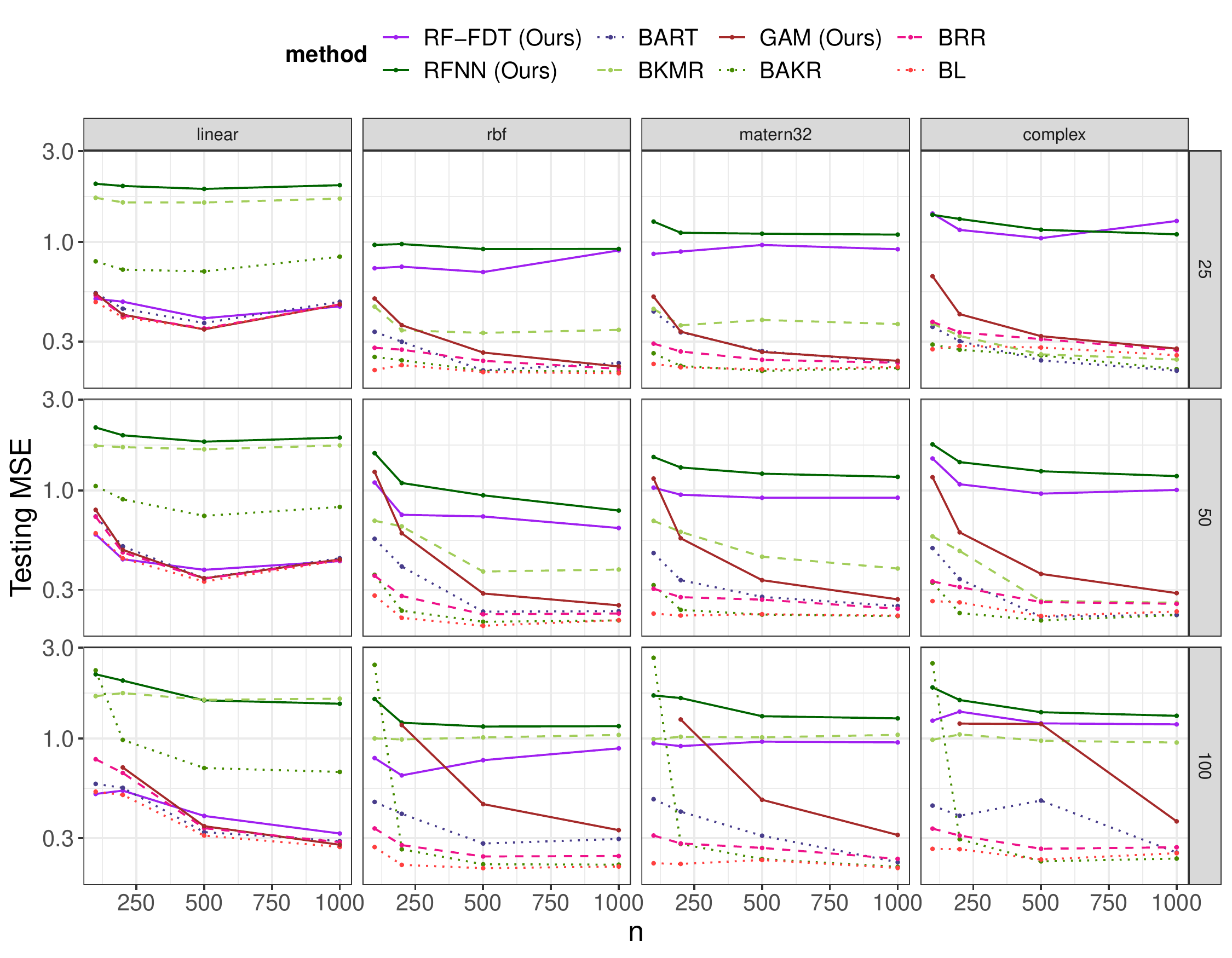}}
\caption{Testing MSE for \textbf{adult} data. A method will not be shown if they share the model fit with another method (\gls{impurity} and \gls{knockoff}), or if the method does not produce valid result due to small sample size (\gls{GAM}).}
\label{fig:tst_adult}
\end{center}
\vskip -0.2in
\end{figure*}

\begin{figure*}[ht]
\vskip 0.2in
\begin{center}
\centerline{\includegraphics[width=1.0\textwidth]{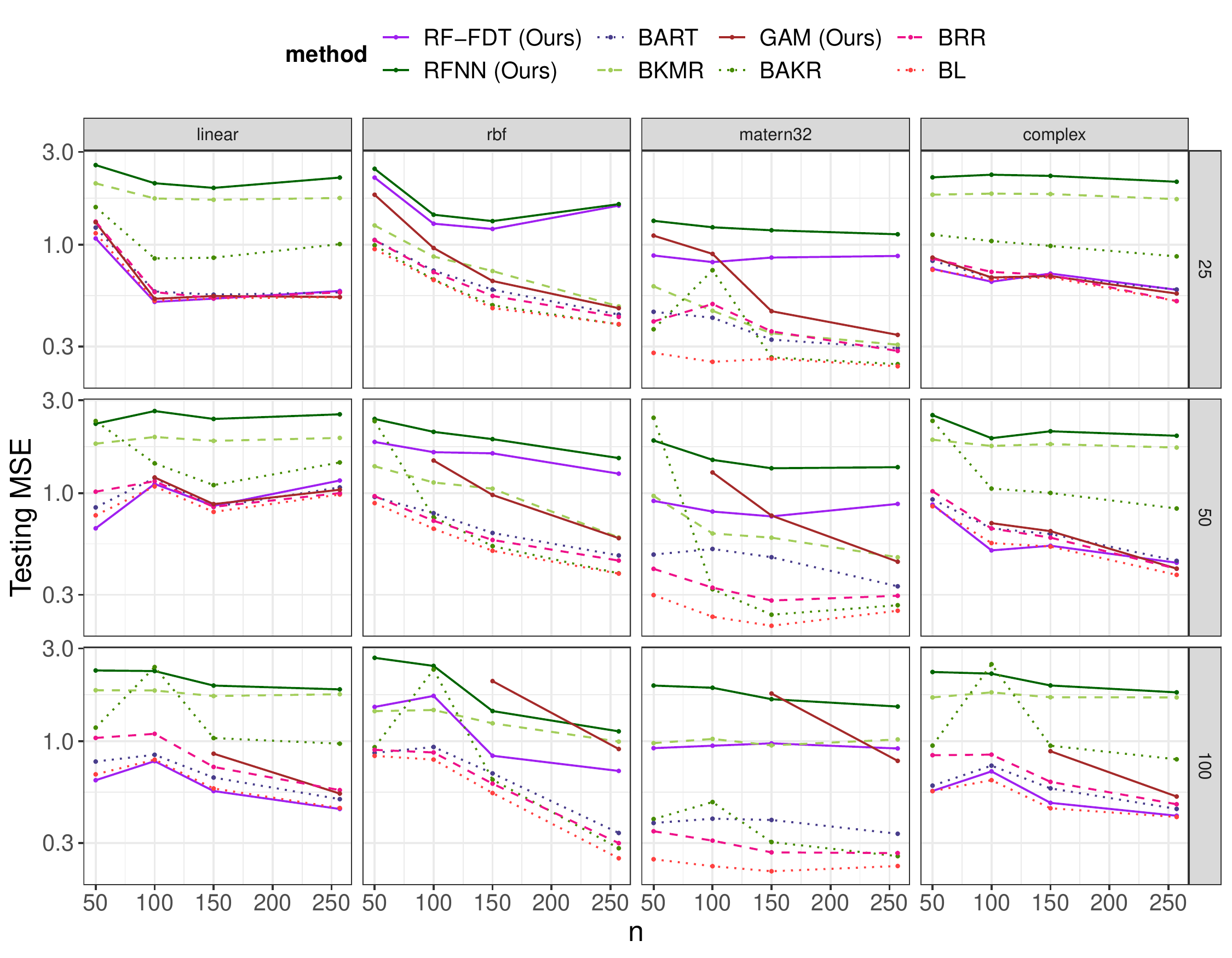}}
\caption{Testing MSE for \textbf{heart} data. A method will not be shown if they share the model fit with another method (\gls{impurity} and \gls{knockoff}), or if the method does not produce valid result due to small sample size (\gls{GAM}). Notice that this dataset contains a setting $n=p$, which can lead to the double descent phenonmenon for some random-feature-based models \citep{d2020triple}.}
\label{fig:tst_heart}
\end{center}
\vskip -0.2in
\end{figure*}

\begin{figure*}[ht]
\vskip 0.2in
\begin{center}
\centerline{\includegraphics[width=1.0\textwidth]{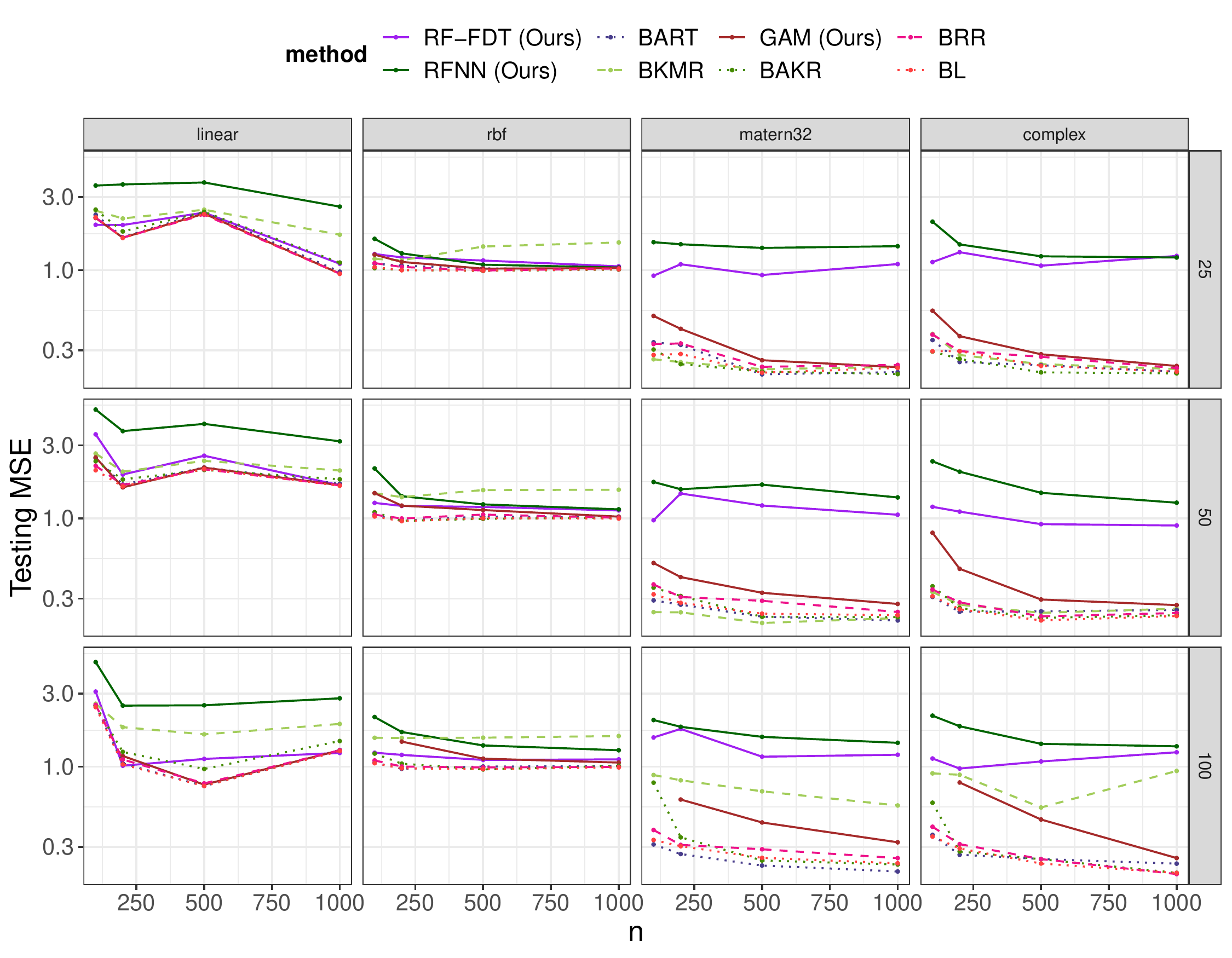}}
\caption{Testing MSE for \textbf{mi} data. A method will not be shown if they share the model fit with another method (\gls{impurity} and \gls{knockoff}), or if the method does not produce valid result due to small sample size (\gls{GAM}).}
\label{fig:tst_mi}
\end{center}
\vskip -0.2in
\end{figure*}

\clearpage
\section{Additional Experiments: Regularization Path for Bangladesh birth cohort study}
\label{sec:exp_bang}

We propose a way to visualize the selection path that incorporates the uncertainty of variable importance scores. Specifically, we consider the posterior survival function $S(s) = P(\psi_j > s), j=1, \ldots, d$ for increasing $s$ starting from $0$. Larger value of $S(s)$ indicates larger probability of that certain feature being relevant. This is analogous to the regularization path under the LASSO method. However, our approach incorporates posterior uncertainty, and does not require repeated model fitting at different levels of regularization strength \cite{mairal2012complexity}. 

We apply this to Bangladesh birth cohort study \citep{kile2014prospective} (a well-established dataset in the environmental health literature), where we fit models to learn the association between infant's neural development scores and key environmental factors such as hospital location (\texttt{clinic}), sex (\texttt{sex}), levels of macro nutrient intake (\texttt{prot}, \texttt{fat}, \texttt{carb}, \texttt{fib}, \texttt{ash}) and levels of  measured concentration of environmental toxins in body fluids (\texttt{as$\_$ln}, \texttt{mn$\_$ln}, \texttt{pb$\_$ln}), while controlling for other socio-economic and biological factors (family income, parent education levels, etc). In general, the level of macro-nutrient intake (in particular fiber and protein) indicates a  child's general nutrition status (i.e., whether he/she is eating well), and is known to be positively associated with neural development. On the other hand, the existing studies in the Bangladesh population have established a neurotoxic effect between arsenic exposure (i.e., \texttt{as$\_$ln}), through drink water) on the early-stage cognitive development \citep{hamadani2011critical}, as well as weak but significant effect of the joint mixture of other environmental toxins (manganese (\texttt{mn$\_$ln}) and lead (\texttt{pb$\_$ln}))) \citep{gleason2014contaminated, valeri2017joint}. Furthermore, due the fact that the model has already controlled for biological and socio-economic confounding factors, non-nutrient-related factors such as hospital location and sex should not have a significant effect on the  children's neural development status.

The variable selection result is shown in \cref{fig:reg_bang}, where we plot the posterior survival function $P(\psi_j > s)$ for $s \in (0, 1)$, and compare it to the survival function under \glsfirst{BAKR}, \glsfirst{BRR}, \glsfirst{BL}, and also the frequentist LASSO regularization path under the GAM model. We normalized all variable importance scores within the range $(0, 1)$. As a result, the variable selection performance is indicated by the relative magnitudes of the area under the curve for each variable (and not by the absolute magnitude due to the normalization).

As shown in \cref{fig:reg_bang}, the top variables selected by our method (FDT) correspond well with existing conclusions in the literature: it correctly picked up the larger impact of macro-nutrients (in particular, fibre, fat and protein) and smaller but still significant effects of environmental toxins (arsenic, manganese and lead), also notice that it ranked known non-causal factors such as hospital location and sex to be the lowest. In comparison, the linear methods (\textbf{GAM}, \gls{BRR} and \gls{BL}) all incorrectly reported high effect from hospital location on children's neural developement outcome (likely due to their restrict model form), while the nonlinear model (BAKR, based on RBF kernel) did not properly pick up the effect of environmental toxins.

\begin{figure*}[ht]
\begin{center}
\centerline{\includegraphics[width=1.0\textwidth]{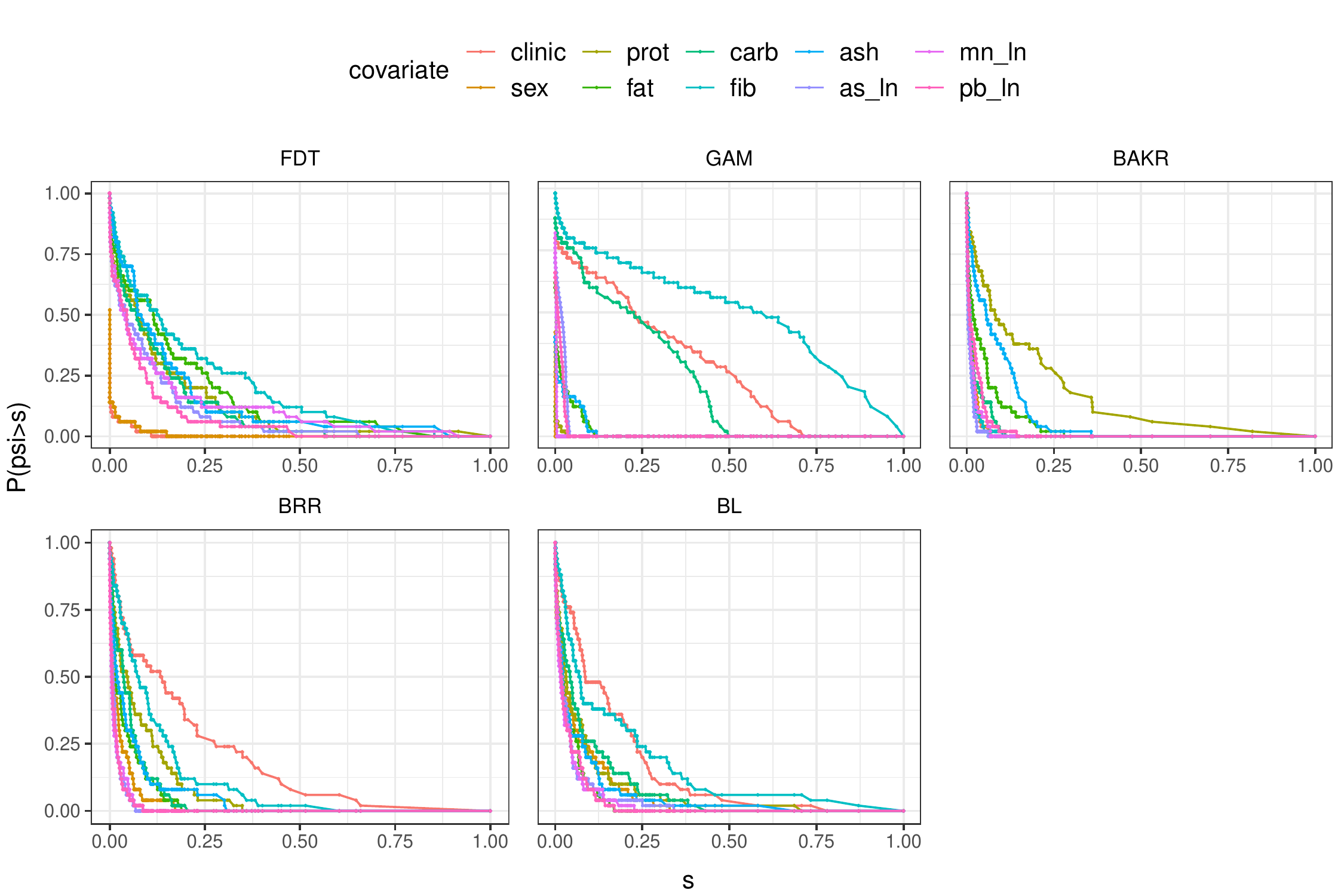}}
\caption{Regularization path for Bangladesh birth cohort study. The top variables selected by our method (FDT) correspond well with established toxicology pathways in the literature.}
\label{fig:reg_bang}
\end{center}
\vskip -0.2in
\end{figure*}


\end{document}

%% file: acronym.tex
\newacronym{RKHS}{RKHS}{reproducing kernel Hilbert space}
\newacronym{GP}{GP}{Gaussian process}
\newacronym{BvM}{BvM}{\textit{Bernstein-von Mises}}
\newacronym{CRB}{CRB}{Cram\'{e}r-Rao bound}

\newacronym{FDT}{\textbf{FDT}s}{featurized decision trees}
\newacronym{DNN}{\textbf{DNN}}{deep neural network}

\newacronym{AUROC}{AUROC}{\textit{area under the receiver operating characteristic}}
\newacronym{RF}{RF}{Random Forests}
\newacronym{impurity}{\textbf{RF-impurity}}{\textit{impurity}}
\newacronym{knockoff}{\textbf{RF-knockoff}}{\textit{RF-based kernel knockoff}}
\newacronym{BART}{\textbf{BART}}{\textit{Bayesian additive regression trees}}
\newacronym{RFNN}{\textbf{RFNN}}{\textit{Random-feature Neural Networks}}
\newacronym{BKMR}{\textbf{BKMR}}{\textit{Bayesian Kernel Machine Regression}}
\newacronym{BAKR}{\textbf{BAKR}}{\textit{Bayesian Approximate Kernel Regression}}
\newacronym{GAM}{\textbf{GAM}}{\textit{Generalized Additive Model}}
\newacronym{BRR}{\textbf{BRR}}{\textit{Bayesian Ridge Regression}}
\newacronym{BL}{\textbf{BL}}{\textit{Bayesian LASSO}}